\documentclass{article} 
\pdfoutput=1
\usepackage[dvipsnames,table,xcdraw]{xcolor}
\usepackage{iclr2022_conference,times}

\usepackage{amsmath,amsfonts,bm}

\def\eqref#1{equation~\ref{#1}}

\def\1{\bm{1}}

\DeclareMathAlphabet{\mathsfit}{\encodingdefault}{\sfdefault}{m}{sl}
\SetMathAlphabet{\mathsfit}{bold}{\encodingdefault}{\sfdefault}{bx}{n}

\usepackage{hyperref}
\usepackage{url}

\usepackage{microtype}
\usepackage{graphicx}
\usepackage{subfigure}
\usepackage{booktabs} 

\usepackage{amsthm}
\usepackage{mathtools}
\usepackage{balance}
\usepackage{cancel}
\usepackage[title]{appendix}

\usepackage{caption}
\usepackage{pifont}
\newcommand{\cmark}{\ding{51}}%
\newcommand{\xmark}{\ding{55}}%

\usepackage{amssymb}
\usepackage{amsmath}
\newtheorem{theorem}{Theorem}
\newtheorem{lemma}{Lemma}

\usepackage[cjk]{kotex}
\usepackage{algorithm} 
\usepackage{algorithmic}  
\usepackage[titlenumbered,ruled,algo2e]{algorithm2e}
\usepackage{multirow,multicol,ctable,tabularx} 
\allowdisplaybreaks
\usepackage{hyperref}

\usepackage{natbib}
\usepackage[resetlabels]{multibib}
\newcites{apndx}{References in Appendix}

\title{Stein Latent Optimization \\ for Generative Adversarial Networks}

\author{Uiwon Hwang$^1$, Heeseung Kim$^1$, Dahuin Jung$^1$, Hyemi Jang$^1$, Hyungyu Lee$^1$, Sungroh Yoon$^{1,2,}$\thanks{Corresponding author}\\
$^1$Department of Electrical and Computer Engineering, $^2$AIIS, ASRI, INMC, ISRC, NSI, and \\ Interdisciplinary Program in Artificial Intelligence, Seoul National University, Seoul 08826, Korea \\
\small{\texttt{\{uiwon.hwang,\:gmltmd789,\:annajung0625,\:wkdal9512,\:rucy74,\:sryoon\}@snu.ac.kr}}
}

\iclrfinalcopy 
\begin{document}

\maketitle

\begin{abstract}
Generative adversarial networks (GANs) with clustered latent spaces can perform conditional generation in a completely unsupervised manner. In the real world, the salient attributes of unlabeled data can be imbalanced. However, most of existing unsupervised conditional GANs cannot cluster attributes of these data in their latent spaces properly because they assume uniform distributions of the attributes. To address this problem, we theoretically derive Stein latent optimization that provides reparameterizable gradient estimations of the latent distribution parameters assuming a Gaussian mixture prior in a continuous latent space. Structurally, we introduce an encoder network and novel unsupervised conditional contrastive loss to ensure that data generated from a single mixture component represent a single attribute. We confirm that the proposed method, named Stein Latent Optimization for GANs (SLOGAN), successfully learns balanced or imbalanced attributes and achieves state-of-the-art unsupervised conditional generation performance even in the absence of attribute information (e.g., the imbalance ratio). Moreover, we demonstrate that the attributes to be learned can be manipulated using a small amount of probe data.
\end{abstract}

\section{Introduction}
GANs have shown remarkable results in the synthesis of realistic data conditioned on a specific class \citep{acgan, projection, contragan}. Training conditional GANs requires a massive amount of labeled data; however, data are often unlabeled or possess only a few labels. For unsupervised conditional generation, the salient attributes of the data are first identified by unsupervised learning and used for conditional generation of data. Recently, several unsupervised conditional GANs have been proposed \citep{infogan, clustergan, cdgan, pgmgan}. By maximizing a lower bound of mutual information between latent codes and generated data, they cluster the attributes of the underlying data distribution in their latent spaces. These GANs achieve satisfactory performance when the salient attributes of data are balanced.

However, the attributes of real-world data can be \textit{imbalanced}. For example, in the CelebA dataset \citep{celeba}, examples with one attribute (not wearing eyeglasses) outnumber the other attribute (wearing eyeglasses). Similarly, the number of examples with disease-related attributes in a biomedical dataset might be miniscule \citep{hexagan}. Thus, the imbalanced nature of real-world attributes must be considered for unsupervised conditional generation. Most of existing unsupervised conditional GANs are not suitable for real-world attributes, because they assume balanced attributes if the imbalance ratio is unknown \citep{infogan, clustergan, cdgan}. Examples where existing methods fail to learn imbalanced attributes are shown in Figure \ref{fig:toy} (a), (b) and (c).


\begin{figure}[t]
\centering
\includegraphics[width=0.99\linewidth]{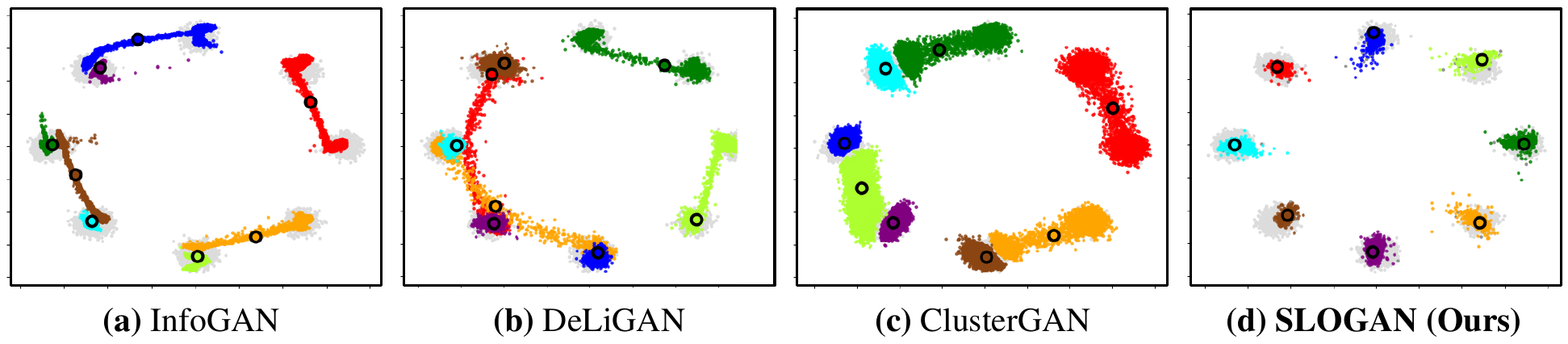} 
\caption{Unsupervised conditional generation on synthetic dataset. Dataset consists of eight two-dimensional Gaussians (gray dots), and the number of unlabeled data instances from each Gaussian distribution is imbalanced (clockwise from the top, imbalance ratio between the first four Gaussians and the remaining four is 1:3). It is considered that the instances sampled from the same Gaussian share an attribute. Dots with different colors denote the data generated from different latent codes. Bold circles represent the samples generated from the mean vectors of latent distributions.}
\label{fig:toy}
\end{figure}

In this paper, we propose unsupervised conditional GANs, referred to as Stein Latent Optimization for GANs (SLOGAN). We define the latent distribution of the GAN models as Gaussian mixtures to enable the imbalanced attributes to be naturally clustered in a continuous latent space. We derive reparameterizable gradient identities for the mean vectors, full covariance matrices, and mixing coefficients of the latent distribution using Stein's lemma. This enables stable learning and makes latent distribution parameters, including the mixing coefficient, learnable. We then devise a GAN framework with an encoder network and an unsupervised conditional contrastive loss (U2C loss), which can interact well with the learnable Gaussian mixture prior (Figure \ref{fig:overview}). This framework facilitates the association of data generated from a Gaussian component with a single attribute. 

For the synthetic dataset, our method (Figure \ref{fig:toy} (d)) shows superior performance on unsupervised conditional generation, with the accurately learned mixing coefficients. We performed experiments on various real-world datasets including MNIST \citep{mnist}, Fashion-MNIST \citep{fmnist}, CIFAR-10 \citep{cifar10}, CelebA \citep{celeba}, CelebA-HQ \citep{celeba-hq}, and AFHQ \citep{afhq} using architectures such as DCGAN \citep{dcgan}, ResGAN \citep{wgan-gp}, and StyleGAN2 \citep{stylegan2}. Through experiments, we verified that the proposed method outperforms existing unsupervised conditional GANs in unsupervised conditional generation on datasets with balanced or imbalanced attributes. Furthermore, we confirmed that we could control the attributes to be learned when a small set of probe data is provided.

The contributions of this work are summarized as follows:
\begin{itemize}
\item We propose novel Stein Latent Optimization for GANs (SLOGAN). To the best of our knowledge, this is one of the first methods that can perform unsupervised conditional generation by considering the imbalanced attributes of real-world data.
\item To enable this, we derive the implicit reparameterization for Gaussian mixture prior using Stein's lemma. Then, we devise a GAN framework with an encoder and an unsupervised conditional contrastive loss (U2C loss) suitable for implicit reparameterization.
\item SLOGAN significantly outperforms the existing methods on unsupervised learning tasks, such as cluster assignment, unconditional data generation, and unsupervised conditional generation, on datasets that include balanced or imbalanced attributes.
\end{itemize}

%

\section{Background}
\subsection{Generative Adversarial Networks}

\paragraph{Unsupervised conditional generation}
Several models including InfoGAN \citep{infogan}, ClusterGAN \citep{clustergan}, Self-conditioned GAN \citep{scgan}, CD-GAN \citep{cdgan}, and PGMGAN \citep{pgmgan} have been proposed to perform conditional generation in a completely unsupervised manner. However, these models primarily have two drawbacks: (1) Most of these methods embed the attributes in discrete variables, which induces discontinuity among the embedded attributes. (2) Most of them assume uniform distributions of the attributes, and thus fail to learn the imbalance in attributes when the imbalance ratio is not provided. In Appendix \ref{sec:related_apndx}, we discuss the above models in detail. By contrast, our work addresses the aforementioned limitations by combining GANs with the gradient estimation of the Gaussian mixture prior via Stein's lemma and representation learning on the latent space.

\paragraph{GANs with Gaussian mixture prior}
DeLiGAN \citep{deligan} is analogous to the proposed method, as it assumes a Gaussian mixture prior and learns the mean vectors and covariance matrices via the reparameterization trick. However, DeLiGAN assumes uniform mixing coefficients without updating them. As a result, it fails to perform unsupervised conditional generation on datasets with imbalanced attributes. In addition, it uses the explicit reparameterization trick, which inevitably suffers from high variance in the estimated gradients. This will be discussed further in Section \ref{sec:gradient}.

\begin{figure*}[t]
\centering
\includegraphics[width=0.99\textwidth]{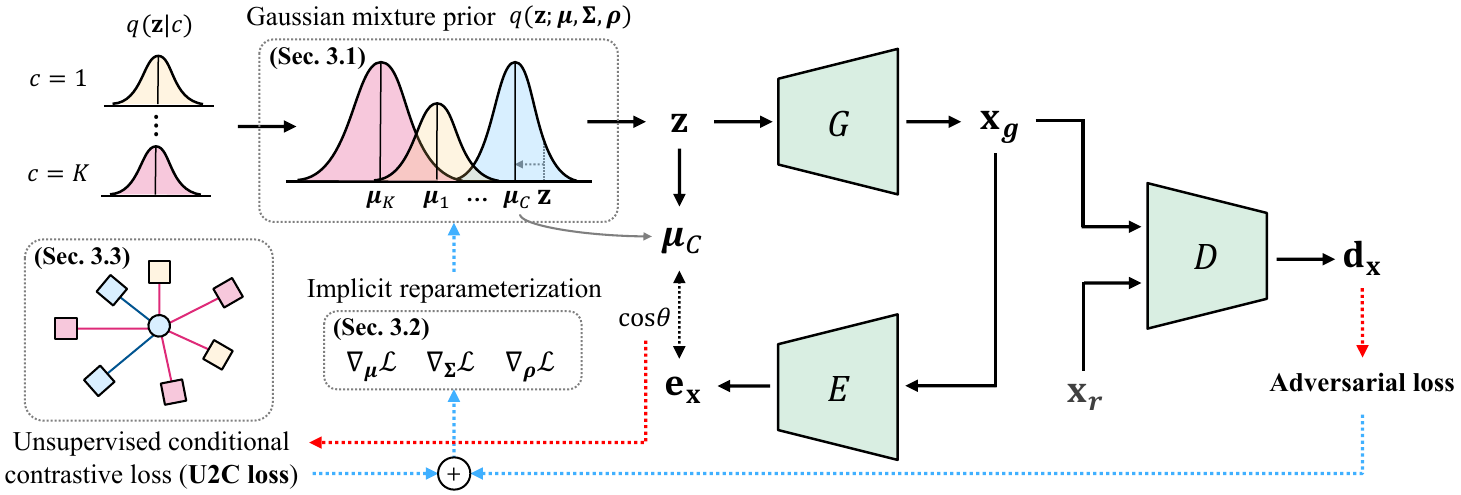} 
\caption{Overview of the SLOGAN model. Here, $\mathbf{x}_g$ denotes the data generated from a latent vector $\mathbf{z}$, $\mathbf{x}_r$ is a real data that is used for adversarial learning, and $C$ indicates a component ID of the Gaussian mixture prior with the highest responsibility $\operatorname*{argmax}_{c} q(c|\mathbf{z})$.}
\label{fig:overview}
\end{figure*}

\subsection{Contrastive Learning}
Contrastive learning aims to learn representations by contrasting neighboring with non-neighboring instances \citep{hadsell2006dimensionality}. In general, contrastive loss is defined as a critic function that approximates the log density ratio $\log p(y|x) / p(y)$ of two random variables $X$ and $Y$. By minimizing the loss, the lower bound of the mutual information $I(X;Y)$ is approximately maximized \citep{poole2019variational}. Several studies have shown that contrastive losses are advantageous for the representation learning of imbalanced data \citep{contrastive_imbalance, contrastive_decoupling, contrastive_bootstrap}. Motivated by these observations, we propose a contrastive loss that cooperates with a learnable latent distribution.

\subsection{Gradient Estimation for Gaussian Mixture} \label{sec:gradient}
\paragraph{Stein's lemma}
Stein's lemma provides a first-order gradient identity for a multivariate Gaussian distribution. The univariate case of Stein's lemma can be described as follows: 
\begin{lemma} \label{le:stein}
    Let function $h(\cdot): \mathbb{R} \mapsto \mathbb{R}$ be continuously differentiable. $q(z)$ is a univariate Gaussian distribution parameterized by the mean $\mu$ and variance $\sigma$. Then, the following identity holds:
    \begin{gather}
        \mathbb{E}_{q(z)} \left[ \sigma^{-1}(z-\mu)h(z) \right] = \mathbb{E}_{q(z)} \left[ \nabla_z h(z) \right]
    \end{gather}
\end{lemma}
\citet{lin2019stein} generalized Stein's lemma to exponential family mixtures and linked it to the implicit reparameterization trick. Stein's lemma has been applied to various fields of deep learning, including Bayesian deep learning \citep{lin2019fast} and adversarial robustness \citep{smoothed}. To the best of our knowledge, our work is the first to apply Stein's lemma to GANs.

\paragraph{Reparameterization trick}
A simple method to estimate gradients of the parameters of Gaussian mixtures is explicit reparameterization, used in DeLiGAN. When the $c$-th component is selected according to the mixing coefficient $p(c)$, the latent variable is calculated as follows: $\mathbf{z}=\boldsymbol{\mu}_c + \boldsymbol{\epsilon} \cdot \boldsymbol{\Sigma}_c^{1/2}$, where $\boldsymbol{\epsilon} \sim \mathcal{N}(\mathbf{0}, \mathbf{I})$. Derivatives of a loss function $\frac{\partial \mathcal{L}(z)}{\partial \boldsymbol{\mu}}$ and $\frac{\partial \mathcal{L}(z)}{\partial \boldsymbol{\Sigma}}$ only update the mean and covariance matrices of the \textit{selected} ($c$-th) component $\boldsymbol{\mu}_c$ and $\boldsymbol{\Sigma}_c$, respectively. Gradient estimation using explicit reparameterization is unbiased; however, it has a distinctly high variance. For a single latent vector $\mathbf{z}$, the implicit reparameterization trick \citep{implicit} updates the parameters of \textit{all} the latent components. Gradient estimation using implicit reparameterization is unbiased and has a lower variance, which enables a more stable and faster convergence of the model. The gradients for the parameters of the Gaussian mixture prior in our method are implicitly reparameterizable. 

\section{Proposed Method}
In the following sections, we propose Stein Latent Optimization for GANs (SLOGAN). We assume a Gaussian mixture prior (Section \ref{sec:3.1}), derive implicit reparameterization of the parameters of the mixture prior (Section \ref{sec:3.2}), and construct a GAN framework with U2C loss (Section \ref{sec:3.3}). Additionally, we devise a method to manipulate attributes to be learned if necessary (Section \ref{sec:3.5}). An overview of SLOGAN is shown in Figure \ref{fig:overview}.

\subsection{Gaussian mixture Prior} \label{sec:3.1}
We consider a GAN with a generator $G: \mathbb{R}^{d_z} \mapsto \mathbb{R}^{d_x}$ and a discriminator $D: \mathbb{R}^{d_x} \mapsto \mathbb{R}$, where $d_z$ and $d_x$ are the dimensions of latent and data spaces, respectively. In the latent space $\mathcal{Z} \in \mathbb{R}^{d_z}$, we consider a conditional latent distribution $q(\mathbf{z}|c)=\mathcal{N}(\mathbf{z}; \boldsymbol{\mu}_c, \Sigma_c)$, $c=1, ..., K$, where $K$ is the number of components we initially set and $\boldsymbol{\mu}_c$, $\Sigma_c$ are the mean vector and covariance matrix of the $c$-th component, respectively. Subsequently, we consider a Gaussian mixture $q(\mathbf{z})=\sum_{c=1}^{K} p(c) q(\mathbf{z}|c)$ parameterized by $\boldsymbol{\mu}=\{ \boldsymbol{\mu}_c \}_{c=1}^K$, $\boldsymbol{\Sigma}=\{ \Sigma_c \}_{c=1}^K$ and $\boldsymbol{\pi}=\{ \pi_c \}_{c=1}^K=\{ p(c) \}_{c=1}^K$ as the prior. 

We hypothesize that a mixture prior in a continuous space could model some continuous attributes of real-world data (e.g., hair color) more naturally than categorical priors which could introduce discontinuity \citep{clustergan}. Because we use implicit reparameterization of a mixture of Gaussian priors (derived in Section \ref{sec:3.2}), SLOGAN can fully benefit from implicit reparameterization and U2C loss. By contrast, the implicit reparameterization of prior distributions that do not belong to the exponential family (e.g., categorical priors) remains an open question.

In the experiments, the elements of $\boldsymbol{\mu}_c$ were sampled from $\mathcal{N}(0,0.1)$, and we selected $\Sigma_c=I$ and $\pi_c=1/K$ as the initial values. For the convenience of notation, we define the latent distribution $q=q(\mathbf{z})$, the mixing coefficient $\pi_c=p(c)$, and $\boldsymbol{\delta}(\mathbf{z})= \{\delta(\mathbf{z})_c\}_{c=1}^K$, where $\delta(\mathbf{z})_c=q(\mathbf{z}|c)/q(\mathbf{z})$.
$q(c|\mathbf{z})$, the responsibility of component $c$ for a latent vector $\mathbf{z}$, can be expressed as follows:
\begin{gather}
    q(c|\mathbf{z}) = \frac{q(c,\mathbf{z})}{q(\mathbf{z})} = \frac{q(\mathbf{z}|c)p(c)}{q(\mathbf{z})} = \delta(\mathbf{z})_{c} \pi_{c} \label{eq:resp}
\end{gather}

\subsection{Gradient Identities} \label{sec:3.2}
We present gradient identities for the latent distribution parameters. To derive the identities, we use the generalized Stein's lemma for Gaussian mixtures with full covariance matrices \citep{lin2019stein}. First, we derive a gradient identity for the mean vector using Bonnet's theorem \citep{bonnet}.
\begin{theorem} \label{th:bonnet}
    Given an expected loss of the generator $\mathcal{L}$ and a loss function for a sample $\ell(\cdot): \mathbb{R}^{d_z} \mapsto \mathbb{R}$, we assume $\ell$ to be continuously differentiable. Then, the following identity holds:
    \begin{gather}
        \nabla_{\boldsymbol{\mu}_c} \mathcal{L} =  \mathbb{E}_{q} \left[ \delta(\mathbf{z})_c \pi_c \nabla_{\mathbf{z}} \ell(\mathbf{z}) \right] \label{eq:bonnet}
    \end{gather}
\end{theorem}
Proof of Theorem \ref{th:bonnet} is given in Appendix \ref{sec:bonnet_apndx}.

We derive a gradient identity for the covariance matrix via Price's theorem \citep{price}. Among the two versions of the Price's theorem, we use the first-order identity to minimize computational cost. \begin{theorem} \label{th:price}
    With the same assumptions as in Theorem \ref{th:bonnet}, the following gradient identity holds:
    \begin{gather}
        \nabla_{\Sigma_c} \mathcal{L} = \frac{1}{2} \mathbb{E}_{q} \left[ \delta(\mathbf{z})_c \pi_c \Sigma_c^{-1} \left( \mathbf{z}-\boldsymbol{\mu}_c \right) \nabla_{\mathbf{z}}^{T} \ell(\mathbf{z}) \right] \label{eq:price}
    \end{gather}
\end{theorem}
Proof of Theorem \ref{th:price} is given in Appendix \ref{sec:price_apndx}. In the implementation, we replaced the expectation of the right-hand side of Equation \ref{eq:price} with the average for a batch of latent vectors; hence, the updated $\Sigma_{c}$ may not be symmetric or positive-definite. To force a valid covariance matrix, we modify the updates of the covariance matrix as follows:
\begin{gather}
    \Delta\Sigma_c = -\nabla_{\Sigma_c} \mathcal{L} = - \frac{1}{2} \mathbb{E}_{q} \left[ \frac{1}{2} \left( S_{\mathbf{z}} + S_{\mathbf{z}}^{T} \right) \right] \label{eq:sym}\\
    \Delta\Sigma_{c}' = \Delta\Sigma_c + \frac{\gamma}{2} \Delta\Sigma_c \Sigma_c^{-1} \Delta\Sigma_c \label{eq:psd}
\end{gather}
where $S_{\mathbf{z}} = \delta(\mathbf{z})_c \pi_c \Sigma_c^{-1} \left( \mathbf{z}-\boldsymbol{\mu}_c \right) \nabla_{\mathbf{z}}^{T} \ell(\mathbf{z})$, and $\gamma$ denotes the learning rate for $\Sigma_c$. Equation \ref{eq:sym} holds as $\Delta\Sigma_c = \frac{1}{2} E_q \left[ S_{\mathbf{z}} \right] = \frac{1}{2} E_q \left[ S_{\mathbf{z}}^{T} \right]$. Motivated by \citet{pmlr-v119-lin20d}, Equation \ref{eq:psd} ensures the positive-definiteness of the covariance matrix, which is proved by Theorem \ref{th:psd}.
\begin{theorem} \label{th:psd}
    The updated covariance matrix $\Sigma_{c}'=\Sigma_{c}+\gamma\Delta\Sigma_{c}'$ with the modified update rule specified in Equation \ref{eq:psd} is positive-definite if $\Sigma_{c}$ is positive-definite.
\end{theorem}
Proof of Theorem \ref{th:psd} is provided in Appendix \ref{sec:psd_apndx}.

We introduce a mixing coefficient parameter $\rho_c$, which is updated instead of the mixing coefficient $\pi_c$, to guarantee that the updated mixing coefficients are non-negative and summed to one. $\pi_c$ can be calculated using the softmax function (i.e., $\pi_c=\exp(\rho_c) / \sum_{i=1}^K \exp(\rho_i)$). We can then derive the gradient identity for the mixing coefficient parameter as follows:
\begin{theorem} \label{th:rho}
    Let $\rho_c$ be a mixing coefficient parameter. Then, the following gradient identity holds:
    \begin{gather}
        \nabla_{\rho_c} \mathcal{L} = \mathbb{E}_{q} \left[ \pi_c \left(\delta(\mathbf{z})_c - 1\right) \ell(\mathbf{z}) \right] \label{eq:rho}
    \end{gather}
\end{theorem}
Proof of Theorem \ref{th:rho} is given in Appendix \ref{sec:rho_apndx}. Because the gradients of the latent vector with respect to the latent parameters are computed by implicit differentiation via Stein's lemma, we obtain the implicit reparameterization gradients introduced by \citet{implicit}.

\begin{figure}[tp]
\vspace*{-\baselineskip}
\begin{minipage}[t]{0.57\linewidth}
\begin{algorithm}[H]
\footnotesize
    \caption{Training procedure of SLOGAN}
    \label{alg:slogan}
\begin{algorithmic}
    \STATE Initialize $\boldsymbol{\mu}$, $\boldsymbol{\Sigma}$, $\boldsymbol{\rho}$, parameters of $D$, $G$, and $E$
    \WHILE{training loss is not converged}
    \STATE Sample a batch of data $\{\mathbf{x}^i\}_{i=1}^B$ $\sim$ $p(\mathbf{x})$ 
    \STATE Sample a batch of latent vectors $\{\mathbf{z}^i\}_{i=1}^B$ $\sim$ $q(\mathbf{z})$
    \FOR{$i=1,...,B$}
    \STATE Calculate $\ell_\mathrm{adv}(\mathbf{z}^i)$ and $\ell_\mathrm{U2C}(\mathbf{z}^i)$ for a latent vector $\mathbf{z}^i$
    \STATE $S_{\mathbf{z}^i} \leftarrow {\scriptstyle \delta(\mathbf{z}^i)_c \pi_c \Sigma_c^{-1} \left( \mathbf{z}^i - \boldsymbol{\mu}_c \right) \nabla_{\mathbf{z}^i}^{T} \left( \ell_\mathrm{adv} (\mathbf{z}^i) + \lambda \ell_\mathrm{U2C} (\mathbf{z}^i) \right)}$ 
    \ENDFOR
    \FOR{$c=1,...,K$}
    \STATE Update {\scriptsize $\boldsymbol{\mu}_c$, $\boldsymbol{\Sigma}_c$} and {\scriptsize $\boldsymbol{\rho}_c$} via stochastic gradient estimation
    \STATE $~~~\boldsymbol{\mu}_c \leftarrow {\scriptstyle \boldsymbol{\mu}_c - \gamma \frac{1}{B} \sum_{i=1}^B \delta(\mathbf{z}^i)_c \pi_c \nabla_{\mathbf{z}^i} \left( \ell_\mathrm{adv} (\mathbf{z}^i) + \lambda \ell_\mathrm{U2C} (\mathbf{z}^i) \right)}$
    \STATE $~~~\Delta \Sigma_c \leftarrow -\frac{1}{4B} \sum_{i=1}^B \left( S_{\mathbf{z}^i} + S_{\mathbf{z}^i}^{T} \right)$
    \STATE $~~~\Sigma_c \leftarrow \Sigma_c + \gamma \left( \Delta\Sigma_c + \frac{\gamma}{2} \Delta\Sigma_c \Sigma_c^{-1} \Delta\Sigma_c \right)$
    \STATE $~~~\rho_c \leftarrow \rho_c - \gamma \frac{1}{B} \sum_{i=1}^B \pi_c \left(\delta(\mathbf{z}^i)_c - 1\right)\ell_\mathrm{adv} (\mathbf{z}^i)$
    \ENDFOR
    \STATE Update $G$, $E$ and $D$ using SGD
    \STATE $~~~\nabla_{G,E} \frac{1}{B} \sum_{i=1}^B \left( \ell_\mathrm{adv} (\mathbf{z}^i) + \lambda \ell_\mathrm{U2C} (\mathbf{z}^i) \right)$
    \STATE $~~~\nabla_{D} \left( - \frac{1}{B} \sum_{i=1}^B \ell_\mathrm{adv} (\mathbf{z}^i) - \frac{1}{B} \sum_{i=1}^B D(\mathbf{x}^i) \right)$
    \ENDWHILE
\end{algorithmic}
\end{algorithm}
\end{minipage}
\hfill
\begin{minipage}[t]{0.41\linewidth}
\begin{algorithm}[H]
\footnotesize
    \SetKwInOut{Input}{input}
    \SetKwInOut{Output}{output}
    \Input{$\{\{\mathbf{x}^i_y\}_{i=1}^N\}_{y=1}^K$ - Data sampled from $p(\mathbf{x}|y)$ for $y=1, ... ,K$; 
    \newline $\{\{\mathbf{z}^i_c\}_{i=1}^N\}_{c=1}^K$ - Latent vectors sampled from $q(\mathbf{z}|c)$ for $c=1, ... ,K$}
    \Output{ICFID - Intra-cluster FID;
    \newline $Y_c$ - Class-cluster assignments}
    \caption{Intra-cluster FID}
    \label{alg:icfid}
\begin{algorithmic}
    \STATE $Y\leftarrow\{1,...,K\}$
    \STATE $C\leftarrow\{1,...,K\}$
    \FOR{each class $y$ in $Y$}
    \STATE $\mathbf{X}_r \leftarrow \{\mathbf{x}^i_y\}_{i=1}^N$
    \FOR{each cluster $c$ in $C$}
    \STATE $\mathbf{X}_g \leftarrow \{\mathbf{x}^i_c\}_{i=1}^N$
    \STATE $d(y, c) \leftarrow$ FID$(\mathbf{X}_r, \mathbf{X}_g)$
    \ENDFOR
    \STATE $c^* \leftarrow \operatorname*{argmin}_{c \in C} d(y, c)$
    \STATE ICFID$(y) \leftarrow d(y, c^*)$
    \STATE $Y_c(y) \leftarrow c^*$
    \STATE Remove $c^*$ from $C$
    \ENDFOR
    \STATE ICFID $\leftarrow \frac{1}{K} \sum_{y=1}^{K} \textsc{ICFID}(y)$
\end{algorithmic}
\end{algorithm}
\end{minipage}
\end{figure}

\subsection{Contrastive Learning} \label{sec:3.3}
We introduce new unsupervised conditional contrastive loss (U2C loss) to learn salient attributes from data and to facilitate unsupervised conditional generation. We consider a batch of latent vectors $\{\mathbf{z}^i\}_{i=1}^{B} \sim q(\mathbf{z})$, where $B$ is the batch size. Generator $G$ receives the $i$-th latent vector $\mathbf{z}^i$ and generates data $\mathbf{x}^{i}_{g}=G(\mathbf{z}^{i})$. The adversarial loss for $G$ with respect to the sample $\mathbf{z}^i$ is as follows:
\begin{gather}
    \ell_\mathrm{adv}(\mathbf{z}^i) = -D(G(\mathbf{z}^i))
\end{gather}
We also introduce an encoder network $E$ to implement U2C loss. The synthesized data $\mathbf{x}^{i}_{g}$ enters $E$, and $E$ generates an encoded vector $\mathbf{e}^{i}_\mathbf{x}=E(\mathbf{x}^{i}_{g})$. Then, we find the mean vector $\boldsymbol{\mu}^{i}_{C}$, where $C$ is the component ID with the highest responsibility $q(c|\mathbf{z}^{i})$. We calculate $C$ first because a generated sample should have the attribute of the most responsible component among multiple components in the continuous space. Second, to update the parameters of the prior using implicit reparameterization, the loss should be a function of a latent vector $\mathbf{z}^i$, as proved in Theorems \ref{th:bonnet}, \ref{th:price}, and \ref{th:rho}. The component ID for each sample is calculated as follows:
\begin{gather}
    C^{i} = \operatorname*{argmax}_{c} {q(c|\mathbf{z}^{i})} = \operatorname*{argmax}_{c} {\delta(\mathbf{z}^{i})_{c} \pi_{c}}
\end{gather}
where $q(c|\mathbf{z}^i)=\delta(\mathbf{z}^i)_c \pi_c$ is derived from Equation \ref{eq:resp}. To satisfy the assumption of the continuously differentiable loss function in Theorems \ref{th:bonnet} and \ref{th:price}, we adopt the Gumbel-Softmax relaxation \citep{gumbel}, instead of the $\operatorname*{argmax}$ function. We use $\boldsymbol{\mu}^{i}_\mathbf{C} = \sum_{c=1}^K \mathbf{C}^i_c \boldsymbol{\mu}_c$ to calculate U2C loss to ensure that the loss function is continuously differentiable with respect to $\mathbf{z}^i$, where $\mathbf{C}^i = \operatorname*{Gumbel-Softmax}_\tau (\boldsymbol{\delta}(\mathbf{z}^i) \boldsymbol{\pi})$ and $\tau=0.01$. We derive U2C loss as follows:
\begin{gather}
    \ell_\mathrm{U2C} (\mathbf{z}^i) = - \log \frac{\exp(\cos \theta_{ii})}{\frac{1}{B} \sum_{j=1}^{B} \exp(\cos \theta_{ij})} \label{eq:contrastive}
\end{gather}
where we select the cosine similarity between $\mathbf{e}^{i}_\mathbf{x}$ and $\boldsymbol{\mu}^{j}_\mathbf{C}$, $\cos \theta_{ij} = \mathbf{e}^{i}_\mathbf{x} \cdot \boldsymbol{\mu}^{j}_\mathbf{C} / \|\mathbf{e}^{i}_\mathbf{x}\| \|\boldsymbol{\mu}^{j}_\mathbf{C}\|$ as the critic function that approximates the log density ratio $\log p(C^j|\mathbf{x}_g^i) / p(C^j)$ for contrastive learning. Given a test data, the probability for each cluster can be calculated using the assumption of the critic function, which enables us to assign a cluster for the data. Cluster assignment is described in Appendix \ref{sec:cluster_assignment_apndx}.

Intuitively, a mean vector $\boldsymbol{\mu}^{i}_\mathbf{C}$ of a latent mixture component is regarded as a prototype of each attribute. U2C loss encourages the encoded vector $\mathbf{e}^{i}_\mathbf{x}$ of the generated sample to be similar to its assigned low-dimensional prototype $\boldsymbol{\mu}^{i}_\mathbf{C}$ in the latent space. This allows each salient attribute clusters in the latent space, and each component of the learned latent distribution is responsible for a certain attribute of the data. If $\cos \theta_{ii}$ is proportional to the log density ratio $\log p(C^i|\mathbf{x}_g^i) / p(C^i)$, minimizing U2C loss in Equation \ref{eq:contrastive} is equivalent to maximizing the lower bound of the mutual information $I(C^i;\mathbf{x}_g^i)$, as discussed by \citet{poole2019variational} and \citet{contrastive_clustering}. 

$G$ and $E$ are trained to minimize $\frac{1}{B} \sum_{i=1}^B \left( \ell_\mathrm{adv} (\mathbf{z}^i) + \lambda \ell_\mathrm{U2C} (\mathbf{z}^i) \right)$, where $\lambda$ denotes the coefficient of U2C loss. Both $\boldsymbol{\mu}$ and $\boldsymbol{\Sigma}$ are learned by substituting $\ell_\mathrm{adv} (\mathbf{z}^i) + \lambda \ell_\mathrm{U2C} (\mathbf{z}^i)$ into $\ell$ of Equations \ref{eq:bonnet} and \ref{eq:psd}, respectively. When U2C loss is used to update $\boldsymbol{\pi}$, U2C loss hinders $\boldsymbol{\pi}$ from estimating the imbalance ratio of attributes in the data well, which is discussed in Appendix \ref{sec:rho_update_apndx} with a detailed explanation and an empirical result. Therefore, $\boldsymbol{\rho}$, from which $\boldsymbol{\pi}$ is calculated, uses only the adversarial loss, and $\ell$ of Equation \ref{eq:rho} is substituted by $\ell_\mathrm{adv} (\mathbf{z}^i)$. $\boldsymbol{\mu}$, $\boldsymbol{\Sigma}$ and $\boldsymbol{\rho}$ are learned using a batch average of estimated gradients, which is referred to as stochastic gradient estimation, instead of expectation over the latent distribution $q$. The entire training procedure of SLOGAN is presented in Algorithm \ref{alg:slogan}.

To help that the latent space does not learn low-level attributes, such as background color, we additionally used the SimCLR \citep{simclr} loss on the generated data with DiffAugment \citep{diffaugment} to train the encoder on colored image datasets. Methodological details and discussion on SimCLR are presented in Appendix \ref{sec:simclr_apndx} and \ref{sec:simclr_analysis_apndx}, respectively.

\subsection{Attribute Manipulation} \label{sec:3.5}
For datasets such as face attributes, a data point can have multiple attributes simultaneously. To learn a desired attribute from such data, a probe dataset $\{ \mathbf{x}_c^i \}_{i=1}^M$ for the $c$-th latent component, which consists of $M$ data points with the desired attribute, can be utilized. We propose the following loss:
\begin{gather}
    \mathcal{L}_\mathrm{m} = \frac{1}{M} \sum_{i=1}^M - \log \frac{\exp(\cos \theta^{i}_{c})}{\sum_{k=1}^{K} \exp(\cos \theta^{i}_{k})} \label{eq:probe}
\end{gather}
where $\cos \theta^{i}_{k} = E(\mathbf{x}_c^i) \cdot \boldsymbol{\mu}_k / \|E(\mathbf{x}_c^i)\| \|\boldsymbol{\mu}_k\|$ is the cosine similarity between $E(\mathbf{x}_c^i)$ and $\boldsymbol{\mu}_k$. Our model manipulates attributes by minimizing $\mathcal{L}_\mathrm{m}$ for $\boldsymbol{\mu},\boldsymbol{\Sigma}$, $G$, and $E$. In addition, mixup \citep{mixup} can be used to better learn attributes from a small probe dataset. The advantage of SLOGAN in attribute manipulation is that it can learn imbalanced attributes even if the attributes in the probe dataset are balanced, and perform better conditional generation. The detailed procedure of attribute manipulation is described in Appendix \ref{sec:attribute_manipulation_apndx}.

\section{Experiments} \label{sec:experiment}
\begin{table}[t!]
\centering
\caption{Performance comparison on balanced attributes}
\label{tab:balance}
{\resizebox{1\columnwidth}{!}
{\begin{tabular}{ccccccccccc}
\hline
	\toprule
    Dataset & Metric & InfoGAN & DeLiGAN & DeLiGAN+ & ClusterGAN & SCGAN & CD-GAN & PGMGAN & \textbf{SLOGAN} \\
    \midrule
    \multirow{3}*{\centering FMNIST}
    & NMI $\uparrow$ & 0.64\footnotesize{$\pm$0.02} & 0.64\footnotesize{$\pm$0.03} & 0.57\footnotesize{$\pm$0.07} & 0.61\footnotesize{$\pm$0.03} & 0.56\footnotesize{$\pm$0.01} & 0.56\footnotesize{$\pm$0.04} & 0.47\footnotesize{$\pm$0.01} & \textbf{0.66\footnotesize{$\pm$0.01}} \\
    & FID $\downarrow$ & 5.28\footnotesize{$\pm$0.12} & 6.65\footnotesize{$\pm$0.48} & 7.23\footnotesize{$\pm$0.56} & 6.32\footnotesize{$\pm$0.25} & \textbf{5.07\footnotesize{$\pm$0.19}} & 9.05\footnotesize{$\pm$0.11} & 9.13\footnotesize{$\pm$0.28} & 5.20\footnotesize{$\pm$0.36} \\
    & ICFID $\downarrow$ & 32.18\footnotesize{$\pm$2.11} & 34.87\footnotesize{$\pm$5.29} & 30.53\footnotesize{$\pm$8.71} & 37.20\footnotesize{$\pm$5.50} & 26.23\footnotesize{$\pm$7.10} & 36.61\footnotesize{$\pm$0.47} & 40.00\footnotesize{$\pm$4.38} & \textbf{23.31\footnotesize{$\pm$2.77}} \\
    \midrule 
    \multirow{3}*{\centering CIFAR-10}
    & NMI $\uparrow$ & 0.03\footnotesize{$\pm$0.00} & 0.06\footnotesize{$\pm$0.00} & 0.09\footnotesize{$\pm$0.04} & 0.10\footnotesize{$\pm$0.00} & 0.01\footnotesize{$\pm$0.00} & 0.03\footnotesize{$\pm$0.01} & 0.29\footnotesize{$\pm$0.02} & \textbf{0.34\footnotesize{$\pm$0.01}} \\
    & FID $\downarrow$ & 81.84\footnotesize{$\pm$2.27} & \footnotesize{212.20}\footnotesize{$\pm$4.52} & \footnotesize{110.51}\footnotesize{$\pm$7.70} & 61.97\footnotesize{$\pm$3.69} & \footnotesize{199.28}\footnotesize{$\pm$57.16} &  34.13\footnotesize{$\pm$1.13} &  31.50\footnotesize{$\pm$0.73} & \textbf{20.61\footnotesize{$\pm$0.40}} \\
    & ICFID $\downarrow$ & \footnotesize{139.20}\footnotesize{$\pm$2.09} & \footnotesize{305.32}\footnotesize{$\pm$5.05} & \footnotesize{215.63}\footnotesize{$\pm$11.16} & \footnotesize{124.27}\footnotesize{$\pm$5.95} & \footnotesize{262.54}\footnotesize{$\pm$59.29} & 
    95.43\footnotesize{$\pm$3.58} & 81.25\footnotesize{$\pm$11.55} & \textbf{71.23\footnotesize{$\pm$6.76}} \\
    \bottomrule
\end{tabular}}
}
\end{table}
\begin{table}[t!]
\centering
\caption{Performance comparison on imbalanced attributes}
\label{tab:imbalance}
{\resizebox{1\columnwidth}{!}
{\begin{tabular}{cccccccccc}
\hline
	\toprule
    Dataset & Metric & InfoGAN & DeLiGAN & DeLiGAN+ & ClusterGAN & SCGAN & CD-GAN & PGMGAN & \textbf{SLOGAN} \\
    \midrule
    \multirow{3}*{\centering FMNIST-5}
    & NMI $\uparrow$ & 0.58\footnotesize{$\pm$0.07} & \textbf{0.68\footnotesize{$\pm$0.05}} & 0.65\footnotesize{$\pm$0.01} & 0.60\footnotesize{$\pm$0.02} & 0.60\footnotesize{$\pm$0.06} & 0.59\footnotesize{$\pm$0.01} & 0.24\footnotesize{$\pm$0.02} & 0.66\footnotesize{$\pm$0.06} \\
    & FID $\downarrow$ & 5.40\footnotesize{$\pm$0.14} & 7.05\footnotesize{$\pm$0.49} & 6.33\footnotesize{$\pm$0.44} & 5.61\footnotesize{$\pm$0.17} & \textbf{5.01\footnotesize{$\pm$0.20}} & 9.34\footnotesize{$\pm$0.56} & 11.80\footnotesize{$\pm$0.43} & 5.29\footnotesize{$\pm$0.16} \\
    & ICFID $\downarrow$ & 43.69\footnotesize{$\pm$10.84} & 36.21\footnotesize{$\pm$3.07} & 35.41\footnotesize{$\pm$0.79} & 36.94\footnotesize{$\pm$5.81} & 44.48\footnotesize{$\pm$21.62} & 39.31\footnotesize{$\pm$1.18} & 77.30\footnotesize{$\pm$8.60} & \textbf{32.46\footnotesize{$\pm$3.18}} \\
    \midrule
    \multirow{3}*{\centering {{\begin{tabular}[c]{@{}c@{}}CIFAR-2\\ (7:3)\end{tabular}}}}
    & NMI $\uparrow$ & 0.05\footnotesize{$\pm$0.01} & 0.00\footnotesize{$\pm$0.00} & 0.03\footnotesize{$\pm$0.03} & 0.22\footnotesize{$\pm$0.02} & 0.00\footnotesize{$\pm$0.00} & 0.22\footnotesize{$\pm$0.03} & 0.42\footnotesize{$\pm$0.03} & \textbf{0.69\footnotesize{$\pm$0.02}} \\
    & FID $\downarrow$ & 51.30\footnotesize{$\pm$2.53} & \footnotesize{131.73}\footnotesize{$\pm$50.98} & \footnotesize{115.19}\footnotesize{$\pm$17.95} & 36.62\footnotesize{$\pm$2.16} & 45.28\footnotesize{$\pm$1.81} & 36.40\footnotesize{$\pm$1.01} & 29.76\footnotesize{$\pm$1.65} & \textbf{29.09\footnotesize{$\pm$0.73}} \\
    & ICFID $\downarrow$ & 88.49\footnotesize{$\pm$6.85} & \footnotesize{186.31}\footnotesize{$\pm$28.31} & \footnotesize{173.81}\footnotesize{$\pm$18.29} & 75.52\footnotesize{$\pm$4.82} & 88.58\footnotesize{$\pm$4.57} & 76.91\footnotesize{$\pm$1.07} & 57.06\footnotesize{$\pm$3.31} & \textbf{45.83\footnotesize{$\pm$3.03}} \\
    \bottomrule
\end{tabular}}
}
\end{table}

\begin{figure*}[t!]
\centering
\includegraphics[width=1\textwidth]{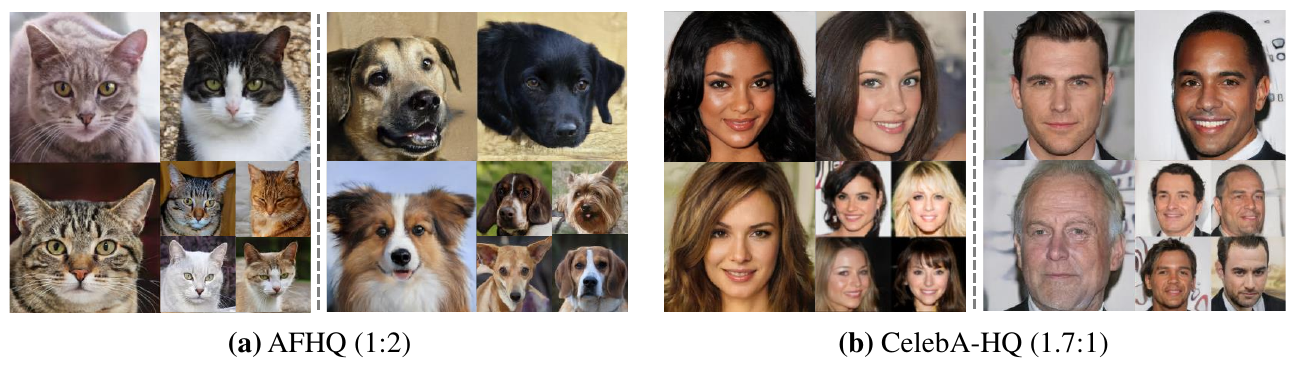} 
\caption{Generated high-fidelity images from SLOGAN on (a) AFHQ and (b) CelebA-HQ.}
\label{fig:qualitative}
\end{figure*}

\subsection{Datasets} \label{sec:dataset}
We used the MNIST \citep{mnist}, Fashion-MNIST (FMNIST) \citep{fmnist}, CIFAR-10 \citep{cifar10}, CelebA \citep{celeba}, CelebA-HQ \citep{celeba-hq}, and AFHQ \citep{afhq} datasets to evaluate the proposed method. We also constructed some datasets with imbalanced attributes. For example, we used two classes of the MNIST dataset (0 vs. 4, referred to as MNIST-2), two classes of the CIFAR-10 dataset (frogs vs. planes, referred to as CIFAR-2), and five clusters of the FMNIST dataset (\{Trouser\}, \{Bag\}, \{T-shirt/top, Dress\}, \{Pullover, Coat, Shirt\}, \{Sneaker, Sandal, Ankle Boot\}, referred to as FMNIST-5 with an imbalance ratio of 1:1:2:3:3). Details of the datasets are provided in Appendix \ref{sec:implementation_details_apndx}.

Although SLOGAN and other methods do not utilize labels for training, the data in experimental settings have labels predefined by humans. We consider that each class of dataset contains a distinct attribute. Thus, the model performance was measured using classes of datasets. The number of latent components or the dimension of the discrete latent code ($K$) was set as the number of classes of data.

\subsection{Evaluation Metrics}
The performance of our method was evaluated quantitatively in three aspects: (1) whether the model could learn distinct attributes and cluster real data (i.e., cluster assignment), which is evaluated using normalized mutual information (NMI) \citep{clustergan}, (2) whether the overall data distribution $p(\mathbf{x}_r)$ could be estimated (i.e., unconditional data generation), which is measured using the Fr\'echet inception distance (FID) \citep{ttur}, and, most importantly, (3) whether the data distribution for each attribute $p(\mathbf{x}_r|c)$ could be estimated (i.e., unsupervised conditional generation). 

For unsupervised conditional generation, it is important to account for intra-cluster diversity as well as the quality of the generated samples. We introduce a modified version of FID named intra-cluster Fr\'echet inception distance (ICFID) described in Algorithm \ref{alg:icfid}. We calculate FIDs between the real data of each class and generated data from each latent code (a mixture component for DeLiGAN and SLOGAN, and a category for other methods). We then greedily match a latent code with a class of real data with the smallest FID. We define ICFID as the average FID between the matched pairs and use it as an evaluation metric for unsupervised conditional generation. ICFID additionally provides class-cluster assignment (i.e., which cluster is the closest to the class).

\subsection{Evaluation Results}
We compared SLOGAN with InfoGAN \citep{infogan}, DeLiGAN \citep{deligan}, ClusterGAN \citep{clustergan}, Self-conditioned GAN (SCGAN) \citep{scgan}, CD-GAN \citep{cdgan}, and PGMGAN \citep{pgmgan}. Following \citet{clustergan}, we used k-means clustering on the encoder outputs of the test data to calculate NMI. DeLiGAN has no encoder network; hence the pre-activation of the penultimate layer of $D$ was used for the clustering metrics. For a fair comparison, we also compared DeLiGAN with an encoder network (referred to as DeLiGAN+). The same network architecture and hyperparameters (e.g., learning rate) were used across all methods for comparison. Details of the experiments and DeLiGAN+ are presented in Appendices \ref{sec:implementation_details_apndx} and \ref{sec:deligan+_apndx}, respectively.

\paragraph{Balanced attributes}
We compare SLOGAN with existing unsupervised conditional GANs on datasets with balanced attributes. As shown in Table \ref{tab:balance} (The complete version is given in Appendix \ref{sec:balance_apndx}.), SLOGAN outperformed other GANs, and comparisons with methods with categorical priors (ClusterGAN and CD-GAN) verified the advantages of the mixture priors. 

\begin{figure}[t!]
\begin{minipage}{0.54\linewidth}
\raggedleft
\includegraphics[width=1\linewidth]{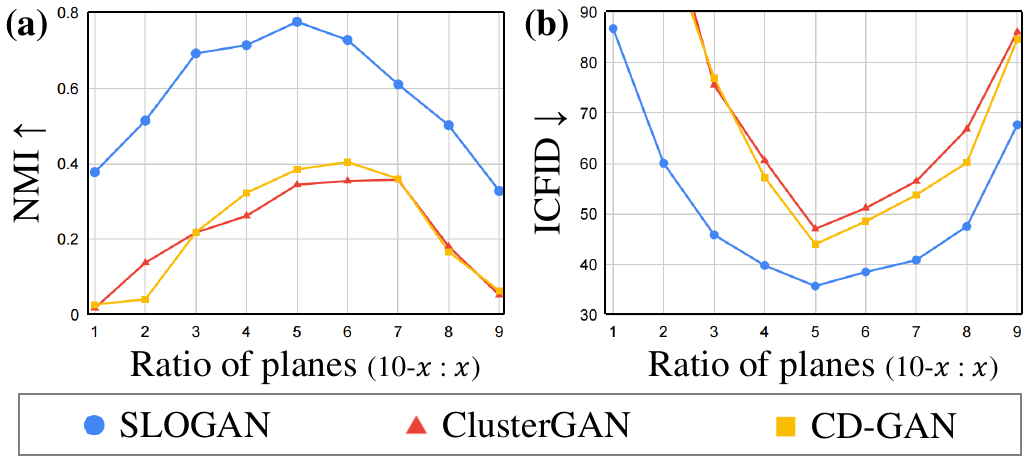} 
\caption{Performance comparison with respect to the imbalance ratio on (a) cluster assignment and (b) unsupervised conditional generation.} %
\label{fig:imb_ratio_apndx}
\end{minipage}
\hfill
\begin{minipage}{0.41\linewidth}
\raggedright
\setlength\tabcolsep{0pt}
\captionof{table}{Effectiveness of U2C loss}
\label{tab:ablation_contrastive}  
\begin{tabular*}{\linewidth}{@{\extracolsep{\fill}}llc}
    \toprule
    Dataset & Ablation & ICFID $\downarrow$ \\
    \midrule
    \multirow{2}{*}{\footnotesize{CIFAR-10}} & \footnotesize{SLOGAN w/o $\ell_\mathrm{U2C}$} & \footnotesize{78.26} \\
    & \footnotesize{SLOGAN} & \footnotesize{\textbf{71.23}} \\
    \midrule
    \footnotesize{MNIST-2} & \footnotesize{SLOGAN w/o $\ell_\mathrm{U2C}$} & \footnotesize{9.43} \\
    ~~~\footnotesize{(7:3)} & \footnotesize{SLOGAN} & \footnotesize{\textbf{5.91}} \\
    \midrule
    \multirow{2}{*}{\footnotesize{Synthetic}} & \footnotesize{SLOGAN w/o $\ell_\mathrm{U2C}$} & \footnotesize{\textcolor{red}{\xmark}} \\
    & \footnotesize{SLOGAN} & \footnotesize{\textcolor{ForestGreen}{\boldsymbol{\cmark}}} \\
    \bottomrule
\end{tabular*}
\end{minipage}
\end{figure}

\begin{table}[t!]
\centering
\caption{Effectiveness of implicit reparameterization}
\label{tab:ablation_implicit}
{\resizebox{0.85\columnwidth}{!}
{\begin{tabular}{llrr}
    \toprule
    Dataset & Ablation & $\pi_{y=0}$ \scriptsize{(ground-truth: 0.7)} & ICFID $\downarrow$ \\
    \midrule
    \multirow{3}{*}{\footnotesize{CIFAR-2 (7:3)}} & DeLiGAN with $\ell_\mathrm{U2C}$ & 0.50 & 60.51 \\
    & DeLiGAN with $\ell_\mathrm{U2C}$ and implicit $\boldsymbol{\rho}$ update & 1.00 & 86.48 \\
    & SLOGAN & \textbf{0.69} & \textbf{45.83} \\
    \bottomrule
\end{tabular}}}
\end{table}

\paragraph{Imbalanced attributes}
In Table \ref{tab:imbalance} (The complete version is presented in Appendix \ref{sec:imbalance_apndx}), we compare SLOGAN with existing methods on datasets with imbalanced attributes. ICFIDs of our method are much better than those of other methods, which indicates that SLOGAN was able to robustly capture the minority attributes in datasets and can generate data conditioned on the learned attributes. In CIFAR-2 (7:3), the ratio of frog and plane is 7 to 3 and the estimated $\boldsymbol{\pi}$ is (0.69$\pm$0.02, 0.31$\pm$0.02), which are very close to the ground-truth (0.7, 0.3). Figure \ref{fig:qualitative} (a) shows the images generated from each latent component of SLOGAN on AFHQ (Cat:Dog=1:2). More qualitative results are presented in Appendix \ref{sec:generated_images_apndx}.

\paragraph{Performance with respect to imbalance ratio}
We compared the performance of SLOGAN with competitive benchmarks (ClusterGAN and CD-GAN) by changing the imbalance ratios of CIFAR-2 from 9:1 to 1:9. SLOGAN showed higher performance than the benchmarks on cluster assignment (Figure~\ref{fig:imb_ratio_apndx} (a)) and unsupervised conditional generation (Figure~\ref{fig:imb_ratio_apndx} (b)) for all imbalance ratios. Furthermore, our method shows a larger gap in ICFID with the benchmarks when the ratio of planes is low. This implies that SLOGAN works robustly in situations in which the attributes of data are highly imbalanced. We conducted additional experiments including interpolation in the latent space, benefits of ICFID. The results of the additional experiments are shown in Appendix \ref{sec:additional_apndx}.

\subsection{Ablation Study}
\paragraph{U2C loss}
Table \ref{tab:ablation_contrastive} (The complete version is given in Appendix \ref{sec:ablation_apndx}) shows the benefit of U2C loss on several datasets. Low-level features (e.g., color) of the CIFAR dataset differ depending on the class, which enables SLOGAN to function to some extent without U2C loss on CIFAR-10. In the MNIST dataset, the colors of the background (black) and object (white) are the same, and only the shape of objects differs depending on the class. U2C loss played an essential role on MNIST (7:3). The modes of the Synthetic dataset (Figure \ref{fig:toy}) are placed adjacent to each other, and SLOGAN cannot function on this dataset without U2C loss. From the results, we observed that the effectiveness of U2C loss depends on the properties of the datasets.

\paragraph{Implicit reparameterization}
To show the advantage of implicit over explicit reparameterization, we implemented DeLiGAN with U2C loss by applying explicit reparameterization on $\boldsymbol{\mu}$ and $\boldsymbol{\Sigma}$. Because the mixing coefficient cannot be updated with explicit reparameterization to the best of our knowledge, we also implemented DeLiGAN with U2C loss and implicit reparameterization on $\boldsymbol{\rho}$ using Equation \ref{eq:rho}. In Table \ref{tab:ablation_implicit}, SLOGAN using implicit reparameterization outperformed explicit reparameterization. When implicit $\boldsymbol{\rho}$ update was added, the prior collapsed into a single component ($\pi_{y=0}=1$) and ICFID increased. The lower variance of implicit reparameterized gradients prevents the prior from collapsing into a single component and improves the performance. Additional ablation studies and discussions are presented in Appendix \ref{sec:ablation_apndx}.

\subsection{Effects of Probe Data} \label{sec:probe}
\paragraph{CelebA + ResGAN}
We demonstrate that SLOGAN can learn the desired attributes using a small amount of probe data. Among multiple attributes which co-exist in the CelebA dataset, we chose Male (1:1) and Eyeglasses (14:1). We randomly selected 30 probe images for each latent component. $\pi_{y=0}$ represents the learned mixing coefficient that correspond to the latent component associated with faces without the attribute. As shown in Figure \ref{fig:celeba} and Table \ref{tab:celeba}, we observed that SLOGAN learned the desired attributes. Additional experiments on attribute manipulation are shown in Appendix \ref{sec:attribute_manipulation_analysis_apndx}.

\paragraph{CelebA-HQ + StyleGAN2}
StyleGAN2 \citep{stylegan2} differs from other GANs in that the latent vectors are used for \textit{style}. Despite this difference, the implicit reparameterization and U2C loss can be applied to the input space of the mapping network. On the CelebA-HQ dataset, we used 30 male and 30 female faces as probe data. As shown in Figure \ref{fig:qualitative} (b), SLOGAN successfully performed on high-resolution images and a recent architecture, even simultaneously with imbalanced attributes.

\begin{figure}[t!]
\centering
\begin{minipage}{0.6\linewidth}
\centering
\includegraphics[width=1.0\columnwidth]{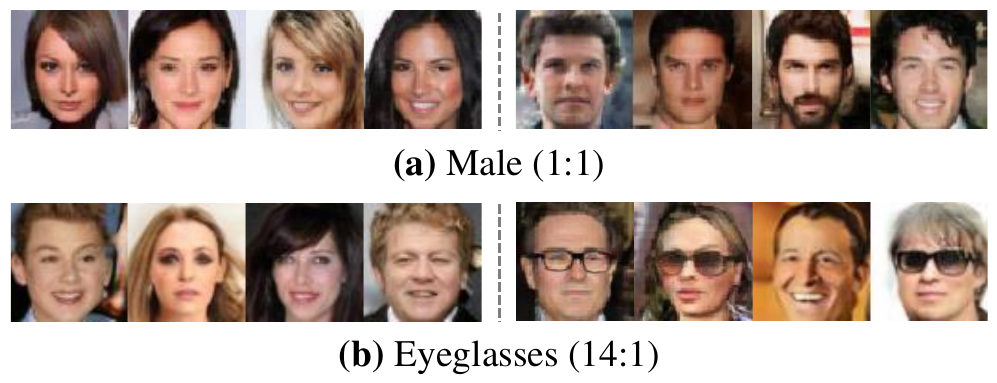} 
\caption{Qualitative results of SLOGAN on CelebA.}
\label{fig:celeba}
\end{minipage}
\hfill
\begin{minipage}{0.35\linewidth}
\setlength\tabcolsep{0pt}
\captionof{table}{Quantitative results of SLOGAN on CelebA}
\label{tab:celeba}
  \begin{tabular*}{\linewidth}{@{\extracolsep{\fill}}ccc}
	\toprule
     & Male & Eyeglasses \\
    \footnotesize{Imb. ratio} & \footnotesize{(1:1)} & \footnotesize{(14:1)} \\
    \midrule
    \footnotesize{NMI} $\uparrow$ & \footnotesize{0.65$\pm$0.01} & \footnotesize{0.29$\pm$0.07} \\
   \footnotesize{FID} $\downarrow$ & \footnotesize{5.18$\pm$0.20} & \footnotesize{5.83$\pm$0.44} \\
    \footnotesize{ICFID} $\downarrow$ & \footnotesize{11.00$\pm$0.66} & \footnotesize{35.57$\pm$5.10} \\
    \footnotesize{$\pi_{y=0}$} & \footnotesize{0.56$\pm$0.02} & \footnotesize{0.82$\pm$0.04} \\
    \bottomrule
\end{tabular*}

\end{minipage}
\end{figure}



\section{Conclusion} \label{sec:conclusion}
We have proposed a method called SLOGAN to generate data conditioned on learned attributes on real-world datasets with balanced or imbalanced attributes. We derive implicit reparameterization for the parameters of the latent distribution. We then proposed a GAN framework and unsupervised conditional contrastive loss (U2C loss). We verified that SLOGAN achieved state-of-the-art unsupervised conditional generation performance. In addition, a small amount of probe data helps SLOGAN control attributes. In future work, we will consider a principled method to learn the number and hierarchy of attributes in real-world data. In addition, improving the quality of samples with minority attributes is an important avenue for future research on unsupervised conditional GANs.

\newpage

\section*{Acknowledgement}
This work was supported by the BK21 FOUR program of the Education and Research Program for Future ICT Pioneers, Seoul National University in 2021, Institute of Information \& communications Technology Planning \& Evaluation (IITP) grant funded by the Korea government(MSIT) [NO.2021-0-01343, Artificial Intelligence Graduate School Program (Seoul National University)], and AIR Lab (AI Research Lab) in Hyundai Motor Company through HMC-SNU AI Consortium Fund.

\bibliography{iclr2022_conference}

\begin{thebibliography}{34}
\providecommand{\natexlab}[1]{#1}
\providecommand{\url}[1]{\texttt{#1}}
\expandafter\ifx\csname urlstyle\endcsname\relax
  \providecommand{\doi}[1]{doi: #1}\else
  \providecommand{\doi}{doi: \begingroup \urlstyle{rm}\Url}\fi

\bibitem[Abadi et~al.(2016)Abadi, Barham, Chen, Chen, Davis, Dean, Devin,
  Ghemawat, Irving, Isard, et~al.]{tensorflow_apndx}
Mart{\'\i}n Abadi, Paul Barham, Jianmin Chen, Zhifeng Chen, Andy Davis, Jeffrey
  Dean, Matthieu Devin, Sanjay Ghemawat, Geoffrey Irving, Michael Isard, et~al.
\newblock Tensorflow: A system for large-scale machine learning.
\newblock In \emph{12th USENIX symposium on operating systems design and
  implementation (OSDI 16)}, pp.\  265--283, 2016.

\bibitem[Arjovsky et~al.(2017)Arjovsky, Chintala, and Bottou]{wgan_apndx}
Martin Arjovsky, Soumith Chintala, and L{\'e}on Bottou.
\newblock Wasserstein generative adversarial networks.
\newblock In \emph{International conference on machine learning}, pp.\
  214--223. PMLR, 2017.

\bibitem[Armandpour et~al.(2021)Armandpour, Sadeghian, Li, and
  Zhou]{pgmgan_apndx}
Mohammadreza Armandpour, Ali Sadeghian, Chunyuan Li, and Mingyuan Zhou.
\newblock Partition-guided gans.
\newblock In \emph{Proceedings of the IEEE/CVF Conference on Computer Vision
  and Pattern Recognition}, pp.\  5099--5109, 2021.

\bibitem[Belghazi et~al.(2018)Belghazi, Rajeswar, Mastropietro, Rostamzadeh,
  Mitrovic, and Courville]{hierarchical_ali_apndx}
Mohamed~Ishmael Belghazi, Sai Rajeswar, Olivier Mastropietro, Negar
  Rostamzadeh, Jovana Mitrovic, and Aaron Courville.
\newblock Hierarchical adversarially learned inference.
\newblock \emph{arXiv preprint arXiv:1802.01071}, 2018.

\bibitem[Bengio et~al.(2013)Bengio, Courville, and
  Vincent]{bengio2013representation_apndx}
Yoshua Bengio, Aaron Courville, and Pascal Vincent.
\newblock Representation learning: A review and new perspectives.
\newblock \emph{IEEE transactions on pattern analysis and machine
  intelligence}, 35\penalty0 (8):\penalty0 1798--1828, 2013.

\bibitem[Chen et~al.(2020)Chen, Kornblith, Norouzi, and Hinton]{simclr_apndx}
Ting Chen, Simon Kornblith, Mohammad Norouzi, and Geoffrey Hinton.
\newblock A simple framework for contrastive learning of visual
  representations.
\newblock In \emph{International conference on machine learning}, pp.\
  1597--1607. PMLR, 2020.

\bibitem[Chen et~al.(2016)Chen, Duan, Houthooft, Schulman, Sutskever, and
  Abbeel]{infogan_apndx}
Xi~Chen, Yan Duan, Rein Houthooft, John Schulman, Ilya Sutskever, and Pieter
  Abbeel.
\newblock Infogan: interpretable representation learning by information
  maximizing generative adversarial nets.
\newblock In \emph{Neural Information Processing Systems (NIPS)}, 2016.

\bibitem[Cohen(2013)]{cohen}
Jacob Cohen.
\newblock \emph{Statistical power analysis for the behavioral sciences}.
\newblock Routledge, 2013.

\bibitem[Deng et~al.(2019)Deng, Guo, Xue, and Zafeiriou]{arcface_apndx}
Jiankang Deng, Jia Guo, Niannan Xue, and Stefanos Zafeiriou.
\newblock Arcface: Additive angular margin loss for deep face recognition.
\newblock In \emph{Proceedings of the IEEE/CVF Conference on Computer Vision
  and Pattern Recognition}, pp.\  4690--4699, 2019.

\bibitem[Donahue et~al.(2017)Donahue, Kr{\"{a}}henb{\"{u}}hl, and
  Darrell]{bigan_apndx}
Jeff Donahue, Philipp Kr{\"{a}}henb{\"{u}}hl, and Trevor Darrell.
\newblock Adversarial feature learning.
\newblock In \emph{5th International Conference on Learning Representations},
  2017.

\bibitem[Dumoulin et~al.(2017)Dumoulin, Belghazi, Poole, Lamb, Arjovsky,
  Mastropietro, and Courville]{ali_apndx}
Vincent Dumoulin, Ishmael Belghazi, Ben Poole, Alex Lamb, Mart{\'{\i}}n
  Arjovsky, Olivier Mastropietro, and Aaron~C. Courville.
\newblock Adversarially learned inference.
\newblock In \emph{5th International Conference on Learning Representations},
  2017.

\bibitem[Heusel et~al.(2017)Heusel, Ramsauer, Unterthiner, Nessler, and
  Hochreiter]{ttur_apndx}
Martin Heusel, Hubert Ramsauer, Thomas Unterthiner, Bernhard Nessler, and Sepp
  Hochreiter.
\newblock Gans trained by a two time-scale update rule converge to a local nash
  equilibrium.
\newblock In \emph{Proceedings of the 31st International Conference on Neural
  Information Processing Systems}, pp.\  6629--6640, 2017.

\bibitem[Jaiswal et~al.(2018)Jaiswal, AbdAlmageed, Wu, and
  Natarajan]{bicogan_apndx}
Ayush Jaiswal, Wael AbdAlmageed, Yue Wu, and Premkumar Natarajan.
\newblock Bidirectional conditional generative adversarial networks.
\newblock In \emph{Asian Conference on Computer Vision}, pp.\  216--232.
  Springer, 2018.

\bibitem[Kang \& Park(2020)Kang and Park]{contragan_apndx}
Minguk Kang and Jaesik Park.
\newblock Contragan: Contrastive learning for conditional image generation.
\newblock In \emph{Advances in Neural Information Processing Systems}, 2020.

\bibitem[Kingma \& Ba(2014)Kingma and Ba]{adam_apndx}
Diederik~P Kingma and Jimmy Ba.
\newblock Adam: A method for stochastic optimization.
\newblock \emph{arXiv preprint arXiv:1412.6980}, 2014.

\bibitem[Krizhevsky et~al.(2009)Krizhevsky, Hinton, et~al.]{cifar10_apndx}
Alex Krizhevsky, Geoffrey Hinton, et~al.
\newblock Learning multiple layers of features from tiny images.
\newblock 2009.

\bibitem[LeCun et~al.(1998)LeCun, Bottou, Bengio, and Haffner]{mnist_apndx}
Yann LeCun, L{\'e}on Bottou, Yoshua Bengio, and Patrick Haffner.
\newblock Gradient-based learning applied to document recognition.
\newblock \emph{Proceedings of the IEEE}, 86\penalty0 (11):\penalty0
  2278--2324, 1998.

\bibitem[Lin et~al.(2019)Lin, Khan, and Schmidt]{lin2019stein_apndx}
Wu~Lin, Mohammad~Emtiyaz Khan, and Mark Schmidt.
\newblock Stein's lemma for the reparameterization trick with exponential
  family mixtures.
\newblock \emph{arXiv preprint arXiv:1910.13398}, 2019.

\bibitem[Liu et~al.(2020)Liu, Wang, Bau, Zhu, and Torralba]{scgan_apndx}
Steven Liu, Tongzhou Wang, David Bau, Jun-Yan Zhu, and Antonio Torralba.
\newblock Diverse image generation via self-conditioned gans.
\newblock In \emph{Proceedings of the IEEE/CVF Conference on Computer Vision
  and Pattern Recognition}, pp.\  14286--14295, 2020.

\bibitem[Liu et~al.(2015)Liu, Luo, Wang, and Tang]{celeba_apndx}
Ziwei Liu, Ping Luo, Xiaogang Wang, and Xiaoou Tang.
\newblock Deep learning face attributes in the wild.
\newblock In \emph{Proceedings of International Conference on Computer Vision
  (ICCV)}, December 2015.

\bibitem[Mishra et~al.(2020)Mishra, Jayendran, and Prathosh]{nemgan_apndx}
Deepak Mishra, Aravind Jayendran, and AP~Prathosh.
\newblock Effect of the latent structure on clustering with gans.
\newblock \emph{IEEE Signal Processing Letters}, 27:\penalty0 900--904, 2020.

\bibitem[Miyato \& Koyama(2018)Miyato and Koyama]{projection_apndx}
Takeru Miyato and Masanori Koyama.
\newblock cgans with projection discriminator.
\newblock In \emph{International Conference on Learning Representations}, 2018.

\bibitem[Mukherjee et~al.(2019)Mukherjee, Asnani, Lin, and
  Kannan]{clustergan_apndx}
Sudipto Mukherjee, Himanshu Asnani, Eugene Lin, and Sreeram Kannan.
\newblock Clustergan: Latent space clustering in generative adversarial
  networks.
\newblock In \emph{Proceedings of the AAAI conference on artificial
  intelligence}, volume~33, pp.\  4610--4617, 2019.

\bibitem[Odena et~al.(2017)Odena, Olah, and Shlens]{acgan_apndx}
Augustus Odena, Christopher Olah, and Jonathon Shlens.
\newblock Conditional image synthesis with auxiliary classifier gans.
\newblock In \emph{International conference on machine learning}, pp.\
  2642--2651. PMLR, 2017.

\bibitem[Pan et~al.(2021)Pan, Tang, Chen, and Xu]{cdgan_apndx}
Lili Pan, Peijun Tang, Zhiyong Chen, and Zenglin Xu.
\newblock Contrastive disentanglement in generative adversarial networks.
\newblock \emph{arXiv preprint arXiv:2103.03636}, 2021.

\bibitem[Petzka et~al.(2018)Petzka, Fischer, and Lukovnikov]{wgan-lp_apndx}
Henning Petzka, Asja Fischer, and Denis Lukovnikov.
\newblock On the regularization of wasserstein {GAN}s.
\newblock In \emph{International Conference on Learning Representations}, 2018.

\bibitem[Shmelkov et~al.(2018)Shmelkov, Schmid, and
  Alahari]{shmelkov2018good_apndx}
Konstantin Shmelkov, Cordelia Schmid, and Karteek Alahari.
\newblock How good is my gan?
\newblock In \emph{Proceedings of the European Conference on Computer Vision
  (ECCV)}, pp.\  213--229, 2018.

\bibitem[Terjék(2020)]{wgan-alp_apndx}
Dávid Terjék.
\newblock Adversarial lipschitz regularization.
\newblock In \emph{International Conference on Learning Representations}, 2020.

\bibitem[Van~Gansbeke et~al.(2020)Van~Gansbeke, Vandenhende, Georgoulis,
  Proesmans, and Van~Gool]{scan_apndx}
Wouter Van~Gansbeke, Simon Vandenhende, Stamatios Georgoulis, Marc Proesmans,
  and Luc Van~Gool.
\newblock Scan: Learning to classify images without labels.
\newblock In \emph{European Conference on Computer Vision}, pp.\  268--285.
  Springer, 2020.

\bibitem[Wilcoxon(1992)]{wilcoxon}
Frank Wilcoxon.
\newblock Individual comparisons by ranking methods.
\newblock In \emph{Breakthroughs in statistics}, pp.\  196--202. Springer,
  1992.

\bibitem[Xiao et~al.(2017)Xiao, Rasul, and Vollgraf]{fmnist_apndx}
Han Xiao, Kashif Rasul, and Roland Vollgraf.
\newblock Fashion-mnist: a novel image dataset for benchmarking machine
  learning algorithms.
\newblock \emph{arXiv preprint arXiv:1708.07747}, 2017.

\bibitem[Zhang et~al.(2018)Zhang, Cisse, Dauphin, and Lopez-Paz]{mixup_apndx}
Hongyi Zhang, Moustapha Cisse, Yann~N. Dauphin, and David Lopez-Paz.
\newblock mixup: Beyond empirical risk minimization.
\newblock In \emph{International Conference on Learning Representations}, 2018.

\bibitem[Zhao et~al.(2020)Zhao, Liu, Lin, Zhu, and Han]{diffaugment_apndx}
Shengyu Zhao, Zhijian Liu, Ji~Lin, Jun-Yan Zhu, and Song Han.
\newblock Differentiable augmentation for data-efficient gan training.
\newblock \emph{Advances in Neural Information Processing Systems}, 33, 2020.

\bibitem[Zheng et~al.(2017)Zheng, Terry, Belgrader, Ryvkin, Bent, Wilson,
  Ziraldo, Wheeler, McDermott, Zhu, et~al.]{10x73k_apndx}
Grace~XY Zheng, Jessica~M Terry, Phillip Belgrader, Paul Ryvkin, Zachary~W
  Bent, Ryan Wilson, Solongo~B Ziraldo, Tobias~D Wheeler, Geoff~P McDermott,
  Junjie Zhu, et~al.
\newblock Massively parallel digital transcriptional profiling of single cells.
\newblock \emph{Nature communications}, 8\penalty0 (1):\penalty0 1--12, 2017.

\end{thebibliography}


\begin{thebibliography}{37}
\providecommand{\natexlab}[1]{#1}
\providecommand{\url}[1]{\texttt{#1}}
\expandafter\ifx\csname urlstyle\endcsname\relax
  \providecommand{\doi}[1]{doi: #1}\else
  \providecommand{\doi}{doi: \begingroup \urlstyle{rm}\Url}\fi

\bibitem[Armandpour et~al.(2021)Armandpour, Sadeghian, Li, and Zhou]{pgmgan}
Mohammadreza Armandpour, Ali Sadeghian, Chunyuan Li, and Mingyuan Zhou.
\newblock Partition-guided gans.
\newblock In \emph{Proceedings of the IEEE/CVF Conference on Computer Vision
  and Pattern Recognition}, pp.\  5099--5109, 2021.

\bibitem[Bonnet(1964)]{bonnet}
Georges Bonnet.
\newblock Transformations des signaux al{\'e}atoires a travers les systemes non
  lin{\'e}aires sans m{\'e}moire.
\newblock In \emph{Annales des T{\'e}l{\'e}communications}, volume~19, pp.\
  203--220. Springer, 1964.

\bibitem[Chen et~al.(2020)Chen, Kornblith, Norouzi, and Hinton]{simclr}
Ting Chen, Simon Kornblith, Mohammad Norouzi, and Geoffrey Hinton.
\newblock A simple framework for contrastive learning of visual
  representations.
\newblock In \emph{International conference on machine learning}, pp.\
  1597--1607. PMLR, 2020.

\bibitem[Chen et~al.(2016)Chen, Duan, Houthooft, Schulman, Sutskever, and
  Abbeel]{infogan}
Xi~Chen, Yan Duan, Rein Houthooft, John Schulman, Ilya Sutskever, and Pieter
  Abbeel.
\newblock Infogan: interpretable representation learning by information
  maximizing generative adversarial nets.
\newblock In \emph{Neural Information Processing Systems (NIPS)}, 2016.

\bibitem[Choi et~al.(2020)Choi, Uh, Yoo, and Ha]{afhq}
Yunjey Choi, Youngjung Uh, Jaejun Yoo, and Jung-Woo Ha.
\newblock Stargan v2: Diverse image synthesis for multiple domains.
\newblock In \emph{Proceedings of the IEEE Conference on Computer Vision and
  Pattern Recognition}, 2020.

\bibitem[Figurnov et~al.(2018)Figurnov, Mohamed, and Mnih]{implicit}
Mikhail Figurnov, Shakir Mohamed, and Andriy Mnih.
\newblock Implicit reparameterization gradients.
\newblock \emph{Advances in Neural Information Processing Systems}, 31, 2018.

\bibitem[Gulrajani et~al.(2017)Gulrajani, Ahmed, Arjovsky, Dumoulin, and
  Courville]{wgan-gp}
Ishaan Gulrajani, Faruk Ahmed, Martin Arjovsky, Vincent Dumoulin, and Aaron
  Courville.
\newblock Improved training of wasserstein gans.
\newblock In \emph{Proceedings of the 31st International Conference on Neural
  Information Processing Systems}, pp.\  5769--5779, 2017.

\bibitem[Gurumurthy et~al.(2017)Gurumurthy, Kiran~Sarvadevabhatla, and
  Venkatesh~Babu]{deligan}
Swaminathan Gurumurthy, Ravi Kiran~Sarvadevabhatla, and R~Venkatesh~Babu.
\newblock Deligan: Generative adversarial networks for diverse and limited
  data.
\newblock In \emph{Proceedings of the IEEE conference on computer vision and
  pattern recognition}, pp.\  166--174, 2017.

\bibitem[Hadsell et~al.(2006)Hadsell, Chopra, and
  LeCun]{hadsell2006dimensionality}
Raia Hadsell, Sumit Chopra, and Yann LeCun.
\newblock Dimensionality reduction by learning an invariant mapping.
\newblock In \emph{2006 IEEE Computer Society Conference on Computer Vision and
  Pattern Recognition (CVPR'06)}, volume~2, pp.\  1735--1742. IEEE, 2006.

\bibitem[Heusel et~al.(2017)Heusel, Ramsauer, Unterthiner, Nessler, and
  Hochreiter]{ttur}
Martin Heusel, Hubert Ramsauer, Thomas Unterthiner, Bernhard Nessler, and Sepp
  Hochreiter.
\newblock Gans trained by a two time-scale update rule converge to a local nash
  equilibrium.
\newblock In \emph{Proceedings of the 31st International Conference on Neural
  Information Processing Systems}, pp.\  6629--6640, 2017.

\bibitem[Hwang et~al.(2019)Hwang, Jung, and Yoon]{hexagan}
Uiwon Hwang, Dahuin Jung, and Sungroh Yoon.
\newblock Hexagan: Generative adversarial nets for real world classification.
\newblock In \emph{International Conference on Machine Learning}, pp.\
  2921--2930. PMLR, 2019.

\bibitem[Jang et~al.(2017)Jang, Gu, and Poole]{gumbel}
Eric Jang, Shixiang Gu, and Ben Poole.
\newblock Categorical reparameterization with gumbel-softmax.
\newblock In \emph{International Conference on Learning Representations}, 2017.

\bibitem[Kang et~al.(2020)Kang, Xie, Rohrbach, Yan, Gordo, Feng, and
  Kalantidis]{contrastive_decoupling}
Bingyi Kang, Saining Xie, Marcus Rohrbach, Zhicheng Yan, Albert Gordo, Jiashi
  Feng, and Yannis Kalantidis.
\newblock Decoupling representation and classifier for long-tailed recognition.
\newblock In \emph{International Conference on Learning Representations}, 2020.

\bibitem[Kang et~al.(2021)Kang, Li, Xie, Yuan, and Feng]{contrastive_imbalance}
Bingyi Kang, Yu~Li, Sa~Xie, Zehuan Yuan, and Jiashi Feng.
\newblock Exploring balanced feature spaces for representation learning.
\newblock In \emph{International Conference on Learning Representations}, 2021.

\bibitem[Kang \& Park(2020)Kang and Park]{contragan}
Minguk Kang and Jaesik Park.
\newblock Contragan: Contrastive learning for conditional image generation.
\newblock In \emph{Advances in Neural Information Processing Systems}, 2020.

\bibitem[Karras et~al.(2017)Karras, Aila, Laine, and Lehtinen]{celeba-hq}
Tero Karras, Timo Aila, Samuli Laine, and Jaakko Lehtinen.
\newblock Progressive growing of gans for improved quality, stability, and
  variation.
\newblock \emph{arXiv preprint arXiv:1710.10196}, 2017.

\bibitem[Karras et~al.(2020)Karras, Laine, Aittala, Hellsten, Lehtinen, and
  Aila]{stylegan2}
Tero Karras, Samuli Laine, Miika Aittala, Janne Hellsten, Jaakko Lehtinen, and
  Timo Aila.
\newblock Analyzing and improving the image quality of stylegan.
\newblock In \emph{Proceedings of the IEEE/CVF Conference on Computer Vision
  and Pattern Recognition}, pp.\  8110--8119, 2020.

\bibitem[Krizhevsky et~al.(2009)Krizhevsky, Hinton, et~al.]{cifar10}
Alex Krizhevsky, Geoffrey Hinton, et~al.
\newblock Learning multiple layers of features from tiny images.
\newblock 2009.

\bibitem[LeCun et~al.(1998)LeCun, Bottou, Bengio, and Haffner]{mnist}
Yann LeCun, L{\'e}on Bottou, Yoshua Bengio, and Patrick Haffner.
\newblock Gradient-based learning applied to document recognition.
\newblock \emph{Proceedings of the IEEE}, 86\penalty0 (11):\penalty0
  2278--2324, 1998.

\bibitem[Lin et~al.(2019{\natexlab{a}})Lin, Khan, and Schmidt]{lin2019fast}
Wu~Lin, Mohammad~Emtiyaz Khan, and Mark Schmidt.
\newblock Fast and simple natural-gradient variational inference with mixture
  of exponential-family approximations.
\newblock In \emph{International Conference on Machine Learning}, pp.\
  3992--4002. PMLR, 2019{\natexlab{a}}.

\bibitem[Lin et~al.(2019{\natexlab{b}})Lin, Khan, and Schmidt]{lin2019stein}
Wu~Lin, Mohammad~Emtiyaz Khan, and Mark Schmidt.
\newblock Stein's lemma for the reparameterization trick with exponential
  family mixtures.
\newblock \emph{arXiv preprint arXiv:1910.13398}, 2019{\natexlab{b}}.

\bibitem[Lin et~al.(2020)Lin, Schmidt, and Khan]{pmlr-v119-lin20d}
Wu~Lin, Mark Schmidt, and Mohammad~Emtiyaz Khan.
\newblock Handling the positive-definite constraint in the {B}ayesian learning
  rule.
\newblock In \emph{Proceedings of the 37th International Conference on Machine
  Learning}, pp.\  6116--6126, 2020.

\bibitem[Liu et~al.(2020)Liu, Wang, Bau, Zhu, and Torralba]{scgan}
Steven Liu, Tongzhou Wang, David Bau, Jun-Yan Zhu, and Antonio Torralba.
\newblock Diverse image generation via self-conditioned gans.
\newblock In \emph{Proceedings of the IEEE/CVF Conference on Computer Vision
  and Pattern Recognition}, pp.\  14286--14295, 2020.

\bibitem[Liu et~al.(2015)Liu, Luo, Wang, and Tang]{celeba}
Ziwei Liu, Ping Luo, Xiaogang Wang, and Xiaoou Tang.
\newblock Deep learning face attributes in the wild.
\newblock In \emph{Proceedings of International Conference on Computer Vision
  (ICCV)}, December 2015.

\bibitem[Miyato \& Koyama(2018)Miyato and Koyama]{projection}
Takeru Miyato and Masanori Koyama.
\newblock cgans with projection discriminator.
\newblock In \emph{International Conference on Learning Representations}, 2018.

\bibitem[Mukherjee et~al.(2019)Mukherjee, Asnani, Lin, and Kannan]{clustergan}
Sudipto Mukherjee, Himanshu Asnani, Eugene Lin, and Sreeram Kannan.
\newblock Clustergan: Latent space clustering in generative adversarial
  networks.
\newblock In \emph{Proceedings of the AAAI conference on artificial
  intelligence}, volume~33, pp.\  4610--4617, 2019.

\bibitem[Odena et~al.(2017)Odena, Olah, and Shlens]{acgan}
Augustus Odena, Christopher Olah, and Jonathon Shlens.
\newblock Conditional image synthesis with auxiliary classifier gans.
\newblock In \emph{International conference on machine learning}, pp.\
  2642--2651. PMLR, 2017.

\bibitem[Pan et~al.(2021)Pan, Tang, Chen, and Xu]{cdgan}
Lili Pan, Peijun Tang, Zhiyong Chen, and Zenglin Xu.
\newblock Contrastive disentanglement in generative adversarial networks.
\newblock \emph{arXiv preprint arXiv:2103.03636}, 2021.

\bibitem[Poole et~al.(2019)Poole, Ozair, Van Den~Oord, Alemi, and
  Tucker]{poole2019variational}
Ben Poole, Sherjil Ozair, Aaron Van Den~Oord, Alex Alemi, and George Tucker.
\newblock On variational bounds of mutual information.
\newblock In \emph{International Conference on Machine Learning}, pp.\
  5171--5180. PMLR, 2019.

\bibitem[Price(1958)]{price}
Robert Price.
\newblock A useful theorem for nonlinear devices having gaussian inputs.
\newblock \emph{IRE Transactions on Information Theory}, 4\penalty0
  (2):\penalty0 69--72, 1958.

\bibitem[Radford et~al.(2016)Radford, Metz, and Chintala]{dcgan}
Alec Radford, Luke Metz, and Soumith Chintala.
\newblock Unsupervised representation learning with deep convolutional
  generative adversarial networks.
\newblock In \emph{International Conference on Learning Representations}, 2016.

\bibitem[Wang et~al.(2020)Wang, Wang, Ramkumar, Mardziel, Fredrikson, and
  Datta]{smoothed}
Zifan Wang, Haofan Wang, Shakul Ramkumar, Piotr Mardziel, Matt Fredrikson, and
  Anupam Datta.
\newblock Smoothed geometry for robust attribution.
\newblock In \emph{Advances in Neural Information Processing Systems}, 2020.

\bibitem[Wanyan et~al.(2021)Wanyan, Zhang, Ding, Azad, Wang, and
  Glicksberg]{contrastive_bootstrap}
Tingyi Wanyan, Jing Zhang, Ying Ding, Ariful Azad, Zhangyang Wang, and
  Benjamin~S Glicksberg.
\newblock Bootstrapping your own positive sample: Contrastive learning with
  electronic health record data.
\newblock \emph{arXiv preprint arXiv:2104.02932}, 2021.

\bibitem[Xiao et~al.(2017)Xiao, Rasul, and Vollgraf]{fmnist}
Han Xiao, Kashif Rasul, and Roland Vollgraf.
\newblock Fashion-mnist: a novel image dataset for benchmarking machine
  learning algorithms.
\newblock \emph{arXiv preprint arXiv:1708.07747}, 2017.

\bibitem[Zhang et~al.(2018)Zhang, Cisse, Dauphin, and Lopez-Paz]{mixup}
Hongyi Zhang, Moustapha Cisse, Yann~N. Dauphin, and David Lopez-Paz.
\newblock mixup: Beyond empirical risk minimization.
\newblock In \emph{International Conference on Learning Representations}, 2018.

\bibitem[Zhao et~al.(2020)Zhao, Liu, Lin, Zhu, and Han]{diffaugment}
Shengyu Zhao, Zhijian Liu, Ji~Lin, Jun-Yan Zhu, and Song Han.
\newblock Differentiable augmentation for data-efficient gan training.
\newblock \emph{Advances in Neural Information Processing Systems}, 33, 2020.

\bibitem[Zhong et~al.(2020)Zhong, Chen, Jin, and Hua]{contrastive_clustering}
Huasong Zhong, Chong Chen, Zhongming Jin, and Xian-Sheng Hua.
\newblock Deep robust clustering by contrastive learning.
\newblock \emph{arXiv preprint arXiv:2008.03030}, 2020.

\end{thebibliography}
\bibliographystyle{iclr2022_conference}

\newpage

\appendix
\newtheorem{thm}{Theorem}
\setcounter{thm}{0}

\section{Additional Results and Discussion} \label{sec:additional_apndx}
We also compared the proposed method with WGAN \citepapndx{wgan_apndx} with the proposed method. We used the pre-activation of the penultimate layer of $D$ for the clustering metrics because WGAN has no encoder network. We could not measure ICFID of WGAN because it cannot perform unsupervised conditional generation. In addition to various datasets, we also used the 10x\_73k dataset \citepapndx{10x73k_apndx}, which consists of RNA transcript counts. From the results of the clustering performances on the 10x 73k dataset, we show that SLOGAN learns useful imbalanced attributes and can be helpful in the use of unlabeled biomedical data.

\subsection{Performance Comparison on Balanced Attributes} \label{sec:balance_apndx}
\begin{table}[h!]
\centering
\caption{Performance comparison on balanced attributes}
\label{tab:balance_apndx}
{\resizebox{1\columnwidth}{!}
{\begin{tabular}{cccccccccccc}
\hline
	\toprule
    Dataset & Metric & WGAN & InfoGAN & DeLiGAN & DeLiGAN+ & ClusterGAN & SCGAN & CD-GAN  & PGMGAN & \textbf{SLOGAN} \\
    \midrule
    \multirow{3}*{\centering MNIST}
    & NMI & 0.78\footnotesize{$\pm$0.02} & 0.90\footnotesize{$\pm$0.03} & 0.70\footnotesize{$\pm$0.05} & 0.77\footnotesize{$\pm$0.05} & 0.81\footnotesize{$\pm$0.02} & 0.74\footnotesize{$\pm$0.06} & 0.87\footnotesize{$\pm$0.03} & 0.16\footnotesize{$\pm$0.27} & \textbf{0.92\footnotesize{$\pm$0.00}} \\
    & FID & 3.05\footnotesize{$\pm$0.20} & 1.72\footnotesize{$\pm$0.17} & 1.92\footnotesize{$\pm$0.12} & 2.00\footnotesize{$\pm$0.16} & 1.71\footnotesize{$\pm$0.07} & 3.06\footnotesize{$\pm$0.53} & 2.75\footnotesize{$\pm$0.04} & 5.76\footnotesize{$\pm$1.67} & \textbf{1.67\footnotesize{$\pm$0.15}} \\
    & ICFID & N/A & 5.56\footnotesize{$\pm$0.71} & 5.74\footnotesize{$\pm$0.25} & 5.64\footnotesize{$\pm$0.39} & 5.12\footnotesize{$\pm$0.07} & 16.65\footnotesize{$\pm$2.01} & 7.03\footnotesize{$\pm$0.23} & 53.40\footnotesize{$\pm$12.49} & \textbf{4.99\footnotesize{$\pm$0.19}} \\
    \midrule 
    \multirow{3}*{\centering FMNIST}
    & NMI & 0.65\footnotesize{$\pm$0.02} & 0.64\footnotesize{$\pm$0.02} & 0.64\footnotesize{$\pm$0.03} & 0.57\footnotesize{$\pm$0.07} & 0.61\footnotesize{$\pm$0.03} & 0.56\footnotesize{$\pm$0.01} & 0.56\footnotesize{$\pm$0.04} & 0.47\footnotesize{$\pm$0.01} & \textbf{0.66\footnotesize{$\pm$0.01}} \\
    & FID & 5.74\footnotesize{$\pm$0.49} & 5.28\footnotesize{$\pm$0.12} & 6.65\footnotesize{$\pm$0.48} & 7.23\footnotesize{$\pm$0.56} & 6.32\footnotesize{$\pm$0.25} & \textbf{5.07\footnotesize{$\pm$0.19}} & 9.05\footnotesize{$\pm$0.11} & 9.13\footnotesize{$\pm$0.28} & 5.20\footnotesize{$\pm$0.36} \\
    & ICFID & N/A & 32.18\footnotesize{$\pm$2.11} & 34.87\footnotesize{$\pm$5.29} & 30.53\footnotesize{$\pm$8.71} & 37.20\footnotesize{$\pm$5.50} & 26.23\footnotesize{$\pm$7.10} & 36.61\footnotesize{$\pm$0.47} & 40.00\footnotesize{$\pm$4.38} & \textbf{23.31\footnotesize{$\pm$2.77}} \\
    \midrule
    \multirow{3}*{\centering CIFAR-2}
    & NMI & 0.14\footnotesize{$\pm$0.02} & 0.05\footnotesize{$\pm$0.03} & 0.15\footnotesize{$\pm$0.13} & 0.12\footnotesize{$\pm$0.12} & 0.34\footnotesize{$\pm$0.02} & 0.00\footnotesize{$\pm$0.00} & 0.38\footnotesize{$\pm$0.01} & 0.67\footnotesize{$\pm$0.00} & \textbf{0.78\footnotesize{$\pm$0.03}} \\
    & FID & 29.54\footnotesize{$\pm$0.59} & 58.84\footnotesize{$\pm$13.11} & \footnotesize{338.97}\footnotesize{$\pm$70.85} & \footnotesize{116.95}\footnotesize{$\pm$19.42} & 36.28\footnotesize{$\pm$1.12} & 39.44\footnotesize{$\pm$1.72} & 34.45\footnotesize{$\pm$0.74} & 29.49\footnotesize{$\pm$0.51} & \textbf{28.99\footnotesize{$\pm$0.36}} \\
    & ICFID & N/A & 91.97\footnotesize{$\pm$14.21} & \footnotesize{361.66}\footnotesize{$\pm$71.28} & \footnotesize{153.19}\footnotesize{$\pm$17.71} & 47.02\footnotesize{$\pm$1.85} & 71.54\footnotesize{$\pm$5.41} & 43.98\footnotesize{$\pm$1.47} & \textbf{35.67\footnotesize{$\pm$0.61}} & 35.68\footnotesize{$\pm$0.51} \\
    \midrule
    \multirow{3}*{\centering CIFAR-10}
    & NMI & 0.27\footnotesize{$\pm$0.05} & 0.03\footnotesize{$\pm$0.00} & 0.06\footnotesize{$\pm$0.00} & 0.09\footnotesize{$\pm$0.04} & 0.10\footnotesize{$\pm$0.00} & 0.01\footnotesize{$\pm$0.00} & 0.03\footnotesize{$\pm$0.01} & 0.29\footnotesize{$\pm$0.02} & \textbf{0.34\footnotesize{$\pm$0.01}} \\
    & FID & \textbf{20.56\footnotesize{$\pm$0.76}} & 81.84\footnotesize{$\pm$2.27} & \footnotesize{212.20}\footnotesize{$\pm$4.52} & \footnotesize{110.51}\footnotesize{$\pm$7.70} & 61.97\footnotesize{$\pm$3.69} & \footnotesize{199.28}\footnotesize{$\pm$57.16} &  34.13\footnotesize{$\pm$1.13} &  31.50\footnotesize{$\pm$0.73} & 20.61\footnotesize{$\pm$0.40} \\
    & ICFID & N/A & \footnotesize{139.20}\footnotesize{$\pm$2.09} & \footnotesize{305.32}\footnotesize{$\pm$5.05} & \footnotesize{215.63}\footnotesize{$\pm$11.16} & \footnotesize{124.27}\footnotesize{$\pm$5.95} & \footnotesize{262.54}\footnotesize{$\pm$59.29} & 
    95.43\footnotesize{$\pm$3.58} & 81.25\footnotesize{$\pm$11.55} & \textbf{71.23\footnotesize{$\pm$6.76}} \\
    \bottomrule
\end{tabular}}
}
\end{table}

\subsection{Performance Comparison on Imalanced Attributes} \label{sec:imbalance_apndx}

\begin{table}[h!]
\centering
\caption{Performance comparison on imbalanced attributes}
\label{tab:imbalance_apndx}
{\resizebox{1\columnwidth}{!}
{\begin{tabular}{ccccccccccc}
\hline
	\toprule
    Dataset & Metric & WGAN & InfoGAN & DeLiGAN & DeLiGAN+ & ClusterGAN & SCGAN & CD-GAN & PGMGAN & \textbf{SLOGAN} \\
    \midrule
    \multirow{3}*{\centering {{\begin{tabular}[c]{@{}c@{}}MNIST-2\\ (7:3)\end{tabular}}}}
    & NMI & 0.90\footnotesize{$\pm$0.03} & 0.28\footnotesize{$\pm$0.19} & 0.90\footnotesize{$\pm$0.04} & 0.48\footnotesize{$\pm$0.09} & 0.27\footnotesize{$\pm$0.19} & 0.67\footnotesize{$\pm$0.11} & 0.41\footnotesize{$\pm$0.03} & 0.79\footnotesize{$\pm$0.21} & \textbf{0.92\footnotesize{$\pm$0.05}} \\
    & FID & 4.27\footnotesize{$\pm$0.19} & 4.92\footnotesize{$\pm$0.85} & 4.21\footnotesize{$\pm$0.84} & 4.63\footnotesize{$\pm$2.02} & 4.25\footnotesize{$\pm$1.06} & 4.34\footnotesize{$\pm$0.73} & 4.67\footnotesize{$\pm$1.92} & 8.90\footnotesize{$\pm$14.82} & \textbf{4.02\footnotesize{$\pm$0.86}} \\
    & ICFID & N/A & 36.35\footnotesize{$\pm$10.65} & 25.34\footnotesize{$\pm$1.72} & 26.61\footnotesize{$\pm$1.49} & 25.41\footnotesize{$\pm$1.02} & 16.47\footnotesize{$\pm$1.51} & 26.71\footnotesize{$\pm$2.47} & 14.82\footnotesize{$\pm$9.16} & \textbf{5.91\footnotesize{$\pm$1.06}} \\
    \midrule 
    \multirow{3}*{\centering FMNIST-5}
    & NMI & 0.65\footnotesize{$\pm$0.00} & 0.58\footnotesize{$\pm$0.07} & \textbf{0.68\footnotesize{$\pm$0.05}} & 0.65\footnotesize{$\pm$0.01} & 0.60\footnotesize{$\pm$0.02} & 0.60\footnotesize{$\pm$0.06} & 0.59\footnotesize{$\pm$0.01} & 0.24\footnotesize{$\pm$0.02} & 0.66\footnotesize{$\pm$0.06} \\
    & FID & 6.55\footnotesize{$\pm$0.20} & 5.40\footnotesize{$\pm$0.14} & 7.05\footnotesize{$\pm$0.49} & 6.33\footnotesize{$\pm$0.44} & 5.61\footnotesize{$\pm$0.17} & \textbf{5.01\footnotesize{$\pm$0.20}} & 9.34\footnotesize{$\pm$0.56} & 11.80\footnotesize{$\pm$0.43} & 5.29\footnotesize{$\pm$0.16} \\
    & ICFID & N/A & 43.69\footnotesize{$\pm$10.84} & 36.21\footnotesize{$\pm$3.07} & 35.41\footnotesize{$\pm$0.79} & 36.94\footnotesize{$\pm$5.81} & 44.48\footnotesize{$\pm$21.62} & 39.31\footnotesize{$\pm$1.18} & 77.30\footnotesize{$\pm$8.60} & \textbf{32.46\footnotesize{$\pm$3.18}} \\
    \midrule
    \multirow{1}*{\centering 10x\_73k}
    & NMI & 0.22\footnotesize{$\pm$0.04} & 0.42\footnotesize{$\pm$0.06} & 0.61\footnotesize{$\pm$0.01} & 0.60\footnotesize{$\pm$0.01} & 0.66\footnotesize{$\pm$0.02} & 0.47\footnotesize{$\pm$0.02} & 0.68\footnotesize{$\pm$0.03} & 0.33\footnotesize{$\pm$0.07} & \textbf{0.76\footnotesize{$\pm$0.02}} \\
    \midrule
    \multirow{3}*{\centering {{\begin{tabular}[c]{@{}c@{}}CIFAR-2\\ (7:3)\end{tabular}}}}
    & NMI & 0.09\footnotesize{$\pm$0.07} & 0.05\footnotesize{$\pm$0.01} & 0.00\footnotesize{$\pm$0.00} & 0.03\footnotesize{$\pm$0.03} & 0.22\footnotesize{$\pm$0.02} & 0.00\footnotesize{$\pm$0.00} & 0.22\footnotesize{$\pm$0.03} & 0.42\footnotesize{$\pm$0.03} & \textbf{0.69\footnotesize{$\pm$0.02}} \\
    & FID & 29.16\footnotesize{$\pm$0.90} & 51.30\footnotesize{$\pm$2.53} & \footnotesize{131.73}\footnotesize{$\pm$50.98} & \footnotesize{115.19}\footnotesize{$\pm$17.95} & 36.62\footnotesize{$\pm$2.16} & 45.28\footnotesize{$\pm$1.81} & 36.40\footnotesize{$\pm$1.01} & 29.76\footnotesize{$\pm$1.65} & \textbf{29.09\footnotesize{$\pm$0.73}} \\
    & ICFID & N/A & 88.49\footnotesize{$\pm$6.85} & \footnotesize{186.31}\footnotesize{$\pm$28.31} & \footnotesize{173.81}\footnotesize{$\pm$18.29} & 75.52\footnotesize{$\pm$4.82} & 88.58\footnotesize{$\pm$4.57} & 76.91\footnotesize{$\pm$1.07} & 57.06\footnotesize{$\pm$3.31} & \textbf{45.83\footnotesize{$\pm$3.03}} \\
    \midrule
    \multirow{3}*{\centering {{\begin{tabular}[c]{@{}c@{}}CIFAR-2\\ (9:1)\end{tabular}}}}
    & NMI & 0.04\footnotesize{$\pm$0.04} & 0.00\footnotesize{$\pm$0.00} & 0.02\footnotesize{$\pm$0.02} & 0.09\footnotesize{$\pm$0.11} & 0.02\footnotesize{$\pm$0.01} & 0.00\footnotesize{$\pm$0.00} & 0.05\footnotesize{$\pm$0.03} & 0.16\footnotesize{$\pm$0.03} & \textbf{0.38\footnotesize{$\pm$0.01}} \\
    & FID & \textbf{29.37\footnotesize{$\pm$0.53}} & 60.76\footnotesize{$\pm$8.97} & \footnotesize{129.50}\footnotesize{$\pm$25.33} & \footnotesize{139.75}\footnotesize{$\pm$47.13} & 41.69\footnotesize{$\pm$0.83} & 50.45\footnotesize{$\pm$1.56} & 38.15\footnotesize{$\pm$2.70} & 30.23\footnotesize{$\pm$1.31} & 29.47\footnotesize{$\pm$1.53} \\
    & ICFID & N/A & \footnotesize{138.24}\footnotesize{$\pm$10.23} & \footnotesize{205.26}\footnotesize{$\pm$10.93} & \footnotesize{196.00}\footnotesize{$\pm$17.86} & \footnotesize{133.31}\footnotesize{$\pm$2.03} & \footnotesize{123.35}\footnotesize{$\pm$6.56} & \footnotesize{128.46}\footnotesize{$\pm$3.03} & \footnotesize{101.68}\footnotesize{$\pm$3.87} & \textbf{86.75\footnotesize{$\pm$1.87}} \\
    \bottomrule
\end{tabular}}
}
\end{table}

\subsection{Ablation Study} \label{sec:ablation_apndx}
\begin{table}[h!]
\centering
\caption{Ablation study on CIFAR-2 (7:3)}
\label{tab:ablation}
{\resizebox{0.85\columnwidth}{!}
{\begin{tabular}{lrr}
    \toprule
    Ablation & $\pi_{y=0}$ \scriptsize{(ground-truth: 0.7)} & ICFID $\downarrow$ \\
    \midrule
    \textbf{Factor analysis} & & \\
    ~~SLOGAN without $\boldsymbol{\mu}$, $\boldsymbol{\Sigma}$, $\boldsymbol{\rho}$ updates, $\ell_\mathrm{U2C}$ & 0.50 & 84.44 \\
    ~~SLOGAN without $\boldsymbol{\mu}$, $\boldsymbol{\Sigma}$, $\boldsymbol{\rho}$ updates & 0.50 & 77.32 \\
    ~~SLOGAN without $\boldsymbol{\mu}$ update & 0.52 & 73.79 \\
    ~~SLOGAN without $\boldsymbol{\rho}$ update & 0.50 & 63.09 \\
    ~~SLOGAN without $\boldsymbol{\Sigma}$ update & 0.69 & 48.34 \\
    ~~SLOGAN without $\ell_\mathrm{U2C}$ & 0.66 & 48.82 \\
    \midrule
    \textbf{Implicit reparameterization} & & \\
    ~~DeLiGAN with $\ell_\mathrm{U2C}$ & 0.50 & 60.51 \\
    ~~DeLiGAN with $\ell_\mathrm{U2C}$ and implicit $\boldsymbol{\rho}$ update & 1.00 & 86.48 \\
    \midrule
    \textbf{Loss for $\boldsymbol{\rho}$ update} & & \\
    ~~SLOGAN with $\ell_\mathrm{U2C}$ for $\boldsymbol{\rho}$ update & 0.62 & 52.67 \\
    \midrule
    \textbf{SimCLR analysis} & & \\
    ~~SLOGAN without SimCLR & 0.66 & 49.25 \\
    ~~SLOGAN without SimCLR on real data only & 0.67 & 48.41 \\
    ~~SLOGAN without SimCLR on both real and fake data & 0.69 & 47.93 \\
    \midrule
    \textbf{Attribute manipulation} & & \\
    ~~SLOGAN with probe data & 0.71 & 44.97 \\
    ~~SLOGAN with probe data and mixup & 0.70 & 44.26 \\
    \midrule
    SLOGAN & 0.69 & 45.83 \\
    \bottomrule
\end{tabular}}}
\end{table}

Table \ref{tab:ablation} shows the ablation study on SLOGAN trained with CIFAR-2 (7:3). $\pi_{y=0}$ and $\pi_{y=1}$ represent the mixing coefficients of the latent components that correspond to the frogs and planes, respectively, and the ground-truth of $\pi_{y=0}$ is 0.7. 

\begin{table}[h!]
\centering
\caption{Effectiveness of U2C loss}
\label{tab:ablation_u2c_apndx}
{\resizebox{0.62\columnwidth}{!}
{\begin{tabular}{crccc}
    \toprule
    Dataset & Ablation & NMI $\uparrow$ & FID $\downarrow$ & ICFID $\downarrow$ \\
    \midrule
    MNIST-2 & SLOGAN w/o $\ell_\mathrm{U2C}$ & 0.25 & 4.62 & 9.43 \\
    (7:3) & SLOGAN & \textbf{0.92} & \textbf{4.02} & \textbf{5.91} \\
    \midrule
    \multirow{2}{*}{FMNIST-5} & SLOGAN w/o $\ell_\mathrm{U2C}$ & 0.14 & \textbf{5.27} & 43.15 \\
    & SLOGAN & \textbf{0.66} & 5.29 & \textbf{32.46} \\
    \midrule
    \multirow{2}{*}{CIFAR-2} & SLOGAN w/o $\ell_\mathrm{U2C}$ & 0.01 & 29.18 & 41.72 \\
    & SLOGAN & \textbf{0.78} & \textbf{28.99} & \textbf{35.68} \\
    \midrule
    CIFAR-2 & SLOGAN w/o $\ell_\mathrm{U2C}$ & 0.08 & 30.34 & 48.82 \\
    \footnotesize{(7:3)} & SLOGAN & \textbf{0.69} & \textbf{29.09} & \textbf{45.83} \\
    \midrule
    \multirow{2}{*}{CIFAR-10} & SLOGAN w/o $\ell_\mathrm{U2C}$ & 0.08 & 20.91 & 78.26 \\
    & SLOGAN & \textbf{0.34} & \textbf{20.61} & \textbf{71.23} \\
    \bottomrule
\end{tabular}}}
\end{table}

\paragraph{Factor analysis}
Rows 1-6 of Table \ref{tab:ablation} compare the performance depending on the factors affecting the performance of SLOGAN ($\boldsymbol{\mu}$, $\boldsymbol{\Sigma}$, $\boldsymbol{\rho}$ updates, and $\ell_\mathrm{U2C}$). We confirmed that SLOGAN with all the factors demonstrated the highest performance. Among the parameters of the latent distribution, the $\boldsymbol{\mu}$ update leads to the highest performance improvement. The intra-cluster Fr\'echet inception distance (ICFID) of SLOGAN without $\boldsymbol{\rho}$ update (the 5th row of Table \ref{tab:ablation}) indicates that SLOGAN outperformed existing unsupervised conditional GANs even when assuming a uniform distribution of the attributes.

\paragraph{Loss for $\boldsymbol{\rho}$ update} \label{sec:rho_update_apndx} 
We do not use U2C loss $\ell_\mathrm{U2C}$ to learn the mixing coefficient parameters $\boldsymbol{\rho}$. We construct U2C loss to approximate the negative mutual information $-I(C;\mathbf{x}_g)$ that can be decomposed into entropy and conditional entropy as follows:
\begin{gather}
\ell_\mathrm{U2C}(\mathbf{z}) \approx -I(C;\mathbf{x}_g)=H(C|\mathbf{x}_g)-H(C)
\end{gather}
The conditional entropy term reduces the uncertainty of the component from which the generated data are obtained. The entropy term promotes that component IDs are uniformly distributed. In terms of $\boldsymbol{\rho}$ update, the entropy term $H(C)$ drives $p(C)$ toward a discrete uniform distribution. Therefore, using $\ell_\mathrm{U2C}$ for learning $\boldsymbol{\rho}$ pulls $\boldsymbol{\pi}$ to a discrete uniform distribution and can hinder the learned $\boldsymbol{\pi}$ from accurately estimating the imbalance ratio inherent in the data. In the 9th row of Table \ref{tab:ablation}, we observed that the unsupervised conditional generation performance was undermined and the estimated imbalance ratio ($\pi_{y=0}$) was learned closer to a discrete uniform distribution ($0.5$) when $\ell_\mathrm{U2C}$ was used for $\boldsymbol{\rho}$ update.  

\paragraph{SimCLR analysis} \label{sec:simclr_analysis_apndx} 
For the colored image datasets, SLOGAN uses the SimCLR loss for the encoder with only fake (generated) data to further enhance the unsupervised conditional generation performance. The 10th to 12th rows of Table \ref{tab:ablation} show several ablation studies that analyzed the effect of SimCLR loss on SLOGAN. SLOGAN without SimCLR still showed at least approximately 35\% performance improvement compared to the existing unsupervised conditional GANs (ICFID of ClusterGAN: 75.52, CD-GAN: 76.91 in Table \ref{tab:imbalance}), even considering the fair computational cost and memory consumption. The SimCLR loss shows the highest performance improvement especially when applied only to fake data. SimCLR improved the performance by 7\%.

\paragraph{Attribute manipulation} \label{sec:attribute_manipulation_analysis_apndx}
As shown in the 13th and 14th rows of Table \ref{tab:ablation}, probe data significantly improved the performance of SLOGAN on CIFAR-2 (7:3) with 10 probe data for each latent component. We also confirmed that the mixup applied to the probe data further enhanced the overall performance of our model. Figure \ref{fig:probe_apndx} (a) shows the data generated from SLOGAN trained on CIFAR-2 (9:1) without the probe data. With extremely imbalanced attributes, SLOGAN mapped frog images with a white background onto the same component as airplanes in its latent space. When we use 10 probe data for each latent component, as shown in Figure \ref{fig:probe_apndx} (b), frogs with a white background were generated from the same latent component as the other frog images.

\begin{figure}[h!]
\centering
\includegraphics[width=1\columnwidth]{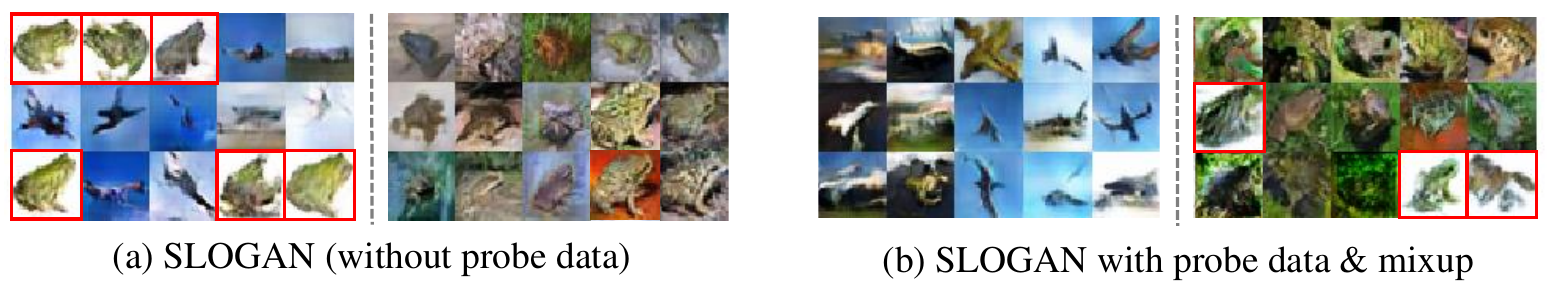} 
\caption{Effects of attribute manipulation on unsupervised conditional generation. Left and right images visualize generated images from different latent components. The red boxes indicate generated frog images with a white background.}
\label{fig:probe_apndx}
\end{figure}

\paragraph{Feature scale}
We introduced the feature scale $s$ described in Appendix \ref{sec:arcface_apndx} to reinforce the discriminative power of U2C loss. For the MNIST-2 (7:3) dataset, $s$ was set to 4. Such a parameter configuration is justified by a greedy search in $[0.5, 1, 2, 4, 8]$. The performances of SLOGAN on MNIST-2 (7:3) with different feature scales are shown in Table \ref{tab:ablation_scale}.

\begin{table}[h!]
\centering
\caption{Ablation study on feature scale}
\label{tab:ablation_scale}
{\resizebox{0.45\columnwidth}{!}
{\begin{tabular}{lccccc}
    \toprule
    $s$ & 0.5 & 1 & 2 & 4 & 8 \\
    \midrule
    ICFID $\downarrow$ & 14.18 & 17.03 & 6.65 & \textbf{5.91} & 33.98 \\
    \bottomrule
\end{tabular}}}
\end{table}

Intuitively, increasing the feature scale $s$ makes the samples generated from the same component closer to each other in the embedding space. From these results, we observed that the optimal choice of the temperature factor enhances the discriminative power of U2C loss.

\subsection{Statistical Significance}
We completed the statistical tests between SLOGAN and other methods in Tables \ref{tab:balance} and \ref{tab:imbalance}. Since the results of the experiment could not satisfy normality and homogeneity of variance, we used the Wilcoxon rank sum test \citepapndx{wilcoxon}. When p-value $> 0.05$, we measured the effect size using Cohen’s d \citepapndx{cohen}. We validated that all the experiments are statistically significant or showed large or medium effect sizes, with the exception of FID of InfoGAN vs. SLOGAN for FMNIST in Table \ref{tab:balance}, NMI of DeLiGAN+ vs. SLOGAN for FMNIST-5 in Table \ref{tab:imbalance}.

\subsection{Interpolation in Latent Space}
We also qualitatively show that the continuous nature of the prior distribution of SLOGAN makes superbly smooth interpolation possible in the latent space. In Figure \ref{fig:interpolation_apndx}, we visualize images generated from latent vectors obtained via interpolation among the mean vectors of the trained latent components. The generated images gradually changed to 3, 5, and 8 for the MNIST dataset, and t-shirt/top, pullover, and dress for the Fashion-MNIST dataset. In particular, we confirmed that the face images generated from the model trained with the CelebA data changed smoothly. The continuous attributes of real-world data are well mapped to the continuous latent space assumed by us, unlike most other methods using separated latent spaces induced via discrete latent codes.

\begin{figure*}[h!]
\centering
\includegraphics[width=0.8\linewidth]{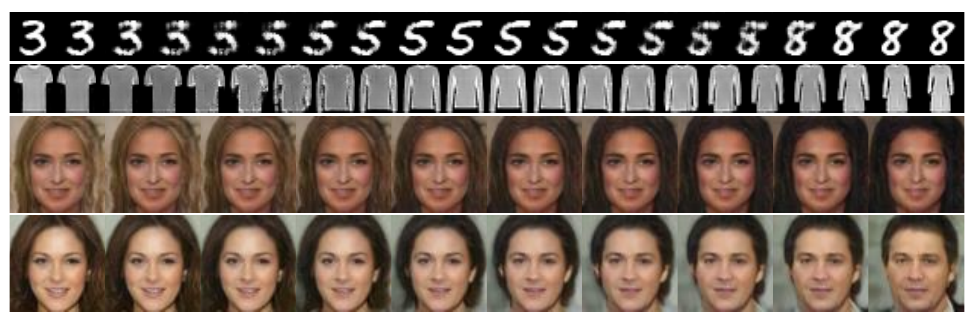} 
\caption{Interpolation in the latent space of the proposed method. For the MNIST and Fashion-MNIST datasets, we selected three mean vectors in the latent space and generated images from linearly interpolated latent vectors. For the CelebA dataset, we used 30 probe data and mixup for each latent component with attributes such as Black hair (3:1) and Male (1:1).}
\label{fig:interpolation_apndx}
\end{figure*}

\subsection{Benefits of ICFID}
DeLiGAN and ClusterGAN trained on the MNIST-2 (7:3) exhibited comparable FIDs to SLOGAN (DeLiGAN: 4.21, ClusterGAN: 4.25, and SLOGAN: 4.02); however they showed ICFIDs approximately four times higher (DeLiGAN: 25.34, ClusterGAN: 25.61, and SLOGAN: 5.91). From the data generated from each latent component of DeLiGAN and ClusterGAN in Figure \ref{fig:icfid1_apndx} (b) and (c), we confirm that the attributes were not learned well in the latent space of DeLiGAN. By contrast, from the data generated from each latent component of SLOGAN presented in Figure \ref{fig:icfid1_apndx} (a), SLOGAN successfully learned the attributes in its latent space. This shows that ICFID is useful for evaluating the performance of unsupervised conditional generation. In addition, ICFID can evaluate the diversity of images generated from a discrete latent code or mode because ICFID is based on FID. As shown in Figure \ref{fig:icfid2_apndx}, when a mode collapse occurs, the diversity of samples decreases drastically, and DeLiGAN trained on the CIFAR-2 (7:3) shows approximately twice the ICFID than those of InfoGAN and ClusterGAN (DeLiGAN: 186.31, InfoGAN: 88.49, and ClusterGAN: 75:52).


\begin{figure*}[h!]
\centering
\includegraphics[width=0.8\textwidth]{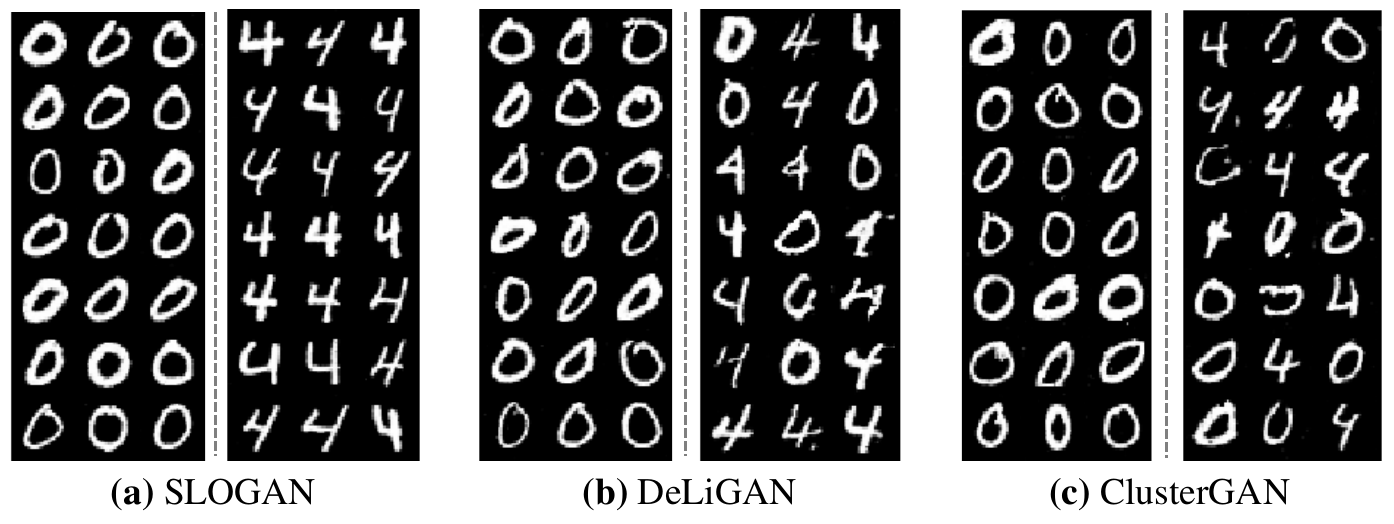} 
\caption{An example where ICFID is useful. The left and right images show generated images from each latent code of (a) SLOGAN, (b) DeLiGAN and (c) ClusterGAN trained on the MNIST-2 (7:3) dataset, respectively.}
\label{fig:icfid1_apndx}
\end{figure*}
\begin{figure*}[h!]
\centering
\includegraphics[width=0.65\textwidth]{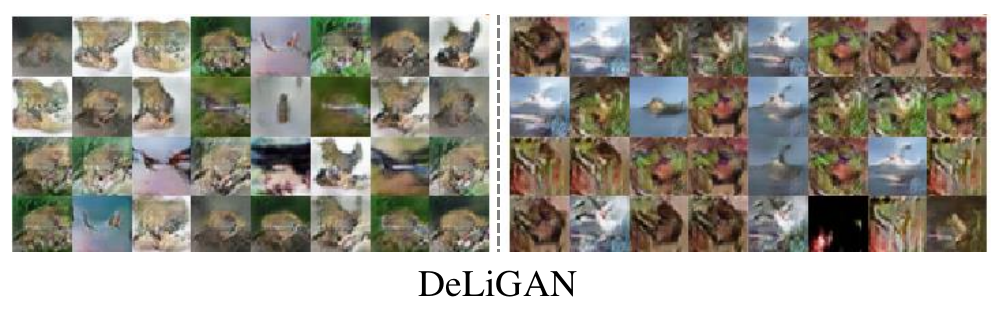} 
\caption{Another example where ICFID is useful. The left and right images show generated images from each latent component of DeLiGAN trained on the CIFAR-2 (7:3) dataset.}
\label{fig:icfid2_apndx}
\end{figure*}

\subsection{Qualitative results} \label{sec:generated_images_apndx}
\paragraph{Synthetic data}
Figures~\ref{fig:toy_apndx} and \ref{fig:toy_variance} shows the synthetic data and data generated by SLOGAN. Each color denotes generated data from each mixture component of the latent distribution. As the training progresses, each mixture component in the latent space is more strongly associated with each Gaussian distribution in the data space.

\paragraph{Generated images and latent spaces}
Figures \ref{fig:result1_apndx}, \ref{fig:result2_apndx}, \ref{fig:result8_apndx}, \ref{fig:celeba_apndx}, \ref{fig:celeba_hq_apndx}, the left plot of Figure \ref{fig:result4_apndx}, and the upper plot of Figure \ref{fig:result6_apndx} show the images generated from each latent component of SLOGAN trained on various datasets. Figure \ref{fig:result3_apndx}, the right plot of Figure \ref{fig:result4_apndx}, and the lower plot of Figure \ref{fig:result6_apndx} visualize 1,000 latent vectors of SLOGAN trained on the MNIST, Fashion-MNIST, MNIST-2 (7:3), and FMNIST-5 datasets using 3D principal component analysis (PCA). Each color represents the component with the highest responsibility, and each image shows the generated image from the latent vector. As shown in Figure \ref{fig:result3_apndx}, similar attributes (e.g., 4, 7, and 9) are mapped to nearby components in the latent space.

\paragraph{Comparisons with the most recent methods}
We compare our method with the most recent methods such as CD-GAN \citepapndx{cdgan_apndx} and PGMGAN \citepapndx{pgmgan_apndx} on the CIFAR-2 (7:3) dataset. From the results shown in Figure \ref{fig:recent_apndx}, we qualitatively confirm that SLOGAN learns imbalanced attributes of the dataset most robustly.

\paragraph{Highly imbalanced multi-class data}
We trained our method on highly imbalanced multi-class datasets by setting class 8 of the MNIST dataset to very low proportions of the other nine classes (e.g., 10:10:10:10:10:10:10:10:1:10 and 10:10:10:10:10:10:10:10:2:10). When class 8 is 0.1 fraction of the other nine classes, images of class 7 with a horizontal line outnumber images of class 8, and SLOGAN identifies 7 with a horizontal line as a more salient attribute than 8 as shown in the red boxes in Figure \ref{fig:mnist10_imb_apndx} (a). On the other hand, when class 8 is 0.2 fraction of the other nine classes, images of class 8 outnumber images of class 7 with a horizontal line. Therefore, SLOGAN successfully identifies 8 as a salient attribute as shown in the red box in Figure \ref{fig:mnist10_imb_apndx} (b).

\paragraph{Qualitative analysis with various imbalance ratios}
Figure \ref{fig:afhq_apndx} shows generated images from each latent component of SLOGAN trained on the AFHQ dataset. For various imbalance ratios of cats and dogs, we qualitatively analyze SLOGAN without using probe data. When the imbalance ratios are 1:1 and 1:2, SLOGAN identifies cat/dog as the most salient attribute and learned the attribute successfully as presented in Figure \ref{fig:afhq_apndx} (a) and (b). When the imbalance ratio is 1:5, SLOGAN discovers folded ears as the most salient attribute as shown in Figure \ref{fig:afhq_apndx} (c).

\vspace*{\fill}

\begin{figure*}[h!]
\centering
\includegraphics[width=0.99\textwidth]{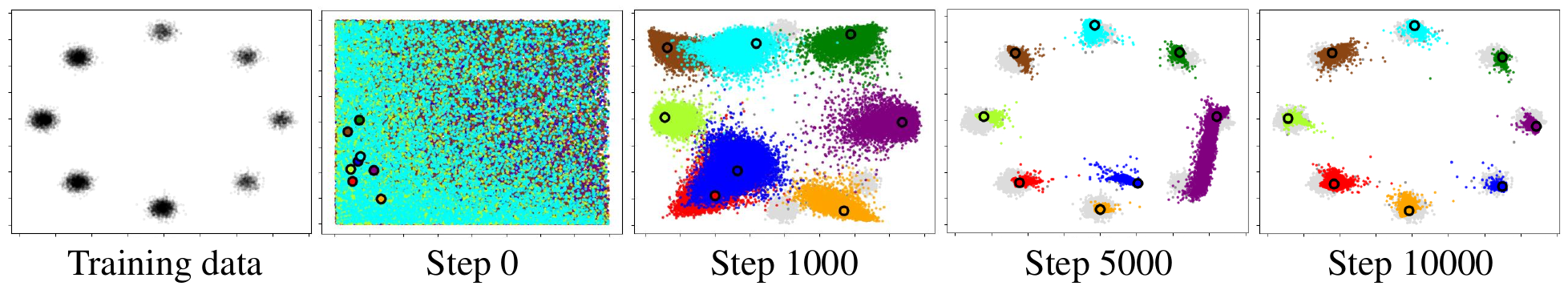} 
\caption{Synthetic dataset and samples generated by SLOGAN at 0, 1000, 5000, and 10000 steps.}
\label{fig:toy_apndx}
\end{figure*}

\vspace*{\fill}

\begin{figure*}[h!]
\centering
\includegraphics[width=0.8\textwidth]{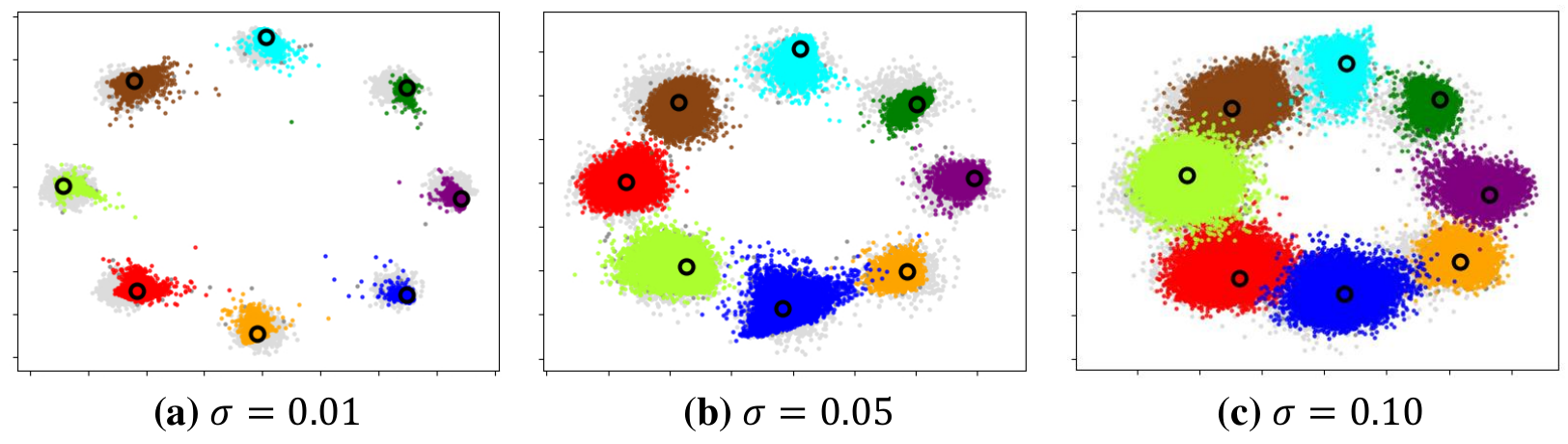} 
\caption{Generated samples from each latent component of SLOGAN trained on the synthetic datasets with variances (a) $0.01I$, (b) $0.05I$, and (c) $0.10I$.}
\label{fig:toy_variance}
\end{figure*}

\vspace*{\fill}

\begin{figure}[h!]
\centering
\includegraphics[width=0.95\textwidth]{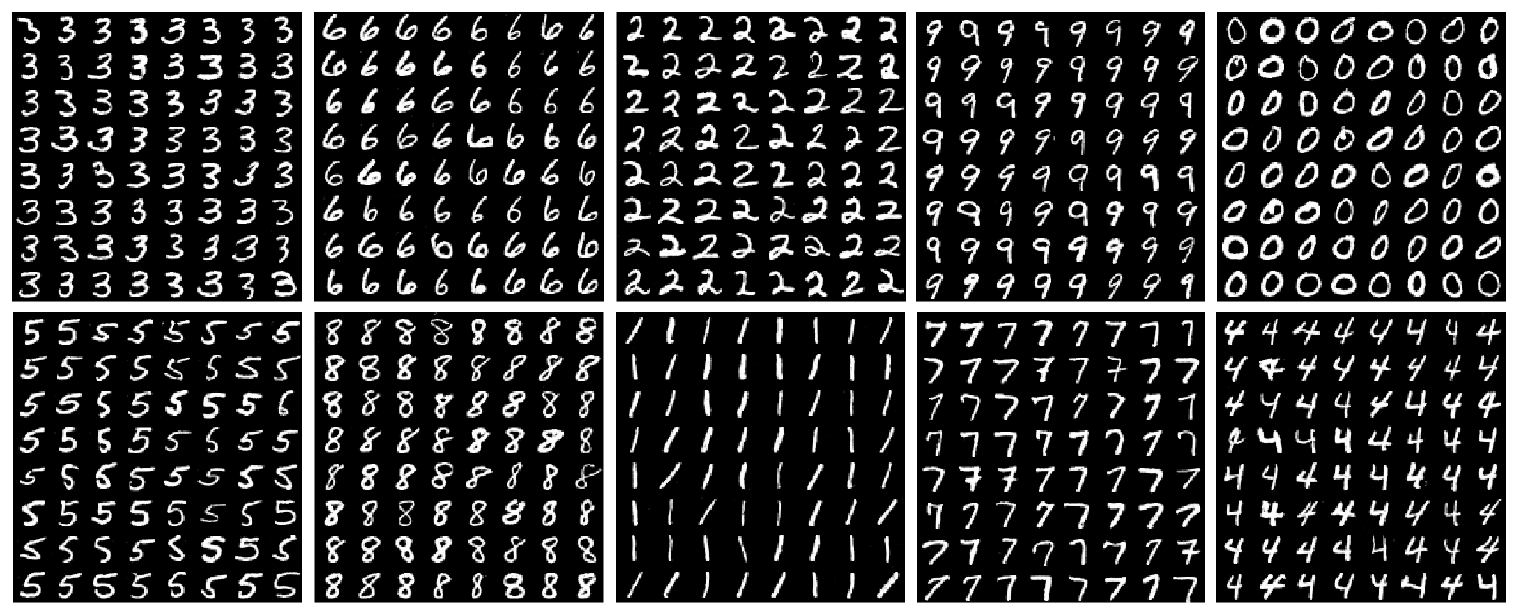} 
\caption{Generated images from each latent component of SLOGAN trained on the MNIST dataset.}
\label{fig:result1_apndx}
\end{figure}

\begin{figure*}[h!]
\centering
\includegraphics[width=0.95\textwidth]{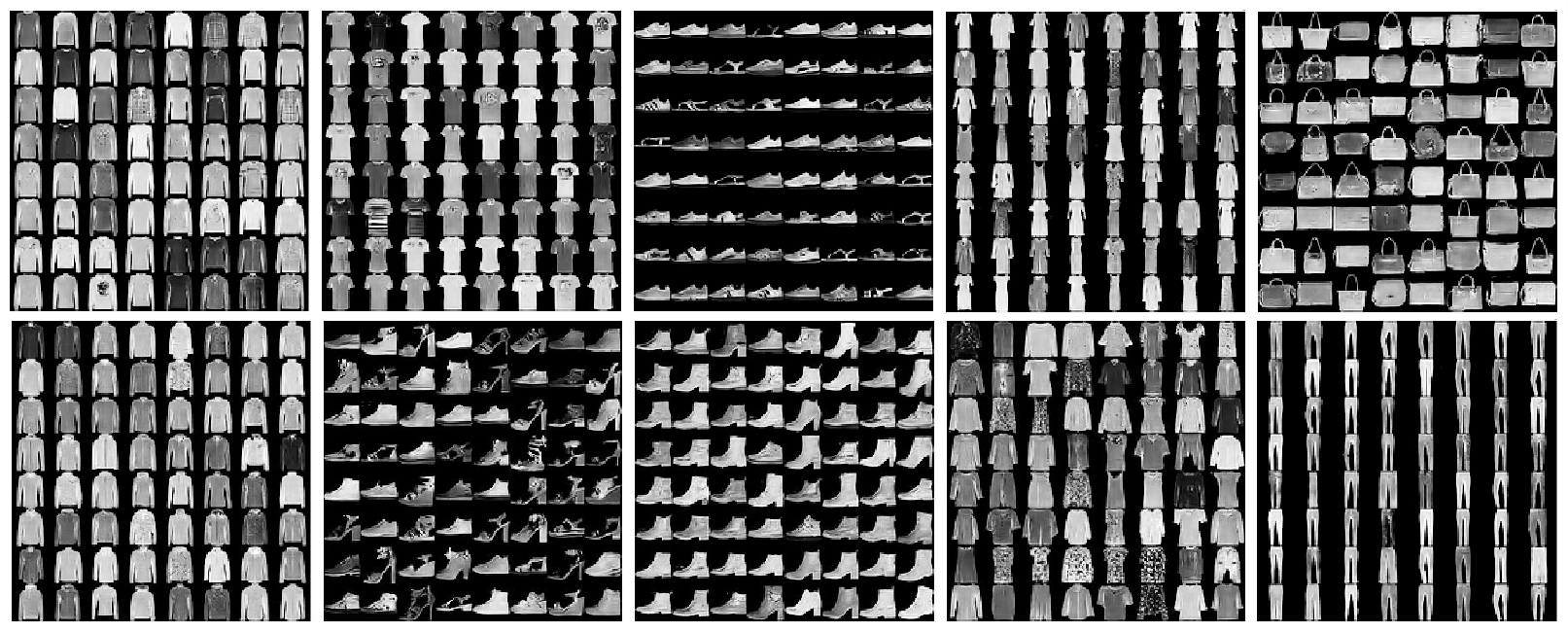} 
\caption{Generated images from each latent component of SLOGAN trained on the Fashion-MNIST dataset.}
\label{fig:result2_apndx}
\end{figure*}

\begin{figure*}[h!]
\centering
\includegraphics[width=0.95\textwidth]{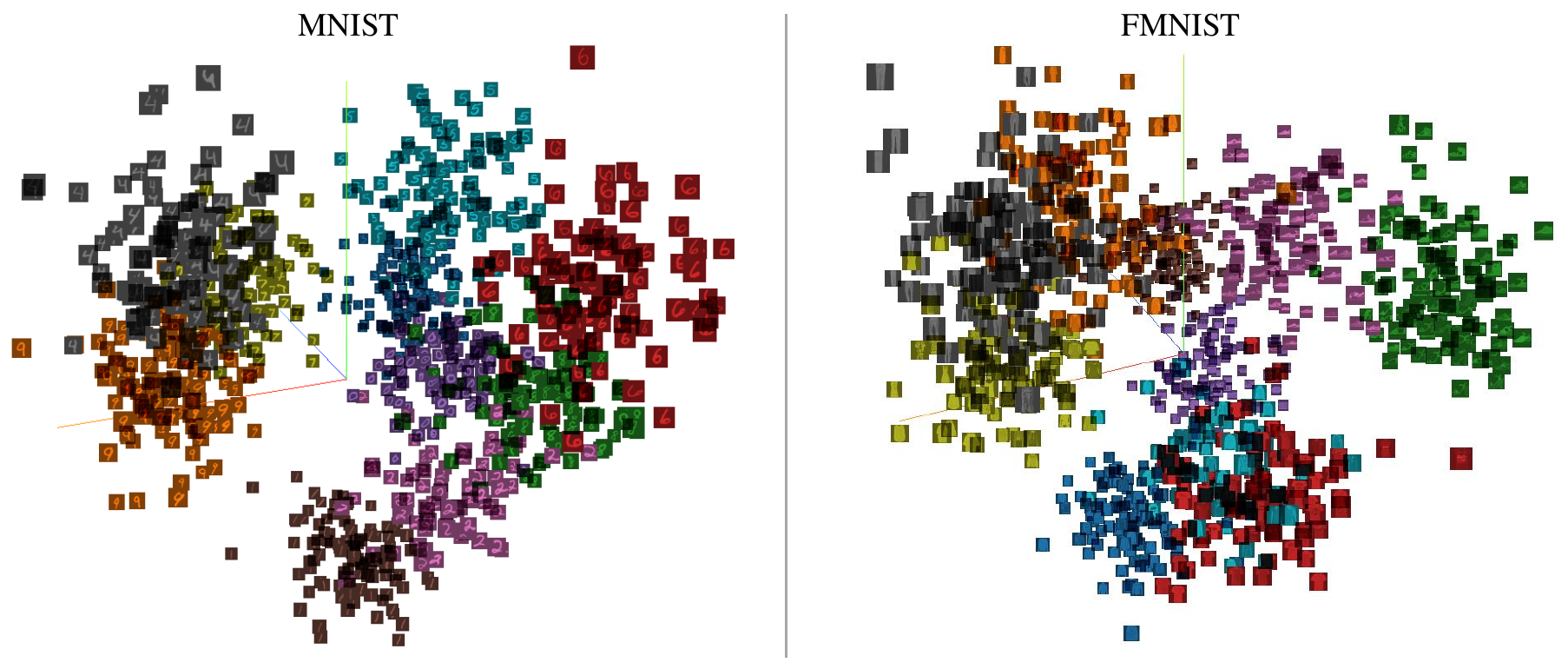} 
\caption{3D PCA of the latent spaces of SLOGAN trained on the MNIST and Fashion-MNIST datasets.}
\label{fig:result3_apndx}
\end{figure*}

\begin{figure*}[h!]
\centering
\includegraphics[width=0.95\textwidth]{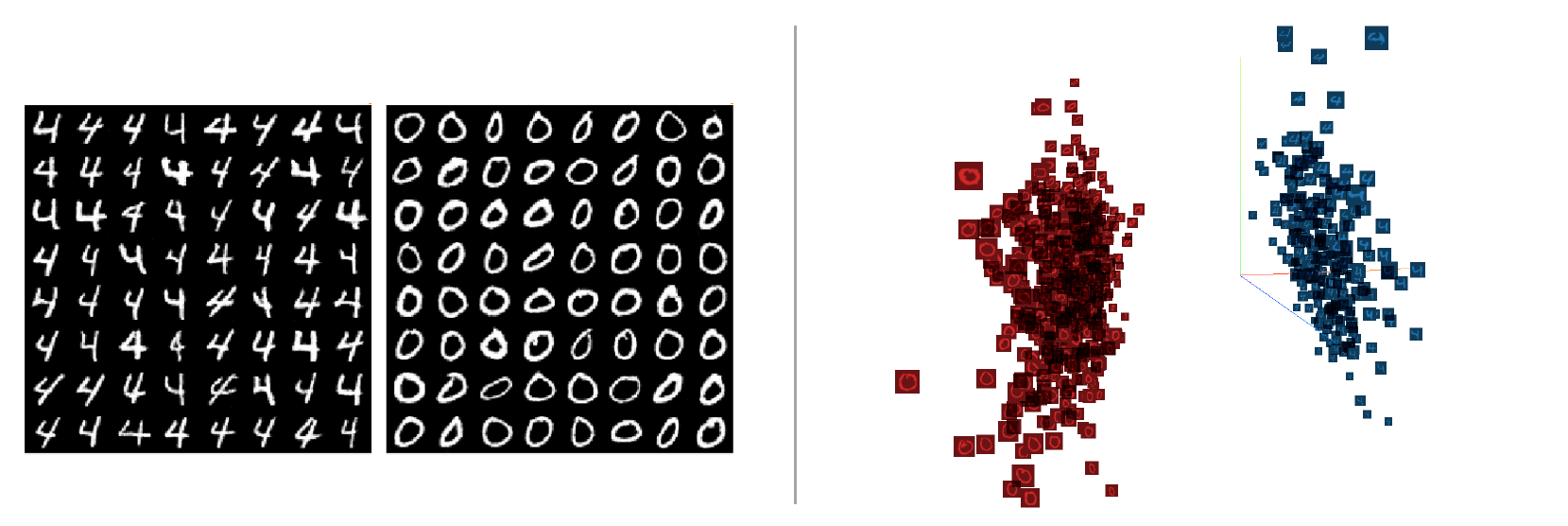} 
\caption{Generated images from each latent component and 3D PCA of the latent spaces of SLOGAN trained on the MNIST-2 (7:3) dataset.}
\label{fig:result4_apndx}
\end{figure*}

\begin{figure*}[h!]
\centering
\includegraphics[width=0.95\textwidth]{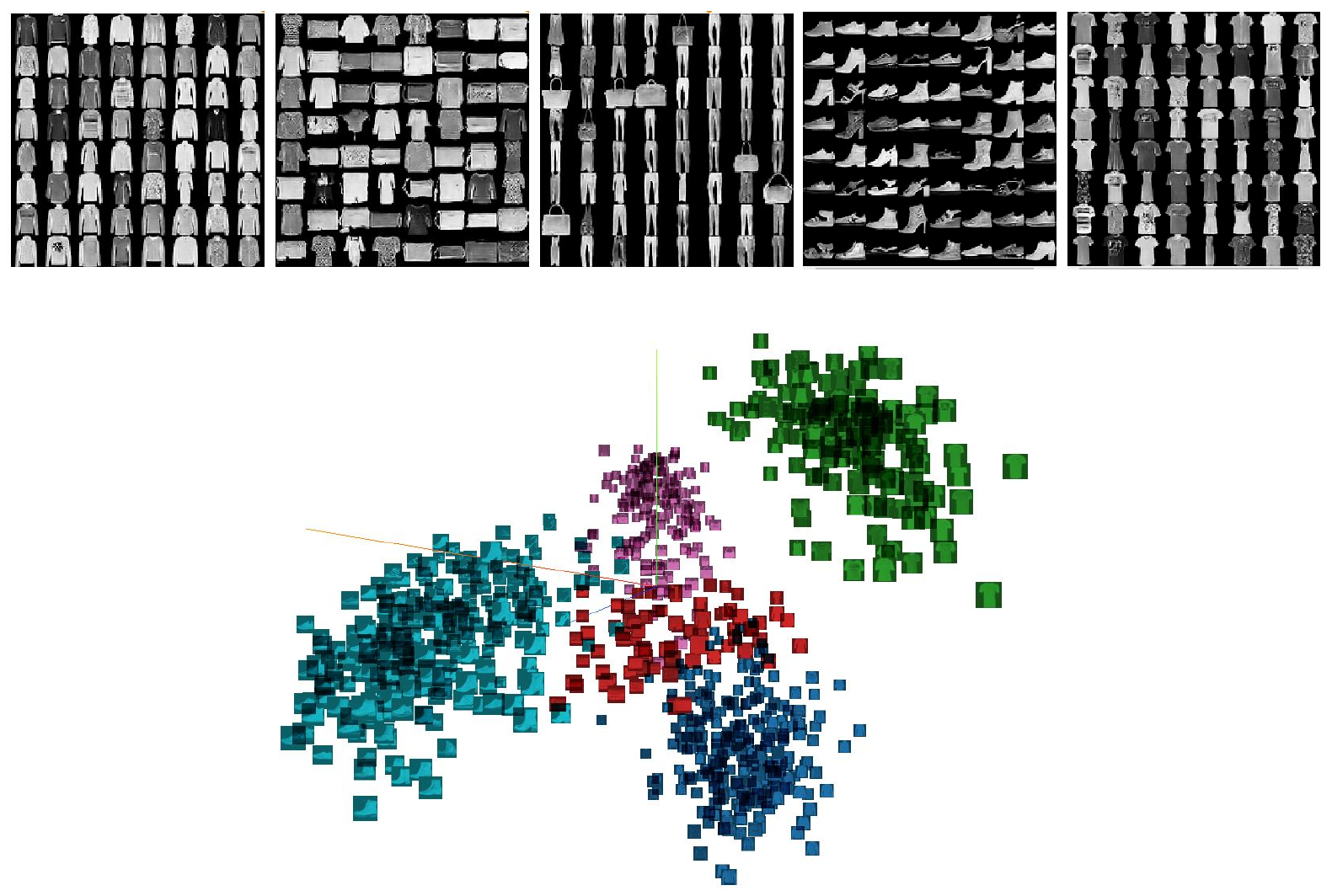} 
\caption{Generated images from each latent component and 3D PCA of the latent space of SLOGAN trained on the FMNIST-5 dataset.}
\label{fig:result6_apndx}
\end{figure*}

\begin{figure*}[h!]
\centering
\includegraphics[width=1\textwidth]{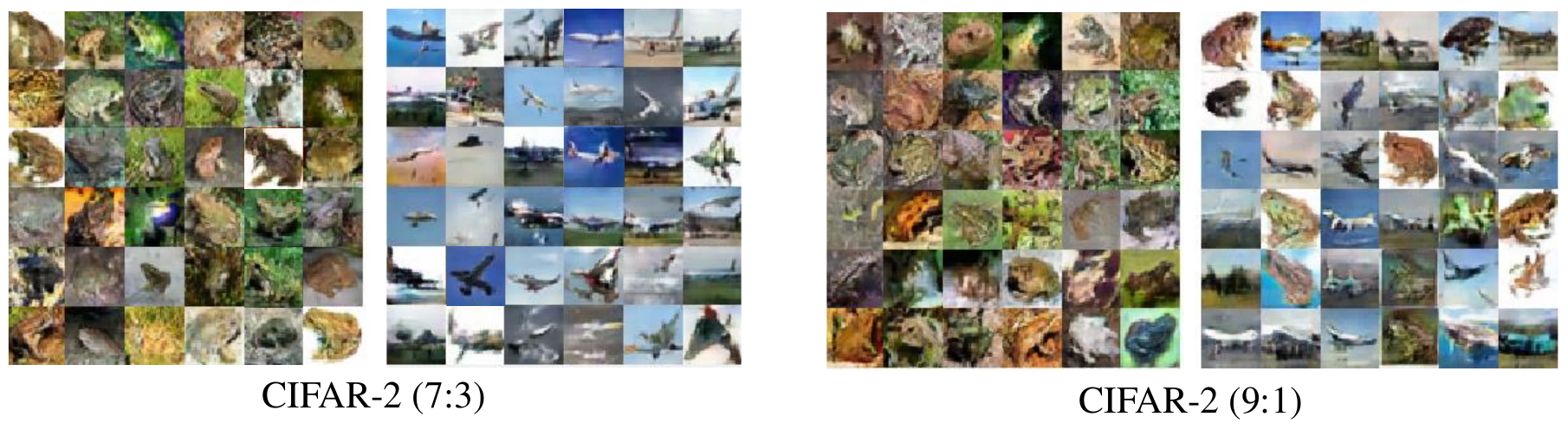} 
\caption{Generated images from each latent component of SLOGAN trained on the CIFAR-2 (7:3) and CIFAR-2 (9:1) datasets.}
\label{fig:result8_apndx}
\end{figure*}
\clearpage
\vspace*{\fill}
\begin{figure*}[h!]
\centering
\includegraphics[width=0.8\textwidth]{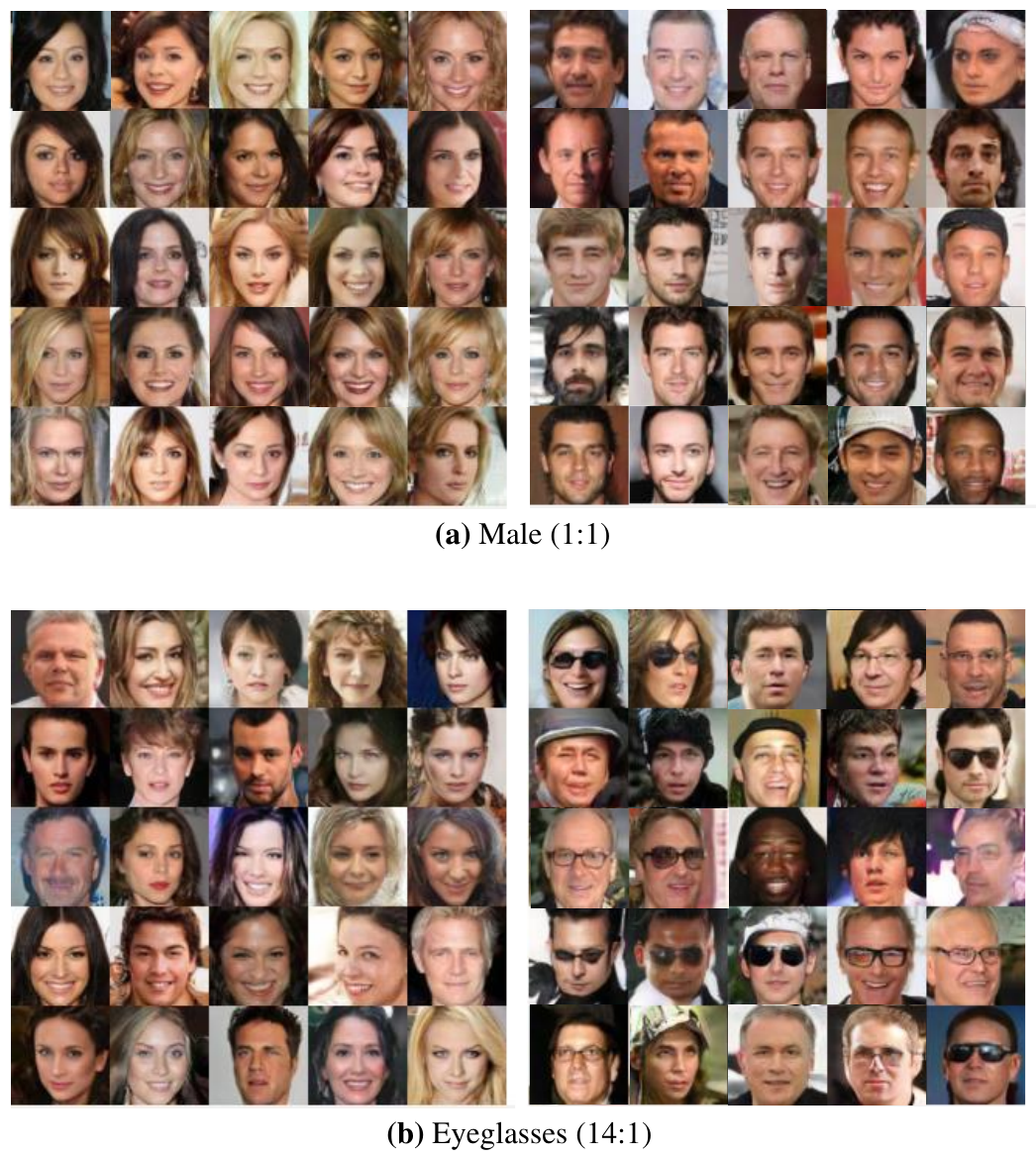} 
\caption{Generated images from each latent component of SLOGAN trained on the CelebA dataset. We used 30 probe data ((a) Female vs. Male, or (b) Faces without eyeglasses vs. Faces with eyeglasses) and mixup for each component.}
\label{fig:celeba_apndx}
\end{figure*}

\vspace*{\fill}
\clearpage
\vspace*{\fill}

\begin{figure*}[h!]
\centering
\includegraphics[width=0.9\textwidth]{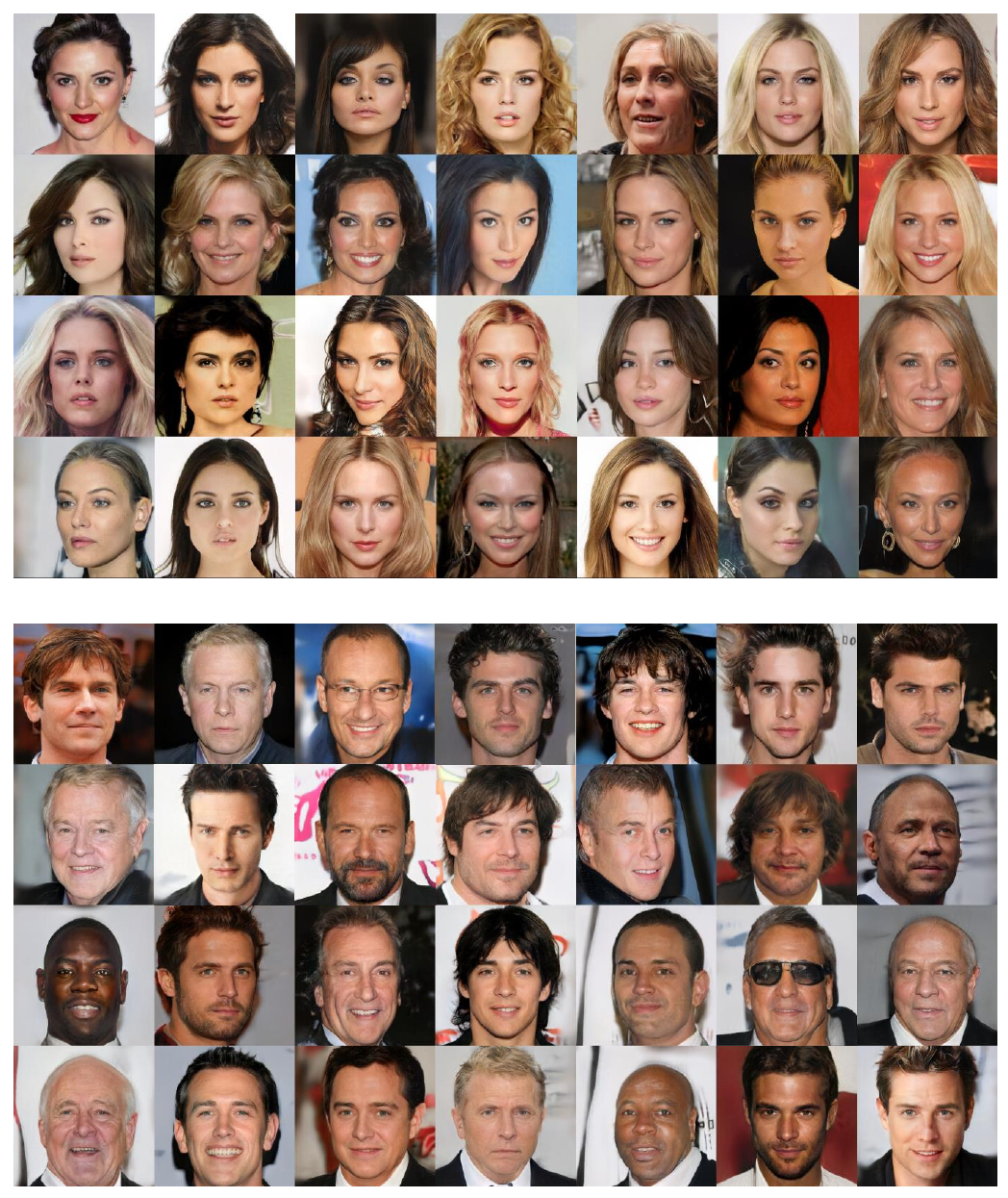} 
\caption{Generated images from each latent component of SLOGAN trained on the CelebA-HQ (256$\times$256) dataset. We used 30 probe data (Female vs. Male) and mixup for each component.}
\label{fig:celeba_hq_apndx}
\end{figure*}

\vspace*{\fill}
\clearpage
\vspace*{\fill}

\begin{figure*}[h!]
\centering
\includegraphics[width=1\textwidth]{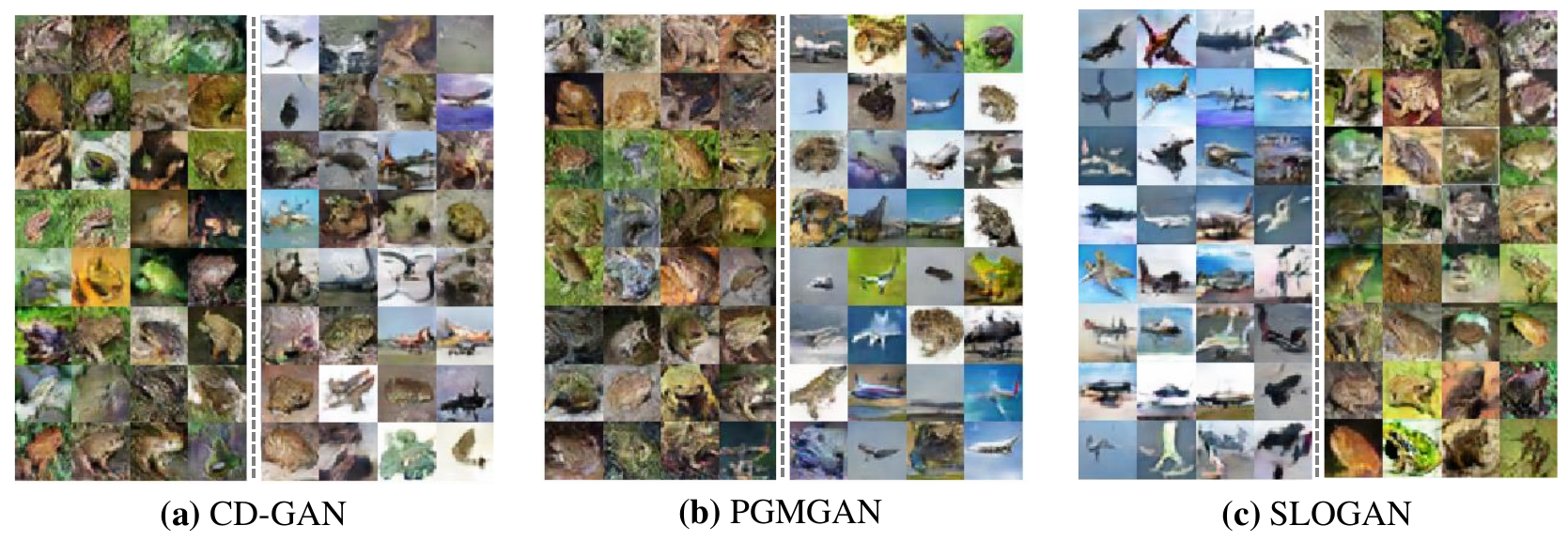} 
\caption{Generated images from the most recent methods including (a) CD-GAN, (b) PGMGAN, and (c) SLOGAN trained on the CIFAR-2 (7:3) dataset.}
\label{fig:recent_apndx}
\end{figure*}

\begin{figure*}[h!]
\centering
\includegraphics[width=1\textwidth]{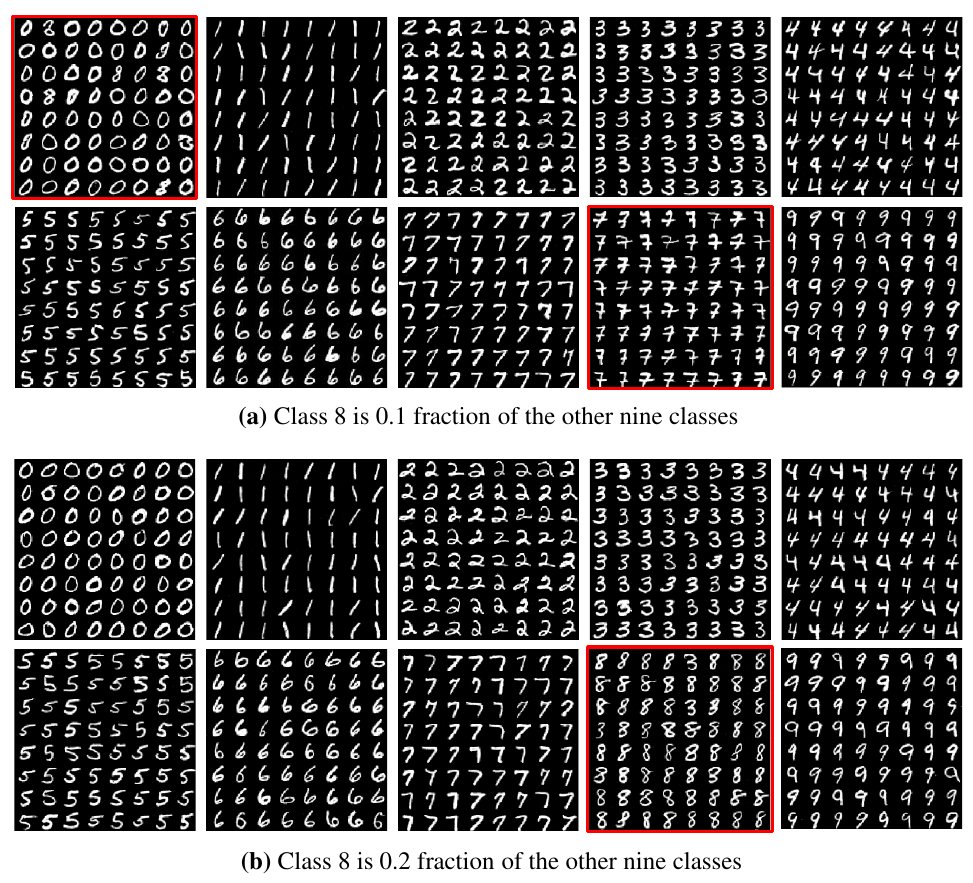} 
\caption{Generated images from each latent component of SLOGAN trained on the MNIST dataset where class 8 is very low fraction of the other nine classes.}
\label{fig:mnist10_imb_apndx}
\end{figure*}

\vspace*{\fill}
\clearpage
\vspace*{\fill}

\begin{figure*}[h!]
\centering
\includegraphics[width=1\textwidth]{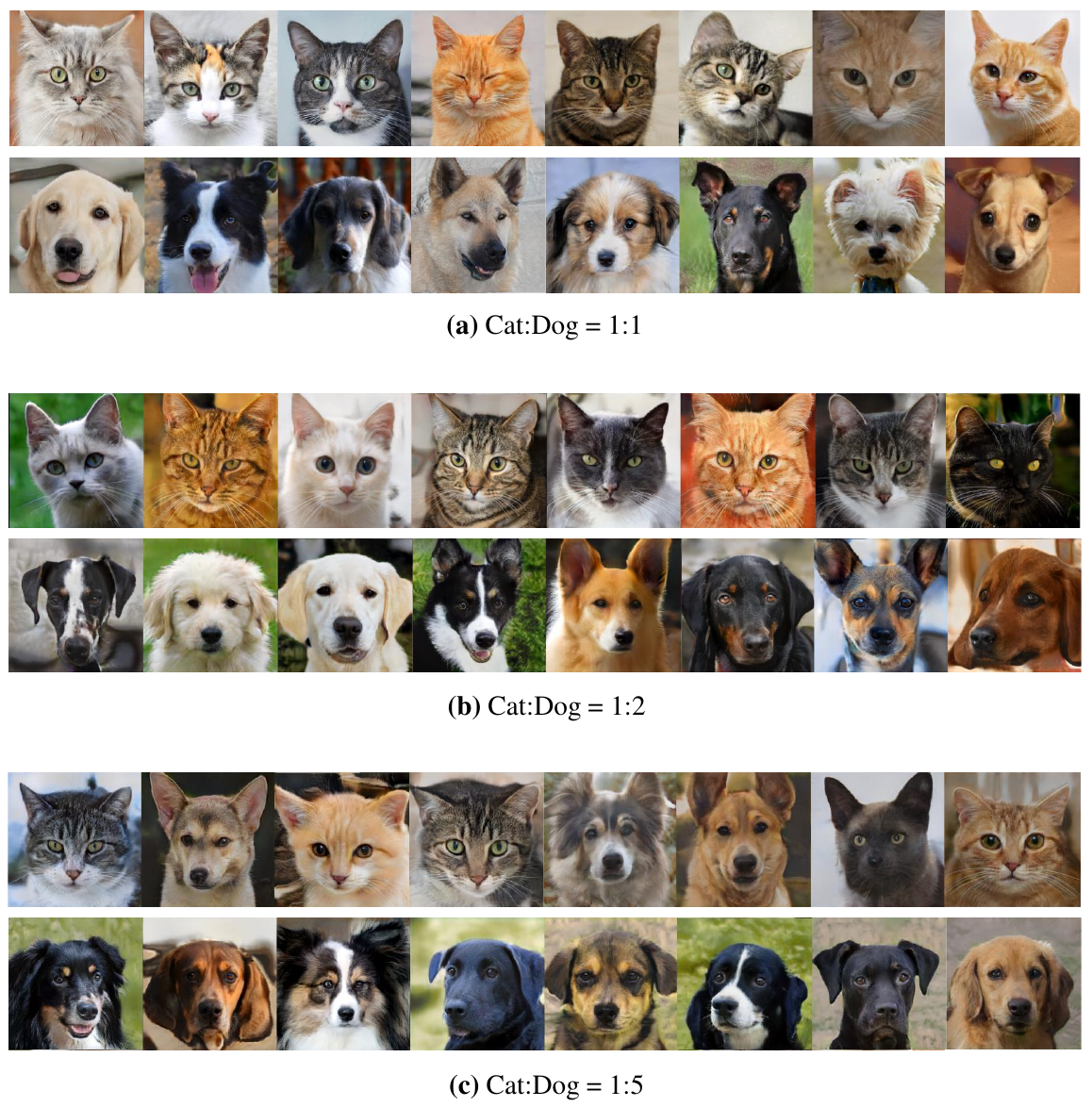} 
\caption{Generated images from each latent component of SLOGAN trained on Cats and Dogs of the AFHQ (256$\times$256) dataset with various imbalance ratios.}
\label{fig:afhq_apndx}
\end{figure*}

\vspace*{\fill}
\clearpage

\section{Additional Background}
Our work is closely related to Stein's lemma and the reparameterization trick for Gaussian mixture and also related to several topics of GAN studies such as representation learning, supervised/unsupervised conditional generation, and Gaussian mixture prior.

\subsection{Supervised conditional generation}
Conditional GANs including ACGAN \citepapndx{acgan_apndx}, projection discriminator \citepapndx{projection_apndx}, and ContraGAN \citepapndx{contragan_apndx} have led to state-of-the-art performances in conditional image generation. However, these conditional GANs are trained with supervision and require a large amount of labeled data.

\subsection{Unsupervised conditional generation} \label{sec:related_apndx}
InfoGAN \citepapndx{infogan_apndx} introduces latent codes composed of categorical and continuous variables and derives the lower bound of mutual information between the latent codes and representations. By maximizing the lower bound, InfoGAN learns disentangled latent variables. ClusterGAN \citepapndx{clustergan_apndx} assumes a discrete-continuous Gaussian prior wherein discrete variables are defined as a one-hot vector and continuous variables are sampled from a Gaussian distribution. The latent distribution is clustered in an unsupervised manner through reconstruction
loses for discrete and continuous variables. CD-GAN \citepapndx{cdgan_apndx} is similar to ClusterGAN, but it uses contrastive loss to disentangle attributes. These methods assume uniform distribution of the attributes, and if the imbalance ratio is unknown, these models cannot thoroughly learn imbalanced attributes.

Recently, unsupervised conditional GANs which do not assume uniform distribution of the attributes have been proposed. Self-conditioned GAN \citepapndx{scgan_apndx} performs unsupervised conditional generation by using clustering of discriminator features as labels. However, it has no loss to facilitate clustering of discriminator features, and the performance seems to be largely influenced by the architecture of the discriminator. PGMGAN \citepapndx{pgmgan_apndx} applies a contrastive clustering method named SCAN \citepapndx{scan_apndx} to perform unsupervised conditional generation. The pretrained space partitioner and a penalization loss function are used to encourage the generator to generate images with salient attributes. To avoid degenerated clusters where one partition contains most of the data, PGMGAN uses an entropy term that regularizes the average cluster probability vector to a uniform distribution. However, when learning PGMGAN on datasets with imbalanced attributes, it is difficult to adjust the coefficient of the regularizer because the performance seems to be sensitive to the strength of the regularizer in practice (e.g., the results on CIFAR-2 (7:3) and (9:1) in Table \ref{tab:imbalance_apndx}). In addition, clustering results does not seem to be reliable when the space partitioner is trained on datasets where transformations are not available (e.g., the result on 10x$\_$73k in Table \ref{tab:imbalance_apndx}).

NEMGAN \citepapndx{nemgan_apndx} argues that it considers the imbalance problem of attributes. However, it assumes that some labeled samples are provided, which is unrealistic in real-world scenarios. In addition, labeled samples should have the same imbalance ratio as the training data to estimate the imbalance ratio. If these samples are not given, NEMGAN is the same method as ClusterGAN. On the other hand, SLOGAN does not require labeled samples to learn imbalanced attributes. Even when a small amount of probe data is used to learn specific attributes, imbalanced attributes can be learned with balanced probe data (e.g., 30 male and 30 female faces on the CelebA-HQ dataset).

\subsection{Representation learning}
Representation learning refers to the discovery of meaningful semantics in datasets \citepapndx{bengio2013representation_apndx}. Perceptually salient and semantically meaningful representations can induce better performance in downstream tasks. In particular, Representation learning in generative models enables the learned latent representations to be semantically meaningful without labels in the training data. BiGAN \citepapndx{bigan_apndx}, ALI \citepapndx{ali_apndx}, and their variants \citepapndx{bicogan_apndx, hierarchical_ali_apndx} are similar to our study in that they add an encoder network to the original GAN framework. The additional encoder may serve as a inversion mapping of the generator, and learned feature representation is useful for supervised tasks. However, the learned generator cannot perform conditional generation without supervision.

\newpage

\section{Proofs}
\subsection{Gradient Identity for \texorpdfstring{$\boldsymbol{\mu}_c$}{}} \label{sec:bonnet_apndx}
\begin{thm} \label{th:bonnet_apndx}
    Given an expected loss of the generator $\mathcal{L}$ and a loss function for a sample $\ell(\cdot): \mathbb{R}^{d_z} \mapsto \mathbb{R}$, we assume $\ell$ is continuously differentiable. Then, the following identity holds:
    \begin{gather}
        \nabla_{\boldsymbol{\mu}_c} \mathcal{L} =  \mathbb{E}_{q} \left[ \delta(\mathbf{z})_c \pi_c \nabla_{\mathbf{z}} \ell(\mathbf{z}) \right] \label{eq:bonnet_apndx}
    \end{gather}
\end{thm}
\begin{proof}
To prove Theorem \ref{th:bonnet_apndx}, the following lemma (Bonnet's theorem) is introduced.
\begin{lemma} \label{le:bonnet_apndx}
    Let $h(\mathbf{z}): \mathbb{R}^{d} \mapsto \mathbb{R}$ be continuously differentiable. $q(\mathbf{z})$ is a multivariate Gaussian distribution $\mathcal{N}(\mathbf{z};\boldsymbol{\mu}, \boldsymbol{\Sigma})$. Then, the following identity holds:
    \begin{gather}
        \nabla_{\boldsymbol{\mu}} \mathbb{E}_{q(\mathbf{z})} \left[ h(\mathbf{z}) \right] = \mathbb{E}_{q(\mathbf{z})} \left[ \nabla_\mathbf{z} h(\mathbf{z}) \right]
    \end{gather}
\end{lemma}
The proof of Lemma \ref{le:bonnet_apndx} is described by Theorem 3 of \citetapndx{lin2019stein_apndx}. Using Lemma \ref{le:bonnet_apndx}, we show that
\begin{align}
    \nabla_{\boldsymbol{\mu}_c} \mathcal{L} &= \nabla_{\boldsymbol{\mu}_c} \mathbb{E}_{q(\mathbf{z})} \left[ \ell(\mathbf{z}) \right]
    = \nabla_{\boldsymbol{\mu}_c} \sum_{i=1}^K p(i) \mathbb{E}_{q(\mathbf{z}|i)} \left[ \ell(\mathbf{z}) \right] \\
    &= p(c) \nabla_{\boldsymbol{\mu}_c} \mathbb{E}_{q(\mathbf{z}|c)} \left[ \ell(\mathbf{z}) \right]
    = p(c)\mathbb{E}_{q(\mathbf{z}|c)} \left[ \nabla_\mathbf{z} \ell(\mathbf{z}) \right] \\
    &= \int q(\mathbf{z}|c)p(c) \nabla_\mathbf{z} \ell(\mathbf{z}) d\mathbf{z}  \\
    &= \int q(\mathbf{z}) \frac{q(\mathbf{z}|c)p(c)}{q(\mathbf{z})} \nabla_\mathbf{z} \ell(\mathbf{z}) d\mathbf{z} \\ 
    &= \int q(\mathbf{z}) \delta(\mathbf{z})_c \pi_c \nabla_\mathbf{z} \ell(\mathbf{z}) d\mathbf{z} \\
    &= \mathbb{E}_{q(\mathbf{z})} \left[ \delta(\mathbf{z})_c \pi_c \nabla_{\mathbf{z}} \ell(\mathbf{z}) \right]
\end{align}
\end{proof}

\subsection{First-order Gradient Identity for \texorpdfstring{$\boldsymbol{\Sigma}_c$}{}} \label{sec:price_apndx}
\begin{thm} \label{th:price_apndx}
    With the same assumptions from Theorem \ref{th:bonnet_apndx}, we have the following gradient identity:
    \begin{gather}
        \nabla_{\Sigma_c} \mathcal{L} = \frac{1}{2} \mathbb{E}_{q} \left[ \delta(\mathbf{z})_c \pi_c \Sigma_c^{-1} \left( \mathbf{z}-\boldsymbol{\mu}_c \right) \nabla_{\mathbf{z}}^{T} \ell(\mathbf{z}) \right] \label{eq:price_apndx}
    \end{gather}
\end{thm}
\begin{proof}
In order to prove Theorem \ref{th:price_apndx}, we introduce the following lemma (Price's theorem).
\begin{lemma} \label{le:price_apndx}
    Let $h(\mathbf{z}): \mathbb{R}^{d} \mapsto \mathbb{R}$ and its derivative $\nabla h(\mathbf{z})$ be continuously differentiable. We further assume that $\mathbb{E}\left[ h(\mathbf{z}) \right]$ is well-defined. Then, the following identity holds:
    \begin{gather}
        \nabla_{\boldsymbol{\Sigma}} \mathbb{E}_{q(\mathbf{z})} \left[ h(\mathbf{z}) \right] = \frac{1}{2} \mathbb{E}_{q(\mathbf{z})} \left[ \boldsymbol{\Sigma}^{-1} (\mathbf{z}-\boldsymbol{\mu})\nabla_\mathbf{z}^{T} h(\mathbf{z}) \right] = \frac{1}{2} \mathbb{E}_{q(\mathbf{z})} \left[ \nabla^2_\mathbf{z} h(\mathbf{z}) \right]
    \end{gather}
\end{lemma}
The proof of Lemma \ref{le:price_apndx} is presented in Theorem 4 of \citetapndx{lin2019stein_apndx}. The rest of the proof is similar with the proof of Theorem \ref{th:bonnet_apndx}. Using the first-order gradient identity of Lemma \ref{le:price_apndx}, we get
\begin{align}
    \nabla_{\Sigma_c} \mathcal{L} &= \nabla_{\Sigma_c} \mathbb{E}_{q(\mathbf{z})} \left[ \ell(\mathbf{z}) \right] 
    = \nabla_{\Sigma_c} \sum_{i=1}^K p(i) \mathbb{E}_{q(\mathbf{z}|i)} \left[ \ell(\mathbf{z}) \right] \\
    &= p(c) \nabla_{\Sigma_c} \mathbb{E}_{q(\mathbf{z}|c)} \left[ \ell(\mathbf{z}) \right]
    = \frac{1}{2} p(c) \mathbb{E}_{q(\mathbf{z}|c)} \left[ \Sigma_c^{-1} (\mathbf{z}-\boldsymbol{\mu}_c)\nabla_\mathbf{z}^{T} \ell(\mathbf{z}) \right] \\
    &= \frac{1}{2} \int q(\mathbf{z}|c)p(c) \Sigma_c^{-1} (\mathbf{z}-\boldsymbol{\mu}_c)\nabla_\mathbf{z}^{T} \ell(\mathbf{z}) d\mathbf{z}  \\
    &= \frac{1}{2} \int q(\mathbf{z}) \frac{q(\mathbf{z}|c)p(c)}{q(\mathbf{z})} \Sigma_c^{-1} (\mathbf{z}-\boldsymbol{\mu}_c)\nabla_\mathbf{z}^{T} \ell(\mathbf{z}) d\mathbf{z} \\ 
    &= \frac{1}{2} \int q(\mathbf{z}) \delta(\mathbf{z})_c \pi_c \Sigma_c^{-1} (\mathbf{z}-\boldsymbol{\mu}_c)\nabla_\mathbf{z}^{T} \ell(\mathbf{z}) d\mathbf{z} \\
    &= \frac{1}{2} \mathbb{E}_{q(\mathbf{z})} \left[ \delta(\mathbf{z})_c \pi_c \Sigma_c^{-1} (\mathbf{z}-\boldsymbol{\mu}_c)\nabla_\mathbf{z}^{T} \ell(\mathbf{z}) \right]
\end{align}
\end{proof}

\subsection{Ensuring Positive-definiteness of \texorpdfstring{$\boldsymbol{\Sigma}_c$}{}} \label{sec:psd_apndx}
\begin{thm} \label{th:psd_apndx}
    The updated covariance matrix $\Sigma_{c}'=\Sigma_{c}+\gamma\Delta\Sigma_{c}'$ with the modified update rule specified in Equation \ref{eq:psd} is positive-definite if $\Sigma_{c}$ is positive-definite.
\end{thm}
\begin{proof}
Because $\Sigma_{c}$ is symmetric and positive-definite, we can decompose $\Sigma_{c} = \mathrm{L}\mathrm{L}^{T}$ using Cholesky decomposition, where $\mathrm{L}$ is the lower triangular matrix. Then, we can prove the positive-definiteness of the updated covariance matrix as follows:
\begin{align}
    \Sigma_{c}' &= \Sigma_{c}+\gamma\Delta\Sigma_{c}' \\
    &= \Sigma_{c}+\gamma (\Delta\Sigma_c + \frac{\gamma}{2} \Delta\Sigma_c \Sigma_c^{-1} \Delta\Sigma_c) \\
    &= \Sigma_{c}+\gamma\Delta\Sigma_c + \frac{\gamma^2}{2} \Delta\Sigma_c \Sigma_c^{-1} \Delta\Sigma_c \\
    &= \frac{1}{2} \left( \Sigma_{c}+ (\mathrm{L} + \gamma\Delta\Sigma_c\mathrm{L}^{-T}) (\mathrm{L}^T + \gamma \mathrm{L}^{-1} \Delta\Sigma_c) \right) \\
\end{align}
Let us define $\mathrm{U} \vcentcolon= \mathrm{L}^T + \gamma \mathrm{L}^{-1} \Delta\Sigma_c$. Then, we have the following:
\begin{align}
    \Sigma_{c}' &= \frac{1}{2} \left( \Sigma_{c} + \mathrm{U}^T \mathrm{U} \right) \succ 0
\end{align}
where both $\Sigma_{c}$ and $\mathrm{U}^T \mathrm{U}$ are positive-definite, concluding the proof.
\end{proof}

\subsection{Gradient Identity for \texorpdfstring{$\boldsymbol{\rho}_c$}{}} \label{sec:rho_apndx}
\begin{thm} \label{th:rho_apndx}
    Let $\rho_c$ be a mixing coefficient parameter, and the following gradient identity holds:
    \begin{gather}
        \nabla_{\rho_c} \mathcal{L} = \mathbb{E}_{q} \left[ \pi_c \left(\delta(\mathbf{z})_c - 1\right) \ell(\mathbf{z}) \right] 
    \end{gather}
\end{thm}
\begin{proof}
The gradient of the latent distribution with respect to the mixing coefficient parameter is derived as follows:
\begin{align}
    \nabla_{\rho_c} q(\mathbf{z}) &= \nabla_{\rho_c} \sum_{i=1}^K \mathrm{softmax} (\rho_i) q(\mathbf{z}|i)
    = \pi_c \left( q(\mathbf{z}|c) - \sum_{i=1}^K \pi_i q(\mathbf{z}|i) \right)
\end{align}
where $\mathrm{softmax} (\cdot)$ denotes the softmax function (e.g., $p(i)=\pi_i=\mathrm{softmax} (\rho_i)$). Using the above equation, we have
\begin{align}
    \nabla_{\rho_c} \mathcal{L} &= \nabla_{\rho_c} \mathbb{E}_{q(\mathbf{z})} \left[ \ell(\mathbf{z}) \right] = \int \nabla_{\rho_c} q(\mathbf{z}) \ell(\mathbf{z}) d\mathbf{z} \\
    &= \int \pi_c \left( q(\mathbf{z}|c) - \sum_{i=1}^K \pi_i q(\mathbf{z}|i) \right) \ell(\mathbf{z}) d\mathbf{z} \\
    &= \int \pi_c q(\mathbf{z}) \left( \frac{q(\mathbf{z}|c)}{q(\mathbf{z})} - \sum_{i=1}^K \pi_i \frac{q(\mathbf{z}|i)}{q(\mathbf{z})} \right) \ell(\mathbf{z}) d\mathbf{z} \\
    &= \int q(\mathbf{z}) \pi_c \left( \delta(\mathbf{z})_c - \sum_{i=1}^K \pi_i \delta(\mathbf{z})_i \right) \ell(\mathbf{z}) d\mathbf{z} \label{eq:rho_apndx}
\end{align}
Here, $\sum_{i=1}^K \pi_i \delta(\mathbf{z})_i = \sum_{i=1}^K \frac{p(i)q(\mathbf{z}|i)}{q(\mathbf{z})}=1$. Plugging this back into Equation \ref{eq:rho_apndx}, we obtain
\begin{align}
    \nabla_{\rho_c} \mathcal{L} &= \int q(\mathbf{z}) \pi_c \left( \delta(\mathbf{z})_c - 1 \right) \ell(\mathbf{z}) d\mathbf{z} \\
    &= \mathbb{E}_{q(\mathbf{z})} \left[ \pi_c \left( \delta(\mathbf{z})_c - 1 \right) \ell(\mathbf{z}) \right]
\end{align}
\end{proof}

\subsection{Second-order Gradient Identity for \texorpdfstring{$\boldsymbol{\Sigma}_c$}{}}
We consider a generator $\sigma(G(\mathbf{z}))=\sigma(G)=\boldsymbol{\sigma}_G$, a discriminator $D(\boldsymbol{\sigma}_G)=D$ with the piecewise-linear activation functions (e.g. ReLU or LeakyReLU) except the activation function of the output layer of the generator which is the hyperbolic tangent function $\sigma(x)=(e^x-e^{-x})/(e^x+e^{-x})$, and a Wasserstein adversarial loss function for each latent vector $\ell(\mathbf{z})=-D(\sigma(G(\mathbf{z})))=\ell$. We note that the following holds except on a set of zero Lebesgue measure because the piecewise-linear activation functions are linear except where they switch:
\begin{align}
    \frac{\partial^2 D}{\partial {\boldsymbol{\sigma}_G}^2} = 0, \;\;\;
    \frac{\partial^2 G}{\partial {\mathbf{z}}^2} = 0
\end{align}
We provide the second-order gradient identity of the Wasserstein GAN loss for the covariance matrix $\boldsymbol{\Sigma}_c$ that does not compute second-order derivatives. However, it is impractical and not included in our method because of the excessive computational cost of Jacobian matrices.
\begin{thm} \label{th:price2_apndx}
   Given the Wasserstein GAN loss for the generator $\mathcal{L}$ and a loss function for a sample $\ell(\mathbf{z}): \mathbb{R}^{d} \mapsto \mathbb{R}$, we assume $\ell$ and its derivative $\nabla \ell(\mathbf{z})$ are continuously differentiable. We further assume that $\mathbb{E}\left[ \ell(\mathbf{z}) \right]$ is well-defined. Then, the following identity holds:
    \begin{gather*}
        \nabla_{\Sigma_c} \mathcal{L} = \nabla_{\Sigma_c} \mathbb{E}_{q(\mathbf{z})} \left[ \ell(\mathbf{z}) \right] = - \mathbb{E}_{q(\mathbf{z})} \left[ \pi_c \delta(\mathbf{z})_c \nabla^T_\mathbf{z} G(\mathbf{z}) \operatorname{diag}\left( \nabla_{\sigma_G}D \odot (\boldsymbol{\sigma}_G^{\circ 3}-\boldsymbol{\sigma}_G) \right) \nabla_\mathbf{z} G(\mathbf{z}) \right]
    \end{gather*}
where $\odot$ denotes element-wise multiplication of vectors and $\mathbf{x}^{\circ i}$ denotes the i-th Hadamard (element-wise) power of a vector $\mathbf{x}$.
\end{thm}
\begin{proof}
From the second-order identity of Lemma \ref{le:price_apndx}, we have the following:
\begin{align}
   \nabla_{\Sigma_c} \mathcal{L} &=  \nabla_{\Sigma_c} \mathbb{E}_{q(\mathbf{z})} \left[ \ell(\mathbf{z}) \right]
   = \nabla_{\Sigma_c} \sum_{i=1}^K p(i) \mathbb{E}_{q(\mathbf{z}|i)} \left[ \ell(\mathbf{z}) \right] \\
   &= p(c) \nabla_{\Sigma_c} \mathbb{E}_{q(\mathbf{z}|c)} \left[ \ell(\mathbf{z}) \right]
   =\frac{1}{2} p(c) \mathbb{E}_{q(\mathbf{z}|c)} \left[ \nabla^2_\mathbf{z} \ell(\mathbf{z}) \right] \\
   &= \frac{1}{2} \mathbb{E}_{q(\mathbf{z})} \left[ \frac{q(\mathbf{z}|c)p(c)}{q(\mathbf{z})} \nabla^2_\mathbf{z} \ell(\mathbf{z}) \right] \\
   &= \frac{1}{2} \mathbb{E}_{q(\mathbf{z})} \left[ \pi_c \delta(\mathbf{z})_c \nabla^2_\mathbf{z} \ell(\mathbf{z}) \right] \label{eq:36_apndx}
\end{align}
The Hessian of the sample loss with respect to $\mathbf{z}$ is given as follows:
\begin{align}
    \nabla^2_\mathbf{z} \ell(\mathbf{z}) &= \nabla_\mathbf{z} \left( - \nabla^T_\mathbf{z} D \right) 
    = \nabla_\mathbf{z} \left( - ( \nabla_G \boldsymbol{\sigma}_G \nabla_\mathbf{z} G )^T \nabla_{\boldsymbol{\sigma}_G} D \right)^T \\
    &= \nabla_\mathbf{z} \left( - (\nabla_{\boldsymbol{\sigma}_G} D)^T \nabla_G \boldsymbol{\sigma}_G \nabla_\mathbf{z} G \right) \\
    &= - \frac{\partial}{\partial \mathbf{z}} \begin{bmatrix} \frac{\partial D}{\partial (\sigma_G)_1} (\sigma'_G)_1 & \dots & \frac{\partial D}{\partial (\sigma_G)_{d_x}} (\sigma'_G)_{d_x} \end{bmatrix}
    \begin{bmatrix}
        \frac{\partial G_1}{\partial z_1} & \dots & \frac{\partial G_1}{\partial z_{d_z}} \\
        \vdots & \ddots & \vdots \\
        \frac{\partial G_{d_x}}{\partial z_1} & \dots & \frac{\partial G_{d_x}}{\partial z_{d_z}}
    \end{bmatrix} \\
    &= - \frac{\partial}{\partial \mathbf{z}} \begin{bmatrix} \sum_{i=1}^{d_x} \frac{\partial D}{\partial (\sigma_G)_i} (\sigma'_G)_i \frac{\partial G_i}{\partial z_1} & \dots & \sum_{i=1}^{d_x} \frac{\partial D}{\partial (\sigma_G)_i} (\sigma'_G)_i \frac{\partial G_i}{\partial z_{d_z}} 
    \end{bmatrix} \\
    &= - \begin{bmatrix}
    \frac{\partial}{\partial z_1} \left( \sum_{i=1}^{d_x} \frac{\partial D}{\partial (\sigma_G)_i} (\sigma'_G)_i \frac{\partial G_i}{\partial z_{1}} \right) & \dots & \frac{\partial}{\partial z_1} \left( \sum_{i=1}^{d_x} \frac{\partial D}{\partial (\sigma_G)_i} (\sigma'_G)_i \frac{\partial G_i}{\partial z_{d_z}} \right) \\
    \vdots & \ddots & \vdots \\
    \frac{\partial}{\partial z_{d_z}} \left( \sum_{i=1}^{d_x} \frac{\partial D}{\partial (\sigma_G)_i} (\sigma'_G)_i \frac{\partial G_i}{\partial z_{1}} \right) & \dots & \frac{\partial}{\partial z_{d_z}} \left( \sum_{i=1}^{d_x} \frac{\partial D}{\partial (\sigma_G)_i} (\sigma'_G)_i \frac{\partial G_i}{\partial z_{d_z}} \right)
    \end{bmatrix}
\end{align}
where $\boldsymbol{\sigma}'_G=\nabla_G \boldsymbol{\sigma}_G$. To simplify an element $(\nabla^2_\mathbf{z} \ell(\mathbf{z}))_{jk}= - \frac{\partial}{\partial z_{j}} \left( \sum_{i=1}^{d_x} \frac{\partial D}{\partial (\sigma_G)_i} (\sigma'_G)_i \frac{\partial G_i}{\partial z_{k}} \right)$ \\for arbitrary $j, k \in \{1, ..., d_z\}$, we have
\begin{align}
    \frac{\partial^2 D}{\partial z_j \partial (\sigma_G)_i} =&~ \frac{\partial}{\partial (\sigma_G)_i} \frac{\partial D}{\partial z_j} = \frac{\partial}{\partial (\sigma_G)_i} \left( \sum_{l=1}^{d_x} \frac{\partial D}{\partial (\sigma_G)_l} \frac{\partial (\sigma_G)_l}{\partial z_j} \right) \\
    =&~ \frac{\partial}{\partial (\sigma_G)_i} \left( \sum_{l=1}^{d_x} \frac{\partial D}{\partial (\sigma_G)_l} (\sigma'_G)_l \frac{\partial G_l}{\partial z_j} \right) \\
    =&~ \sum_{l=1}^{d_x} \cancel{\frac{\partial^2 D}{\partial (\sigma_G)_i \partial (\sigma_G)_l}} (\sigma'_G)_l \frac{\partial G_l}{\partial z_j} + \sum_{l=1}^{d_x} \frac{\partial D}{\partial (\sigma_G)_l} \frac{\partial (\sigma'_G)_l}{\partial (\sigma_G)_i} \frac{\partial G_l}{\partial z_j} \\&+ \sum_{l=1}^{d_x} \frac{\partial D}{\partial (\sigma_G)_l} (\sigma'_G)_l \frac{\partial^2 G_l}{\partial (\sigma_G)_i \partial z_j} \\
    =&~ \frac{\partial D}{\partial (\sigma_G)_i} (-2(\sigma_G)_i) \frac{\partial G_i}{\partial z_j} + \frac{\partial D}{\partial (\sigma_G)_i} (\sigma'_G)_i \frac{\partial}{\partial z_j}\left(\frac{1}{(\sigma'_G)_i}\right) \\
    =&~ \frac{\partial D}{\partial (\sigma_G)_i} (-2(\sigma_G)_i) \frac{\partial G_i}{\partial z_j} + \frac{\partial D}{\partial (\sigma_G)_i} \cancel{(\sigma'_G)_i} \frac{2(\sigma_G)_i}{\cancel{(\sigma'_G)_i}} \frac{\partial G_i}{\partial z_j} \\=&~0
\end{align}
\begin{align}
    \sum_{i=1}^{d_x} \frac{\partial D}{\partial (\sigma_G)_i} \frac{\partial (\sigma'_G)_i}{\partial z_{j}} \frac{\partial G_i}{\partial z_{k}} 
    =&~ \sum_{i=1}^{d_x} \frac{\partial D}{\partial (\sigma_G)_i} \frac{\partial (\sigma'_G)_i}{\partial G_i} \frac{\partial G_i}{\partial z_j} \frac{\partial G_i}{\partial z_{k}} \\
    =&~ \sum_{i=1}^{d_x} \frac{\partial D}{\partial (\sigma_G)_i} \left( 2(\sigma_G)^3_i-2(\sigma_G)_i \right) \frac{\partial G_i}{\partial z_j} \frac{\partial G_i}{\partial z_{k}} \\
    \sum_{i=1}^{d_x} \frac{\partial D}{\partial (\sigma_G)_i} (\sigma'_G)_i \cancel{\frac{\partial^2 G_i}{\partial z_{j} \partial z_{k}}} =&~ 0
\end{align}
Therefore, the simplified element is
\begin{align}
    (\nabla^2_\mathbf{z} \ell(\mathbf{z}))_{jk} = - \sum_{i=1}^{d_x} \frac{\partial D}{\partial (\sigma_G)_i} \left( 2(\sigma_G)^3_i-2(\sigma_G)_i \right) \frac{\partial G_i}{\partial z_j} \frac{\partial G_i}{\partial z_{k}} 
\end{align}
We now vectorize the expression of $\nabla^2_\mathbf{z} \ell(\mathbf{z})$ as follows:
\begin{align}
    \nabla^2_\mathbf{z} \ell(\mathbf{z}) = - 2 \left(\frac{\partial G}{\partial \mathbf{z}}\right)^T \operatorname{diag}\left( \frac{\partial D}{\partial \boldsymbol{\sigma}_G} \right) \operatorname{diag}\left( \boldsymbol{\sigma}_G^{\circ 3}-\boldsymbol{\sigma}_G \right) \left(\frac{\partial G}{\partial \mathbf{z}}\right)
\end{align}
Plugging this to Equation \ref{eq:36_apndx}, the following is obtained:
\begin{gather}
    \nabla_{\Sigma_c} \mathcal{L} = - \mathbb{E}_{q(\mathbf{z})} \left[ \pi_c \delta(\mathbf{z})_c \nabla^T_\mathbf{z} G \operatorname{diag}\left( \nabla_{\sigma_G}D \odot (\boldsymbol{\sigma}_G^{\circ 3}-\boldsymbol{\sigma}_G) \right) \nabla_\mathbf{z} G)  \right]
\end{gather}
\end{proof}

\newpage

\section{Methodological Details}
\subsection{Additive Angular Margin} \label{sec:arcface_apndx}
To enhance the discriminative power of U2C loss, we adopted the additive angular margin \citepapndx{arcface_apndx} as follows:
\begin{gather}
    \ell_\mathrm{U2C}(\mathbf{z}^i) = - \log \frac{\exp(s \cdot \cos (\theta_{ii} + m))}{\frac{1}{B} \{ \exp(s \cdot \cos (\theta_{ii} + m)) + \sum_{j \neq i} \exp(s \cdot \cos \theta_{ij}) \}} \label{eq:arcface}
\end{gather}
where $s$ denotes the feature scale, and $m$ is the angular margin. The feature scale $m$ and the coefficient of U2C loss $\lambda$ are linearly decayed to 0 during training, so that SLOGAN can focus more on the adversarial loss as training progresses.

\subsection{Cluster Assignment} \label{sec:cluster_assignment_apndx}
In Section 3.3, we chose $\cos \theta_{ic}$ as the critic function assuming that it is proportional to $\log p(c|\mathbf{x}_{g})$. If the real data distribution $p(\mathbf{x}_r)$ and the generator distribution $p(\mathbf{x}_g)$ are sufficiently similar via adversarial learning, the cosine similarity between $E(\mathbf{x}_r)$ and $\boldsymbol{\mu}_c$ can also be considered proportional to $\log p(c|\mathbf{x}_r)$. Therefore, for real data, we obtain the probability for each cluster as follows:
\begin{gather}
    \hat{p}(c|\mathbf{x}_r) = \frac{\exp(\cos \theta_{c})}{\sum_{k=1}^K \exp(\cos \theta_{k})}
\end{gather}
where $\cos \theta_{k} = E(\mathbf{x}_r) \cdot \boldsymbol{\mu}_{k} / \|E(\mathbf{x}_r)\| \|\boldsymbol{\mu}_{k}\|$ is the cosine similarity between $E(\mathbf{x}_r)$ and $\boldsymbol{\mu}_{k}$. The data can then be assigned to the cluster with the highest probability (i.e., $\operatorname*{argmax}_{c} \hat{p}(c|\mathbf{x}_r)$).

\subsection{Attribute Manipulation} \label{sec:attribute_manipulation_apndx}
We utilized mixup \citepapndx{mixup_apndx} to make the best use of a small amount of probe data when manipulating attributes. Algorithm \ref{alg:mixup} describes the procedure for using mixup for attribute manipulation when $K=2$. We applied the same feature scale and angular margin to $\mathcal{L}_\mathrm{m}$, as shown in Equation \ref{eq:arcface}. The number of iterations for the mixup ($T$) was set to five.
\newcommand{\algrule}[1][.2pt]{\par\vskip.2\baselineskip\hrule height #1\par\vskip.2\baselineskip}
\begin{algorithm}[h!]
    \caption{Attribute manipulation}
    \label{alg:mixup}
\begin{algorithmic}
    \STATE Initialize probe data with the desired attribute $\mathbf{X}_{c=1} \leftarrow \{\mathbf{x}^i_{c=1}\}_{i=1}^M$ and $\bar{\mathbf{X}}_{c=1} \leftarrow \{\mathbf{x}^i_{c=1}\}_{i=1}^M$
    \STATE Initialize probe data without the desired attribute $\mathbf{X}_{c=0} \leftarrow \{\mathbf{x}^i_{c=0}\}_{i=1}^M$ and $\bar{\mathbf{X}}_{c=0} \leftarrow \{\mathbf{x}^i_{c=0}\}_{i=1}^M$
    \FOR{each mixup iteration $t$ in $\{1, ..., T\}$}
    \STATE $\bar{\mathbf{X}}_{c=1} \leftarrow \bar{\mathbf{X}}_{c=1} \cup \textsc{Mixup}(\mathbf{X}_{c=1}, \textsc{Permute}(\mathbf{X}_{c=1}))$
    \STATE $\bar{\mathbf{X}}_{c=0} \leftarrow \bar{\mathbf{X}}_{c=0} \cup \textsc{Mixup}(\mathbf{X}_{c=0}, \textsc{Permute}(\mathbf{X}_{c=0}))$
    \ENDFOR
    \FOR{each augmented data index $j$ in $\{1, ..., M(T+1)\}$}
    \STATE $\cos \theta^{j}_{00} \leftarrow E(\bar{\mathbf{x}}_{c=0}^j) \cdot \boldsymbol{\mu}_0 / \|E(\bar{\mathbf{x}}_{c=0}^j)\| \|\boldsymbol{\mu}_0\|$
    \STATE $\cos \theta^{j}_{01} \leftarrow E(\bar{\mathbf{x}}_{c=0}^j) \cdot \boldsymbol{\mu}_1 / \|E(\bar{\mathbf{x}}_{c=0}^j)\| \|\boldsymbol{\mu}_1\|$
    \STATE $\cos \theta^{j}_{10} \leftarrow E(\bar{\mathbf{x}}_{c=1}^j) \cdot \boldsymbol{\mu}_0 / \|E(\bar{\mathbf{x}}_{c=1}^j)\| \|\boldsymbol{\mu}_0\|$
    \STATE $\cos \theta^{j}_{11} \leftarrow E(\bar{\mathbf{x}}_{c=1}^j) \cdot \boldsymbol{\mu}_1 / \|E(\bar{\mathbf{x}}_{c=1}^j)\| \|\boldsymbol{\mu}_1\|$
    \ENDFOR
    \STATE $\mathcal{L}_\mathrm{m} = - \frac{1}{M(T+1)} \sum_{j=1}^{M(T+1)} \log \frac{\exp(s \cdot \cos (\theta^{j}_{00}+m)) + \exp(s \cdot \cos (\theta^{j}_{11}+m))}{\exp(s \cdot \cos (\theta^{j}_{00}+m)) + \exp(s \cdot \cos \theta^{j}_{01}) + \exp(s \cdot \cos \theta^{j}_{10}) + \exp(s \cdot \cos (\theta^{j}_{11}+m))}$
    \STATE Minimize $\mathcal{L}^{p}$ with respect to $E$, $G$, and $\boldsymbol{\mu}$
\end{algorithmic}
\end{algorithm}

\subsection{SimCLR} \label{sec:simclr_apndx}
For the CIFAR and CelebA datasets, we used the SimCLR loss \citepapndx{simclr_apndx} for the encoder. We applied color, translation, and cutout transformations to the generated data using DiffAugment\footnote{\url{https://github.com/mit-han-lab/data-efficient-gans}} \citepapndx{diffaugment_apndx}. The SimCLR loss is calculated using the generated data $\mathbf{x}_g^i$ and augmentation $A$ as follows:
\begin{gather}
    \ell_{\mathrm{SimCLR}}(\mathbf{z}^i) = -\log \frac{\exp(E(\mathbf{x}_g^i) \cdot E(A(\mathbf{x}_g^i))~/~\|E(\mathbf{x}_g^i) \| \|E(A(\mathbf{x}_g^i)) \|)}{\sum_{j=1}^B \exp(E(\mathbf{x}_g^i) \cdot E(A(\mathbf{x}_g^j))~/~\|E(\mathbf{x}_g^i) \| \|E(A(\mathbf{x}_g^j)) \|)}
\end{gather}
where $\mathbf{x}_g^i=G(\mathbf{z}^i)$. The encoder $E$ is trained to minimize $\frac{1}{B} \sum_{i=1}^B \left( \ell_\mathrm{adv} (\mathbf{z}^i) + \ell_\mathrm{SimCLR} (\mathbf{z}^i) \right.$ $\left.+ \lambda \ell_\mathrm{U2C} (\mathbf{z}^i) \right)$.

\subsection{DeLiGAN+} \label{sec:deligan+_apndx}
Among the existing unsupervised conditional GANs, DeLiGAN lacks an encoder network. Therefore, for a fair comparison, we added an encoder network and named it DeLiGAN+. We set the output dimension of the encoder to equal the number of mixture components of the latent distribution. For the $i$-th example in the batch, when the $c_i$-th mixture component of the latent distribution is selected, DeLiGAN+ is learned through the following objective:
\begin{gather}
    \min_{G,E,\mu_c,\sigma_c} \max_D \frac{1}{B} \sum_{i=1}^B \left[ D(\mathbf{x}^i) - D(G(\mathbf{z}^i, \mathbf{c}_i)) - \lambda_{CE} ~ \mathbf{c}_i^T \log E(G(\mathbf{z}^i, \mathbf{c}_i)) \right]
\end{gather}
where $\mathbf{c}_i$ is the one-hot vector corresponding to $c_i$ and $\lambda_{CE}$ is the coefficient of the cross entropy loss. We set $\lambda_{CE}$ to 10 in the experiments.

\subsection{Evaluation Metric}
\paragraph{Cluster assignment} We do not use clustering purity which is an evaluation metric for cluster assignment. To compute the clustering purity, the most frequent class in the cluster is obtained, and the ratio of the data points belonging to the class is calculated. However, if the attributes in the data are imbalanced, multiple clusters can be assigned to a single class in duplicate, and this high clustering purity misleads the results. Therefore, we utilized the normalized mutual information (NMI) implemented in scikit-learn\footnote{\url{https://github.com/scikit-learn/scikit-learn/blob/15a949460/sklearn/metrics/cluster/_supervised.py}}.

\paragraph{Unconditional generation} 
FID has the advantage of considering not only sample quality but also diversity, whereas Inception score (IS) cannot assess the diversity properly because IS does not compare generated samples with real samples \citepapndx{shmelkov2018good_apndx}. Therefore, we used FID as the evaluation metric for unsupervised generation.

\paragraph{Unsupervised conditional generation} \label{sec:icfid_apndx}
If attributes in data are severely imbalanced, FID does not increase (deteriorate) considerably even if the model does not generate data containing the minority attributes. Therefore, the FID cannot accurately measure the unsupervised conditional generation performance for data with severely imbalanced attributes. We introduce ICFID to evaluate the performance of unsupervised conditional generation. When calculating ICFID, multiple clusters cannot be assigned to a single class in duplicate. Therefore, if data of a single class are generated from multiple discrete latent variables or modes, the model shows high (bad) ICFID.

\newpage

\section{Implementation Details} \label{sec:implementation_details_apndx}

\subsection{General Settings and Environments}
For simplicity, we denote the learning rate of G as $\eta$ and the learning rate of $\boldsymbol{\Sigma}$ as $\gamma$. Throughout the experiments, we set the learning rate of $E$ to $\eta$, and $D$ to $4\eta$ using the two-timescale update rule (TTUR) \citepapndx{ttur_apndx}. We set the learning rate of $\boldsymbol{\mu}$ to $10\gamma$, and the learning rate of $\boldsymbol{\rho}$ to $\gamma$. We set $B$ to 64 and the number of training steps to 100k. To stabilize discriminator learning, we used 
Lipschitz penalty \citepapndx{wgan-lp_apndx} for the synthetic, MNIST, FMNIST, and 10x\_73k datasets, and adversarial Lipschitz regularization \citepapndx{wgan-alp_apndx} for the CIFAR-10 and CelebA datasets. We repeated each experiment 3 times and reported the means and standard deviations of model performances. Hyperparameters are determined by a grid search. We used the Adam optimizer \citepapndx{adam_apndx} for training $D$, $G$, and $E$, and a gradient descent optimizer for training $\boldsymbol{\mu}$, $\boldsymbol{\Sigma}$, and $\boldsymbol{\rho}$. The experiments herein were conducted on a machine equipped with an Intel Xeon Gold 6242 CPU and an NVIDIA Quadro RTX 8000 GPU. The code is implemented in Python 3.7 and Tensorflow 1.14 \citepapndx{tensorflow_apndx}.

\subsection{Synthetic Dataset}
For the synthetic dataset, we first set the mean of eight 2-dimensional Gaussian distributions as $(0,2)$, $(\sqrt{2},\sqrt{2})$, $(2,0)$, $(\sqrt{2},-\sqrt{2})$, $(0,-2)$, $(-\sqrt{2},-\sqrt{2})$, $(-2,0)$, and $(-\sqrt{2},\sqrt{2})$, and the variance as $0.01I$. In Figure \ref{fig:toy_variance}, we also set the variances as $0.05I$ and $0.1I$. The number of data sampled from the Gaussian distributions was set to 5,000, 5,000, 5,000, 5,000, 15,000, 15,000, 15,000, and 15,000. We scaled a total of 80,000 data points to a range between -1 and 1. Table~\ref{tab:arc_synthetic} shows the network architectures of SLOGAN used for the synthetic dataset. Linear $n$ denotes a fully-connected layer with $n$ output units. BN and SN denote batch normalization and spectral normalization, respectively. LReLU denotes the leaky ReLU. We set $\lambda=4$, $\eta=0.001$, $\gamma=0.01$, $s=2$, and $m=0.5$.

\begin{table}[h!]
\centering
\caption{SLOGAN architecture used for the synthetic dataset}
\label{tab:arc_synthetic}
{\resizebox{0.75\columnwidth}{!}
{\begin{tabular}{lll}
    \toprule
    \multicolumn{1}{c}{$G$} & \multicolumn{1}{c}{$D$} & \multicolumn{1}{c}{$E$} \\
    \midrule
    $\mathbf{z} \in \mathbb{R}^{64}$ & $\mathbf{x} \in \mathbb{R}^2$ & $\mathbf{x} \in \mathbb{R}^2$ \\
    Linear 128 + BN + ReLU & Linear 128 + LReLU & Linear 128 + SN + LReLU \\
    Linear 128 + BN + ReLU & Linear 128 + LReLU & Linear 128 + SN + LReLU \\
    Linear 2 + Tanh & Linear 1 & Linear 64 + SN \\
    \bottomrule
\end{tabular}}}
\end{table}

\subsection{MNIST and Fashion-MNIST Datasets}
The MNIST dataset \citepapndx{mnist_apndx} consists of handwritten digits, and the Fashion-MNIST (FMNIST) dataset \citepapndx{fmnist_apndx} is comprised of fashion products. Both the MNIST and Fashion-MNIST datasets have 60,000 training and 10,000 test 28$\times$28 grayscale images. Each pixel was scaled to a range of 0$-$1. The datasets consist of 10 classes, and the number of data points per class is balanced. Table~\ref{tab:arc_mnist} shows the network architectures of SLOGAN used for the MNIST and FMNIST datasets. Conv $k \times k$, $s$, $n$ denotes a convolutional network with $n$ feature maps, filter size $k \times k$, and stride $s$. Deconv $k \times k$, $s$, $n$ denotes a deconvolutional network with $n$ feature maps, filter size $k \times k$, and stride $s$. For the MNIST dataset, we set $\lambda=10$, $\eta=0.0001$, $\gamma=0.002$, $s=8$, and $m=0.5$. For MNIST-2, we set $\lambda=4$, $\eta=0.0001$, $\gamma=0.002$, $s=4$, and $m=0.5$. For the FMNIST dataset, we set $\lambda=10$, $\eta=0.0001$, $\gamma=0.001$, $s=1$, and $m=0$. For FMNIST-5, we set $\lambda=1$, $\eta=0.0002$, $\gamma=0.004$, $s=4$, and $m=0.5$.

\begin{table}[h!]
\centering
\caption{SLOGAN architecture used for the MNIST and FMNIST datasets}
\label{tab:arc_mnist}
{\resizebox{0.9\columnwidth}{!}
{\begin{tabular}{lll}
    \toprule
    \multicolumn{1}{c}{$G$} & \multicolumn{1}{c}{$D$} & \multicolumn{1}{c}{$E$} \\
    \midrule
    $\mathbf{z} \in \mathbb{R}^{1\times1\times64}$ & $\mathbf{x} \in \mathbb{R}^{28\times28\times1}$ & $\mathbf{x} \in \mathbb{R}^{28\times28\times1}$ \\
    Deconv 1$\times$1, 1, 1024 + BN + ReLU & Conv 4$\times$4, 2, 64 + LReLU & Conv 4$\times$4, 2, 64 + LReLU \\
    Deconv 7$\times$7, 1, 128 + BN + ReLU & Conv 4$\times$4, 2, 64 + LReLU & Conv 4$\times$4, 2, 64 + LReLU \\
    Deconv 4$\times$4, 2, 64 + BN + ReLU & Conv 7$\times$7, 1, 1024 + LReLU & Conv 7$\times$7, 1, 1024 + LReLU \\
    Deconv 4$\times$4, 2, 1 + Sigmoid & Conv 1$\times$1, 1, 1 & Conv 1$\times$1, 1, 64 \\
    \bottomrule
\end{tabular}}}
\end{table}

\subsection{10x\_73k Dataset}
The 10x\_73k dataset \citepapndx{10x73k_apndx} consists of 73,233 720-dimensional vectors, which are obtained from RNA transcript counts, and has eight cell types (classes). The number of data points per cell type is 10,085, 2,612, 9,232, 8,385, 10,224, 11,953, 10,479, and 10,263. We converted each element to logscale (i.e., $\log_2 (x+1)$) and scaled each element to a range between 0 and 1. Table~\ref{tab:arc_10x} shows the network architectures of SLOGAN used for the 10x\_73k dataset. We set $\lambda=10$, $\eta=0.0001$, $\gamma=0.004$, $s=4$, and $m=0$.

\begin{table}[h!]
\centering
\caption{SLOGAN architecture used for the 10x\_73k dataset}
\label{tab:arc_10x}
{\resizebox{0.7\columnwidth}{!}
{\begin{tabular}{lll}
    \toprule
    \multicolumn{1}{c}{$G$} & \multicolumn{1}{c}{$D$} & \multicolumn{1}{c}{$E$} \\
    \midrule
    $\mathbf{z} \in \mathbb{R}^{64}$ & $\mathbf{x} \in \mathbb{R}^2$ & $\mathbf{x} \in \mathbb{R}^2$ \\
    Linear 256 + LReLU & Linear 256 + LReLU & Linear 256 + SN + LReLU \\
    Linear 256 + LReLU & Linear 256 + LReLU & Linear 256 + SN + LReLU \\
    Linear 720 & Linear 1 & Linear 64 + SN \\
    \bottomrule
\end{tabular}}}
\end{table}

\subsection{CIFAR-10 Dataset}
The CIFAR-10 \citepapndx{cifar10_apndx} dataset comprises 50,000 training and 10,000 test 32$\times$32 color images. Each pixel was scaled to a range of -1 to 1. The number of data points per class is balanced. Figure \ref{fig:resblock_apndx} and Table \ref{tab:arc_cifar} show the network architectures of residual blocks and SLOGAN used for the CIFAR datasets. AvgPool and GlobalAvgPool denote the average pooling and global average pooling layers, respectively. For the CIFAR-10, CIFAR-2, CIFAR-2 (7:3), and CIFAR-2 (9:1) datasets, we set $\lambda=1$, $\eta=0.0001$, $\gamma=0.002$, $s=4$, and $m=0.5$.

\begin{figure*}[h!]
\centering
\includegraphics[width=1\textwidth]{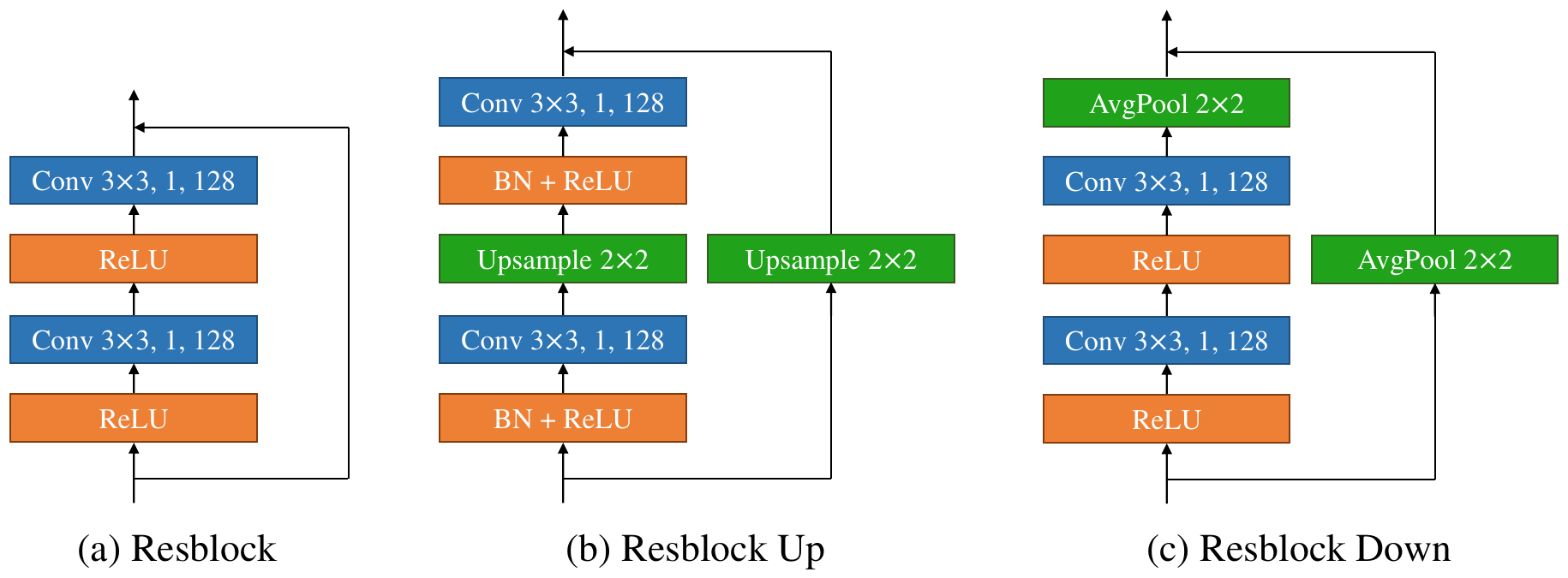} 
\caption{Resblock architectures used for colored image datasets.}
\label{fig:resblock_apndx}
\end{figure*}

\begin{table}[h!]
\centering
\caption{SLOGAN architecture used for the CIFAR datasets}
\label{tab:arc_cifar}
{\resizebox{0.8\columnwidth}{!}
{\begin{tabular}{lll}
    \toprule
    \multicolumn{1}{c}{$G$} & \multicolumn{1}{c}{$D$} & \multicolumn{1}{c}{$E$} \\
    \midrule
    $\mathbf{z} \in \mathbb{R}^{128}$ & $\mathbf{x} \in \mathbb{R}^{32\times32\times3}$ & $\mathbf{x} \in \mathbb{R}^{32\times32\times3}$ \\
    Linear 2048 + Reshape 4, 4, 128 & Resblock Down & Resblock Down \\
    Resblock Up & Resblock Down & Resblock Down \\
    Resblock Up & Resblock & Resblock \\
    Resblock Up & Resblock & Resblock \\
    BN + ReLU & ReLU + GlobalAvgPool & ReLU + GlobalAvgPool \\
    Conv 3$\times$3, 1, 3 + Tanh & Linear 1 & Linear 128 \\
    \bottomrule
\end{tabular}}}
\end{table}

\subsection{CelebA Dataset}
The CelebA dataset \citepapndx{celeba_apndx} consists of 202,599 face attributes. We cropped the face part of each image to 140$\times$140 pixels, resized it to 64$\times$64 pixels, and scaled it to a range between -1 and 1. The imbalanced ratio is different for each attribute, and the imbalanced ratios of male and eyeglasses used in the experiment are 1:1 and 14:1, respectively. Table \ref{tab:arc_celeba} lists the network architectures of SLOGAN used for the CelebA dataset. We set $\lambda=1$, $\eta=0.0002$, $\gamma=0.0006$, $s=4$, and $m=0.5$.

\begin{table}[h!]
\centering
\caption{SLOGAN architecture used for the CelebA dataset}
\label{tab:arc_celeba}
{\resizebox{0.8\columnwidth}{!}
{\begin{tabular}{lll}
    \toprule
    \multicolumn{1}{c}{$G$} & \multicolumn{1}{c}{$D$} & \multicolumn{1}{c}{$E$} \\
    \midrule
    $\mathbf{z} \in \mathbb{R}^{128}$ & $\mathbf{x} \in \mathbb{R}^{64\times64\times3}$ & $\mathbf{x} \in \mathbb{R}^{64\times64\times3}$ \\
    Linear 8192 + Reshape 8, 8, 128 & Resblock Down & Resblock Down \\
    Resblock Up & Resblock Down & Resblock Down \\
    Resblock Up & Resblock Down & Resblock Down \\
    Resblock Up & Resblock & Resblock \\
    BN + ReLU & ReLU + GlobalAvgPool & ReLU + GlobalAvgPool \\
    Conv 3$\times$3, 1, 3 + Tanh & Linear 1 & Linear 128 \\
    \bottomrule
\end{tabular}}}
\end{table}

\subsection{CelebA-HQ Dataset}
The CelebA-HQ dataset \citepapndx{celeba-hq} consists of 30,000 face attributes. We resized each image to 128$\times$128 and 256$\times$256 pixels, and scaled it to a range between -1 and 1. The imbalance ratio of male used in the experiment was 1.7:1. We used StyleGAN2 \citepapndx{stylegan2} architecture with DiffAugment\footnote{\url{https://github.com/mit-han-lab/data-efficient-gans/tree/master/DiffAugment-stylegan2}} and applied implicit reparameterization to the input space of the mapping network. We set $\lambda=1$, $\eta=0.002$, $\gamma=0.0006$, $s=1$, and $m=0$.

\subsection{AFHQ Dataset}
The AFHQ dataset \citepapndx{afhq} consists of 15,000 high-quality animal faces. We used cats and dogs, and there are about 5,000 images each in the dataset. We resized each image to 256$\times$256 pixels, and scaled it to a range between -1 and 1. We set the imbalance ratios of cats and dogs to 1:1, 1:2, and 1:5 by reducing the number of cats in the training dataset. We used the same model architecture and hyperparameters as for the CelebA-HQ dataset.

\subsection{Code Availability}
Code is available at \url{https://github.com/shinyflight/SLOGAN}

\newpage
\bibliographystyleapndx{iclr2022_conference}
\bibliographyapndx{appendix}

\end{document}